\PassOptionsToPackage{dvipsnames,table}{xcolor}
\documentclass{article}

\PassOptionsToPackage{numbers,compress}{natbib}
\usepackage[preprint]{neurips_2026}

\usepackage[utf8]{inputenc}
\usepackage[T1]{fontenc}
\usepackage{hyperref}
\usepackage{url}
\usepackage{microtype}
\usepackage{graphicx}
\usepackage{subcaption}
\usepackage{wrapfig}
\usepackage{booktabs}
\usepackage{amsfonts}
\usepackage{nicefrac}
\usepackage{xcolor}
\usepackage[textsize=scriptsize]{todonotes}
\usepackage{soul}
\usepackage{tcolorbox}
\tcbuselibrary{breakable,skins}
\usepackage[framemethod=tikz]{mdframed}
\usepackage{amsmath}
\usepackage{amssymb}
\usepackage{mathtools}
\usepackage{amsthm}
\usepackage{xargs}
\usepackage{multirow}
\usepackage{array}
\usepackage{stmaryrd}
\usepackage{algorithm}
\usepackage{varwidth}
\usepackage{bm}
\usepackage{tikz}


\usepackage{algpseudocode}

\usepackage[capitalize,noabbrev]{cleveref}

\algtext*{EndFor}
\algtext*{EndIf}
\gdef\algorithmiccomment#1{\hfill $\triangleright$ #1}

\hypersetup{
    colorlinks=true,
    citecolor=black!40,
    linkcolor=RoyalBlue!90,
    urlcolor=black!30
}

\newcommand{\runtimefactor}[1]{\textcolor{RoyalBlue}{(\textbf{#1$\boldsymbol{\times}$})}}

\newcommand{\contribnum}[1]{%
    \tikz[baseline=(char.base)]{
        \node[
            circle,
            fill=orange!30,
            inner sep=1.15pt,
            text height=1.45ex,
            text depth=.2ex,
            font=\bfseries\small
        ] (char) {#1};
    }%
}
\newcommand{\leftcell}[1]{\multicolumn{1}{l}{#1}}
\newcommand{\tabstd}[1]{\textcolor{black!45}{\hspace{0.08em}\raisebox{0.18ex}{\scalebox{0.78}{\(\pm\,#1\)}}}}
\newcommand{\tabmean}[1]{\makebox[2.08em][r]{#1}}
\newcommand{\tabwidemean}[1]{\makebox[2.55em][r]{#1}}
\newcommand{\tabval}[2]{\tabmean{#1}\tabstd{#2}}
\newcommand{\tabbest}[2]{\tabmean{\textbf{#1}}\tabstd{#2}}
\newcommand{\tabwideval}[2]{\tabwidemean{#1}\tabstd{#2}}
\newcommand{\tabwidebest}[2]{\tabwidemean{\textbf{#1}}\tabstd{#2}}
\newcommand{\tabhi}[2]{\begingroup\setlength{\fboxsep}{1.1pt}\kern-1.1pt\colorbox{#1}{#2}\kern-1.1pt\endgroup}
\newcommand{\factorcell}[2]{\multicolumn{1}{l}{\makebox[2.4em][l]{#1}\hspace{0.18em}\makebox[2.45em][r]{\runtimefactor{#2}}}}
\newcommand{\colorfactorcell}[3]{\multicolumn{1}{l}{\tabhi{#1}{\makebox[2.4em][l]{#2}\hspace{0.18em}\makebox[2.45em][r]{\runtimefactor{#3}}}}}
\newcommand{\widefactorcell}[2]{\multicolumn{1}{l}{\makebox[2.55em][l]{#1}\makebox[2.95em][r]{\runtimefactor{#2}}}}

\newenvironment{callout}
{
    \begin{tcolorbox}[
        colback=orange!5!white,
        colframe=orange!75!black,
        coltitle=black,
        boxrule=1pt,
        arc=3pt,
        left=6pt,
        right=6pt,
        top=6pt,
        bottom=6pt,
        before upper=\setlength{\parindent}{0pt}]
}
{
    \end{tcolorbox}
}

\newenvironment{warmstartlemmabox}
{
    \begin{mdframed}[
        backgroundcolor=RoyalBlue!10!white,
        linecolor=RoyalBlue!35!black,
        linewidth=0.2pt,
        roundcorner=3pt,
        innerleftmargin=5pt,
        innerrightmargin=5pt,
        innertopmargin=5pt,
        innerbottommargin=5pt,
        skipabove=6pt,
        skipbelow=6pt]
}
{
    \end{mdframed}
}

\theoremstyle{plain}
\newtheorem{theorem}{Theorem}[section]
\newtheorem{proposition}[theorem]{Proposition}
\newtheorem{lemma}[theorem]{Lemma}
\newtheorem{corollary}[theorem]{Corollary}
\theoremstyle{definition}

\newtheorem{assumption}[theorem]{Assumption}
\theoremstyle{remark}
\newtheorem{remark}[theorem]{Remark}

\sethlcolor{yellow!35}

\def\Id{\mathbf{I}}
\def\zero{0}
\def\pV{\mathbb{V}}
\def\pE{\mathbb{E}}

\def\cov{\Sigma}

\def\eqdef{\vcentcolon=}

\newcommand{\tbar}{\mathrel{\raisebox{0.15ex}{$\scriptscriptstyle\mid$}}}

\def\law{\mathsf{Law}}

\DeclareMathOperator*{\argmin}{argmin}

\def\opA{\mathcal{A}}

\newcommand{\tk}[1]{{t_{#1}}}

 \def\eqdef{\coloneqq}

\newcommandx\targ[4][4=]{
    \ifthenelse{\equal{#2}{}}{    
        \ifthenelse{\equal{#3}{}}{    
            p^{#4} _{#1}
        }{
            p^{#4} _{#1}(#3)
            }
    }{
        \ifthenelse{\equal{#3}{}}{    
            p^{#4} _{#1}(\cdot|#2)
        }{
            p^{#4} _{#1}(#3|#2)
        }
    }}

\newcommandx\interp[4][4=\interpscale]{
    \ifthenelse{\equal{#2}{}}{    
        \ifthenelse{\equal{#3}{}}{    
            \pi^{#4} _{#1}
        }{
            \pi^{#4} _{#1}(#3)
            }
    }{
        \ifthenelse{\equal{#3}{}}{    
            \pi^{#4} _{#1}(\cdot|#2)
        }{
            \pi^{#4} _{#1}(#3|#2)
        }
    }}

\newcommandx\pdata[4][4=]{
    \ifthenelse{\equal{#2}{}}{    
        \ifthenelse{\equal{#3}{}}{    
            p^{#4} _{#1}
        }{
            p^{#4} _{#1}(#3)
        }
    }{
        \ifthenelse{\equal{#3}{}}{    
            p^{#4} _{#1}(\cdot|#2)
        }{
            p^{#4} _{#1}(#3|#2)
        }
    }}

\newcommandx\hpdata[4][4=]{
    \ifthenelse{\equal{#2}{}}{    
        \ifthenelse{\equal{#3}{}}{    
            \hat{p}^{#4} _{#1}
        }{
            \hat{p}^{#4} _{#1}(#3)
        }
    }{
        \ifthenelse{\equal{#3}{}}{    
            \hat{p}^{#4} _{#1}(\cdot|#2)
        }{
            \hat{p}^{#4} _{#1}(#3|#2)
        }
    }}

\newcommandx\revker[4][4=]{
\ifthenelse{\equal{#2}{}}{    
    \ifthenelse{\equal{#3}{}}{    
        r^{#4} _{#1}
    }{
        r^{#4} _{#1}(#3)
    }
}{
    \ifthenelse{\equal{#3}{}}{    
        r^{#4} _{#1}(\cdot|#2)
    }{
        r^{#4} _{#1}(#3|#2)
    }
}}

\newcommandx\hatrevker[4][4=]{
    \ifthenelse{\equal{#2}{}}{    
        \ifthenelse{\equal{#3}{}}{    
            \hat{r}^{#4} _{#1}
        }{
            \hat{r}^{#4} _{#1}(#3)
        }
    }{
        \ifthenelse{\equal{#3}{}}{    
            \hat{r}^{#4} _{#1}(\cdot|#2)
        }{
            \hat{r}^{#4} _{#1}(#3|#2)
        }
    }
}

\newcommandx\refreshker[4][4=]{
    \ifthenelse{\equal{#2}{}}{    
        \ifthenelse{\equal{#3}{}}{    
            m^{#4} _{#1}
        }{
            m^{#4} _{#1}(#3)
        }
    }{
        \ifthenelse{\equal{#3}{}}{    
            m^{#4} _{#1}(\cdot|#2)
        }{
            m^{#4} _{#1}(#3|#2)
        }
    }}
\newcommand\jpdata[3]{
    \ifthenelse{\equal{#2}{}}{    
        \ifthenelse{\equal{#3}{}}{    
            \bar{p}_{#1}
        }{
            \bar{p}_{#1}(#3)
        }
    }{
        \ifthenelse{\equal{#3}{}}{    
            \bar{p}_{#1}(\cdot|#2)
        }{
            \bar{p}_{#1}(#3|#2)
        }
    }}

\newcommandx\post[4][4=]{
    \ifthenelse{\equal{#2}{}}{    
        \ifthenelse{\equal{#3}{}}{    
            \pi^{#4} _{#1}
        }{
            \pi^{#4} _{#1}(#3)
        }
    }{
        \ifthenelse{\equal{#3}{}}{    
            \pi^{#4} _{#1}(\cdot|#2)
        }{
            \pi^{#4} _{#1}(#3|#2)
        }
    }}

\newcommandx\hpost[4][4=]{
    \ifthenelse{\equal{#2}{}}{    
        \ifthenelse{\equal{#3}{}}{    
            \hat\pi^{#4} _{#1}
        }{
            \hat\pi^{#4} _{#1}(#3)
        }
    }{
        \ifthenelse{\equal{#3}{}}{    
            \hat\pi^{#4} _{#1}(\cdot|#2)
        }{
            \hat\pi^{#4} _{#1}(#3|#2)
        }
    }}

  \newcommandx\pot[3][3=]{
        \ifthenelse{\equal{#3}{}}{    
            \ell^{#3} _{#1}(\obs|#2)
        }{
            \ell^{#3} _{#1}(\obs|#2)
        }
    }

\newcommandx\hpot[4][4=]{
    \ifthenelse{\equal{#2}{}}{    
        \ifthenelse{\equal{#3}{}}{    
            \hat{\ell}^{#4} _{#1}
        }{
            \hat{\ell}^{#4} _{#1}(#3)
        }
    }{
        \ifthenelse{\equal{#3}{}}{    
            \hat{\ell}^{#4} _{#1}(\cdot|#2)
        }{
            \hat{\ell}^{#4} _{#1}(#3|#2)
        }
    }}

\newcommandx\fw[4][4=]{
        \ifthenelse{\equal{#3}{}}{    
            q^{#4} _{\smash{#1}}(\cdot|#2)
        }{
            q^{#4} _{\smash{#1}}(#3|#2)
        }
    }

\newcommandx\denoiser[5][4=, 5=0]{
    \ifthenelse{\equal{#2}{}}{    
        \ifthenelse{\equal{#3}{}}{    
            \hat\bx^{#4}_{#5}(\cdot, #1)
        }{
            \hat\bx^{#4} _{#5}(#3, #1)
        }
    }{
        \ifthenelse{\equal{#3}{}}{    
            \hat\bx^{#4} _{#5}(\cdot, #1|#2)
        }{
            \hat\bx^{#4} _{#5}(#3, #1|#2)
        }
    }
}

\newcommandx\noisepred[4][4=]{
    \ifthenelse{\equal{#2}{}}{    
        \ifthenelse{\equal{#3}{}}{    
            \hat\bx^{#4}_{1}(\cdot, #1)
        }{
            \hat\bx^{#4} _{1}(#3, #1)
        }
    }{
        \ifthenelse{\equal{#3}{}}{    
            \hat\bx^{#4} _{1}(\cdot, #1|#2)
        }{
            \hat\bx^{#4} _{1}(#3, #1|#2)
        }
    }
}

    \newcommandx\hdenoiser[4][4=]{
    \ifthenelse{\equal{#2}{}}{    
        \ifthenelse{\equal{#3}{}}{    
            \hat{D}^{#4}_{#1}
        }{
            \hat{D}^{#4} _{#1}(#3)
        }
    }{
        \ifthenelse{\equal{#3}{}}{    
            \hat{D}^{#4} _{#1}(\cdot|#2)
        }{
            \hat{D}^{#4} _{#1}(#3|#2)
        }
    }}

\newcommandx\epspred[4][4=]{
    \ifthenelse{\equal{#2}{}}{    
        \ifthenelse{\equal{#3}{}}{    
            \varepsilon^{#4}_{#1}
        }{
            \varepsilon^{#4} _{#1}(#3)
        }
    }{
        \ifthenelse{\equal{#3}{}}{    
            \varepsilon^{#4} _{#1}(\cdot|#2)
        }{
            \varepsilon^{#4} _{#1}(#3|#2)
        }
    }}

\newcommand\clf[3]{
    \ifthenelse{\equal{#2}{}}{    
        g_{#1}(#3|\cdot)
    }{
        g _{#1}(#3|#2)
    }
}

\newcommand\cfgdist[3]{
    \ifthenelse{\equal{#2}{}}{    
        p^w _{#1}(#3|\cdot)
    }{
        p^w _{#1}(#3|#2)
    }
}

\newcommand\cscore[3]{
    \ifthenelse{\equal{#2}{}}{    
        \ifthenelse{\equal{#3}{}}{    
            s _{#1}
        }{
            s _{#1}(#3)
        }
    }{
        \ifthenelse{\equal{#3}{}}{    
            s _{#1}(\cdot|#2)
        }{
            s _{#1}(#3|#2)
        }
    }}

\def\gauss{\mathcal{N}}

\def\interpscale{\eta}

\def\std{\sigma}

\def\acp{\alpha}

\def\bx{\mathbf{x}}

\def\bz{\mathbf{z}}
\def\BZ{Z}

\def\bmu{\boldsymbol{\mu}}

\def\div2k{{\texttt{DIV2K}}}

\def\param{\theta}

\def\obs{\mathbf{y}}

\newcommandx{\hpredx}[3][2=0,3=\param]{\smash{m^{#3} _{#2|#1}}}
\newcommandx{\predx}[2][2=0]{\smash{m _{#2|#1}}}
\newcommandx{\prednoise}[2][2=\param]{\smash{\epsilon^{#2} _{#1}}}
\newcommandx{\score}[2][2=\param]{s^{#2} _{#1}}

\def\encoder{\mathcal{E}}
\def\decoder{\mathcal{D}}
\def\DC{\mathsf{DC}}

\title{Sparse Scheduled Diffusion Guidance for \\ Inverse Problems}

\author{%
Abduragim Shtanchaev \\
MBZUAI, Abu Dhabi, United Arab Emirates \\
\texttt{abduragim.shtanchaev@mbzuai.ac.ae}
\And
Albina Ilina \\
MBZUAI, Abu Dhabi, United Arab Emirates
\And
Yazid Janati \\
MBZUAI, Abu Dhabi, United Arab Emirates \\
Institute of Foundation Models
\And
Arip Asadulaev \\
MBZUAI, Abu Dhabi, United Arab Emirates
\AND
Martin Tak\'a\v{c} \\
MBZUAI, Abu Dhabi, United Arab Emirates
\And
Eric Moulines \\
MBZUAI, Abu Dhabi, United Arab Emirates \\
EPITA, Paris, France
}

\begin{document}

\maketitle

\begin{abstract}
Pretrained diffusion models are effective priors for Bayesian inverse problems, but posterior sampling with these priors is often costly because data-consistency guidance is applied throughout the full reverse trajectory. Existing methods have shown that vector-Jacobian products through the denoiser can sometimes be avoided, yet they typically still rely on dense guidance through the full trajectory or expensive inner solves. We introduce \textbf{Sp}arse Scheduled Diffusion Guidance for \textbf{In}verse Problems (\textsc{Spin}), a solver that avoids starting posterior sampling from pure noise. \textsc{Spin} first samples from a posterior time-marginal at an intermediate timestep $\bm{t_*}$, and then uses that state as a warm start for a guided reverse diffusion process. At guidance time, instead of enforcing the measurement constraint at every denoising step, \textsc{Spin} applies lightweight corrections only at scheduled timesteps where the denoiser can still clean up artifacts. The resulting procedure decouples prior refinement from data consistency: the prior supplies denoising, while lightweight pixel-space optimization enforces the measurement constraint without backpropagation through the denoiser or decoder. Across linear and nonlinear inverse problems on FFHQ and ImageNet, \textsc{Spin} achieves competitive reconstruction quality with a substantially better runtime--memory profile, running $2\times$ faster on pixel-space models and up to $50\times$ faster on latent diffusion models, with lower memory costs.
\end{abstract}

\begin{figure*}[t!]
    \centering
    \includegraphics[width=0.9\linewidth]{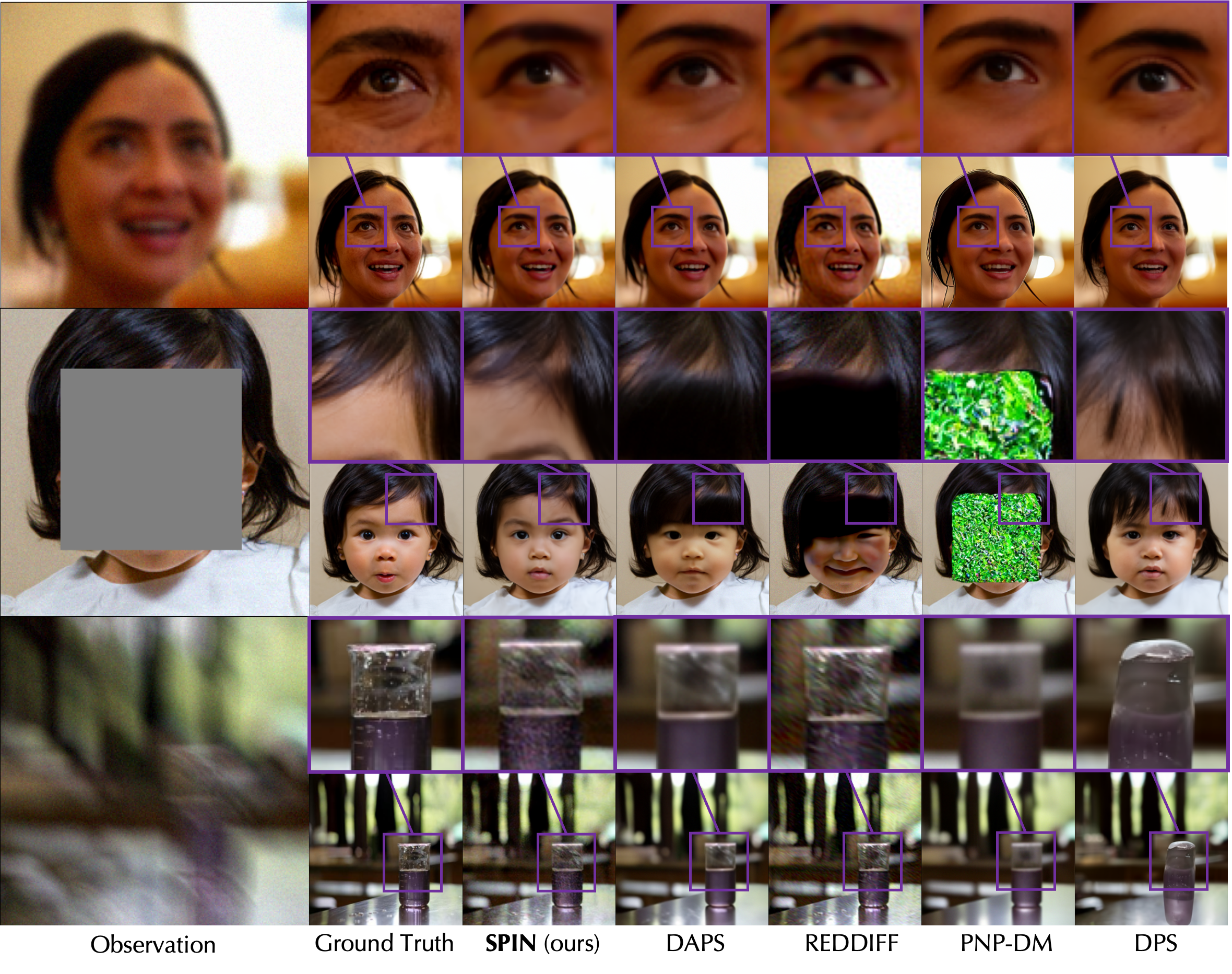}
    \caption{\textbf{Qualitative results on FFHQ and ImageNet.} Comparison with baseline methods across Gaussian deblurring and center inpainting on \texttt{FFHQ} (\textit{top and middle rows, respectively}), and motion deblurring on \texttt{ImageNet} (\textit{bottom row}). \textsc{Spin} (ours) achieves more accurate and detailed restoration compared to existing approaches. See Appendix \ref{sec:images_ffhq}/\ref{sec:images_imagenet} for more samples.}
    \label{fig:main_results}
    \vspace{-15pt}
\end{figure*}

\section{Introduction}
\vspace{-5pt}
\label{sec:intro}

Zero-shot generative modeling refers to leveraging a pre-trained generative model of the data distribution $p(\bx)$ and adapting it \emph{without task-specific retraining} to solve downstream tasks not included during training~\citep{moufad2025efficient}. In this setting, the only new ingredient at test time is a description of the task (e.g., a measurement operator or constraint), and the generative prior is used to produce realistic solutions that satisfy these new conditions.

A convenient probabilistic formalization of many such zero-shot tasks is through \emph{Bayesian inverse problems}. Inverse problems arise when a signal of interest $\bx$ must be recovered from incomplete, corrupted, or indirect observations $\obs$. Formally, this involves retrieving an original signal $\bx$ from a measurement $\obs = \mathcal{A}(\bx) + \mathbf{n}$, where $\mathcal{A}$ is a degradation/forward operator (such as a blur kernel or masking) and $\mathbf{n}$ is noise. Because information is usually lost during this degradation process, the problem is ill-posed and requires strong prior knowledge to recover the original signal $\bx$.

The Bayesian framework offers a principled way to address these challenges by combining a data likelihood with a learned or structured prior. A common strategy is to use a strong generative prior at inference time to solve different inverse problems without retraining. These methods enable adaptation to different measurement processes, such as masking, blurring, or downsampling ~\citep{chung2022diffusion, song2023pseudoinverse}. In this setting, a diffusion model acts as a denoising prior, iteratively guiding candidate solutions toward realistic data while keeping them consistent with the measurements. 

At the core of diffusion-based posterior sampling is a sequential denoising process that approximates a sequence of intermediate posterior distributions. These distributions combine the learned prior with measurement consistency at each noise level. In practice, many existing methods start from pure noise and apply data-consistency guidance throughout the full reverse trajectory. Moreover, approximating the exact posterior denoiser often requires an intractable likelihood gradient, leading common methods to compute vector-Jacobian products through the denoiser at every step~\citep{chung2022diffusion, song2023pseudoinverse}. While effective, these choices introduce substantial computational and memory overhead, limiting scalability and practical deployment.

Our key observation is that \textit{posterior sampling does not need to start from pure noise}, and \textit{data consistency does not need to be enforced at every denoising step}. Instead, we can first construct a task-informed sample at an intermediate timestep $\bm{t_*}$ and then run only the remaining part of the reverse trajectory while applying sparse data-consistency corrections at scheduled timesteps. Applying them too densely can repeatedly perturb the diffusion trajectory, while applying them too late leaves little time for the denoiser to remove artifacts introduced by the measurement objective.

Building on this observation, we introduce \textsc{Spin} (\textbf{Sp}arse Scheduled Diffusion Guidance for \textbf{In}verse Problems), a two-phase framework for efficient diffusion-based inverse problem solving. Our contributions are threefold. \contribnum{1} We formulate an iterative warm-start posterior sampling procedure that targets an intermediate posterior time-marginal instead of initializing the reverse process from pure noise. \contribnum{2} We propose a sparse guidance strategy that applies \textit{lightweight} measurement corrections only at selected noise levels, \textit{thereby avoiding vector-Jacobian products through the denoiser}. \contribnum{3} We show that \textsc{Spin} \textit{reduces inference time and memory usage} while maintaining high reconstruction quality across linear and nonlinear inverse problems. 

\section{Background}
\vspace{-5pt}
\label{sec:background}

Denoising diffusion models (DDMs)~\citep{ddpm} generate samples from a target distribution $p_0$ by constructing a distribution interpolation path $(p_t)_{t \in [0,1]}$ that connects the data distribution $p_0$ to the standard Gaussian base distribution $p_1 := \mathcal{N}(0, I_d)$. This interpolation is defined through a noising process where each intermediate distribution $p_t = \text{Law}(X_t)$ arises from the mixture
\begin{equation}
    X_t = \alpha_t X_0 + \sigma_t X_1, \quad X_0 \sim p_0, \quad X_1 \sim p_1,
    \label{eq:forward_process}
\end{equation}
with independent random variables $X_0 \sim p_0$ and $X_1 \sim p_1$. The coefficients $(\alpha_t)_{t\in[0,1]}$ and $(\sigma_t)_{t\in[0,1]}$ are deterministic schedules that satisfy monotonicity constraints and boundary conditions $(\alpha_0, \sigma_0) := (1, 0)$ and $(\alpha_1, \sigma_1) := (0, 1)$. Common instantiations include the \emph{variance-preserving schedule} where $\alpha_t^2 + \sigma_t^2 = 1$~\citep{ddpm, dhariwal2021diffusion}, and the \emph{linear schedule} with $(\alpha_t, \sigma_t) = (1-t, t)$~\citep{lipman2022flow, esser2024scaling, gao2024diffusion}. 

To sample from $p_0$, the diffusion model starts from a sample drawn from $p_1$ at time $t_K = 1$. It then performs a sequence of deterministic reverse transitions, progressively denoising each sample $\hat{X}_{t_{k+1}}$ into a less noisy sample $\hat{X}_{t_k}$, until reaching the clean distribution at $t_0 = 0$. The DDIM~\citep{ddim} transition from $t_{k+1}$ to $t_k$ is defined as
$$ 
\hat{X}_{t_k} = \acp_{t_k} \denoiser{\tk{k+1}}{}{\hat{X}_\tk{k+1}}[\param] + \std_\tk{k} \noisepred{\tk{k+1}}{}{\hat{X}_\tk{k+1}}[\param],
$$  
where $\denoiser{\tk{k+1}}{}{}[\param]$ and $\noisepred{\tk{k+1}}{}{}[\param]$ are respectively neural network approximations of the denoiser $\pE[X_0 | X_t]$ and noise prediction $\pE[X_1 | X_t]$ under the model \eqref{eq:forward_process}. Under the same forward model, the noise prediction can be recovered from the denoiser through $\pE[X_1 | X_t] = (X_t - \acp_t \pE[X_0 | X_t]) / \sigma_t$. Thus, it is enough to train the denoisers $(\denoiser{t}{}{}[\param])_{t \in [0, 1]}$ using standard denoising objectives~\citep{ddpm,ddim}, and the corresponding noise predictor is defined as $\noisepred{t}{}{\bx}[\param] = (\bx - \acp_t \denoiser{t}{}{\bx}[\param]) / \sigma_t$. The DDIM framework also allows for stochastic updates. In this case, 
\begin{equation*}
\hat{X}_{t_k} = \acp_{t_k} \denoiser{\tk{k+1}}{}{\hat{X}_\tk{k+1}}[\param] + \sqrt{\std^2_\tk{k} - \eta^2 _\tk{k}} \noisepred{\tk{k+1}}{}{\hat{X}_\tk{k+1}}[\param] + \eta_\tk{k} W_k,
\end{equation*}
where $W_k \sim \gauss(\zero, \Id)$ and $0 \leq \eta_\tk{k} \leq \sigma_\tk{k}$. We write $\pdata{\tk{k}\tbar \tk{k+1}}{\bx_\tk{k+1}}{}[\param]$ to refer to the associated conditional density. 

\paragraph{Inverse problems.}  We consider an unknown clean image $\bx_* \in \mathbb{R}^d$ observed through a measurement $\obs = \mathcal{A}(\bx_*) + \mathbf{n}$, where $\mathcal{A} : \mathbb{R}^d \rightarrow \mathbb{R}^m$ may be nonlinear. In the formal model below, we take the noise to be Gaussian, $\mathbf{n} = \sigma_y W$, so that
\begin{equation*}
\obs = \mathcal{A}(\bx_*) + \sigma_y W, \qquad W \sim \mathcal{N}(0, \mathbf{I}_m),
\end{equation*}
where $\sigma_y > 0$ denotes the noise level. The task is to recover a reconstruction $\hat{\bx}$ that explains the observation $(\mathcal{A}(\hat{\bx}) \approx \obs)$ while remaining consistent with the statistical properties of natural images. Within a Bayesian framework, the prior distribution $\pdata{0}{}{}$ captures encoded natural-image statistics, and the measurement process is modeled by the Gaussian likelihood
$ \gauss(\obs; \mathcal{A}(\bx), \sigma_y^2 \mathbf{I}_m),$ 
where $\sigma_y$ governs the trade-off between measurement fidelity and regularization. The posterior distribution over encoded images is then given by
$ \pi_0(\bx|\obs) \propto \ell_0(\obs|\bx)\, \pdata{0}{}{\bx}, $ 
where $\ell_0(\obs|\bx) = \gauss(\obs; \mathcal{A}(\bx), \sigma_y^2 \mathbf{I}_m)$.

\paragraph{Posterior sampling.}
The posterior $\pi_0(\cdot|\obs)$ defines a family of conditional time-marginals by pushing samples through the same forward noising kernel as in \eqref{eq:forward_process}. For any $t \in [0,1]$, let
\begin{equation*}
    \pi_t(\bx_t|\obs) = \int p_t(\bx_t|\bx_0)\,\pi_0(\bx_0|\obs)\,\mathrm{d}\bx_0, \qquad p_t(\bx_t|\bx_0)=\gauss(\bx_t;\alpha_t\bx_0,\sigma_t^2\Id).
\end{equation*}
Equivalently, this conditional marginal can be viewed as the unconditional diffusion marginal reweighted by a propagated likelihood,
\begin{equation}
    \label{eq:posterior_time_marginal_reweighting}
    \pi_t(\bx_t|\obs) \propto \ell_t(\obs|\bx_t)\,\pdata{t}{}{\bx_t},
    \qquad \ell_t(\obs|\bx_t) = \int \ell_0(\obs|\bx_0)\,p_{0|t}(\bx_0|\bx_t)\,\mathrm{d}\bx_0,
\end{equation}
where $p_{0|t}(\bx_0|\bx_t) \propto \pdata{0}{}{\bx_0}\,p_t(\bx_t|\bx_0)$ is the backward kernel induced by the unconditional diffusion prior. This form is useful because it isolates the propagated-likelihood term that posterior-sampling methods must approximate. It also shows that exact posterior sampling can be interpreted as transporting samples along the conditional path $(\pi_t(\cdot|\obs))_{t\in[0,1]}$, rather than along the unconditional path $(p_t)_{t\in[0,1]}$. 

The key insight is to modify the reverse sampling process on-the-fly using likelihood information that steers generated samples toward the observational constraints encoded in $\ell_0(\obs|\cdot)$ while keeping them plausible under the prior $p_0$. This approach treats inverse problems through the Bayesian lens presented in the previous section; we seek to perform approximate sampling from a posterior distribution that conditions on the observation $\obs$. Achieving this with a diffusion model requires access to the denoiser conditioned on the observation $\obs$. Following~\citep[Equations 2.15, 2.17]{daras2024survey}, this posterior denoiser can be written as
\begin{equation}
    \label{eq:posterior-denoiser}
\pE[X_0 | X_t, \obs] = \pE[X_0 | X_t] + \acp_t^{-1} \sigma^2 _t \nabla_{\bx_t} \log \ell_t(\obs|\bx_t), 
\end{equation}
with $\ell_t$ defined in \eqref{eq:posterior_time_marginal_reweighting}. The correction term added to the unconditional denoiser is referred to as \emph{guidance}.


Since the pretrained denoiser already approximates $\pE[X_0\mid X_t]$, the remaining challenge in \eqref{eq:posterior-denoiser} is the propagated-likelihood score $\nabla_{\bx_t}\log \ell_t(\obs|\bx_t)$, which is generally intractable. One practical approximation replaces the backward kernel in \eqref{eq:posterior_time_marginal_reweighting} with a point mass at the denoiser prediction, yielding $\ell_t(\obs|\bx_t)\approx \ell_0(\obs|\denoiser{t}{}{\bx_t}[\param])$. When applied at every reverse step, this approximation requires backpropagation through $\denoiser{t}{}{\bx_t}[\param]$, incurring repeated denoiser VJPs. The time-marginal view therefore suggests two opportunities for acceleration: approximate the conditional path at an intermediate marginal $\pi_{t_*}(\cdot|\obs)$ rather than starting from pure noise, and concentrate data-consistency corrections at selected noise levels rather than applying dense guidance along the entire trajectory.

\begin{figure}[t]
    \centering
    \includegraphics[width=\linewidth]{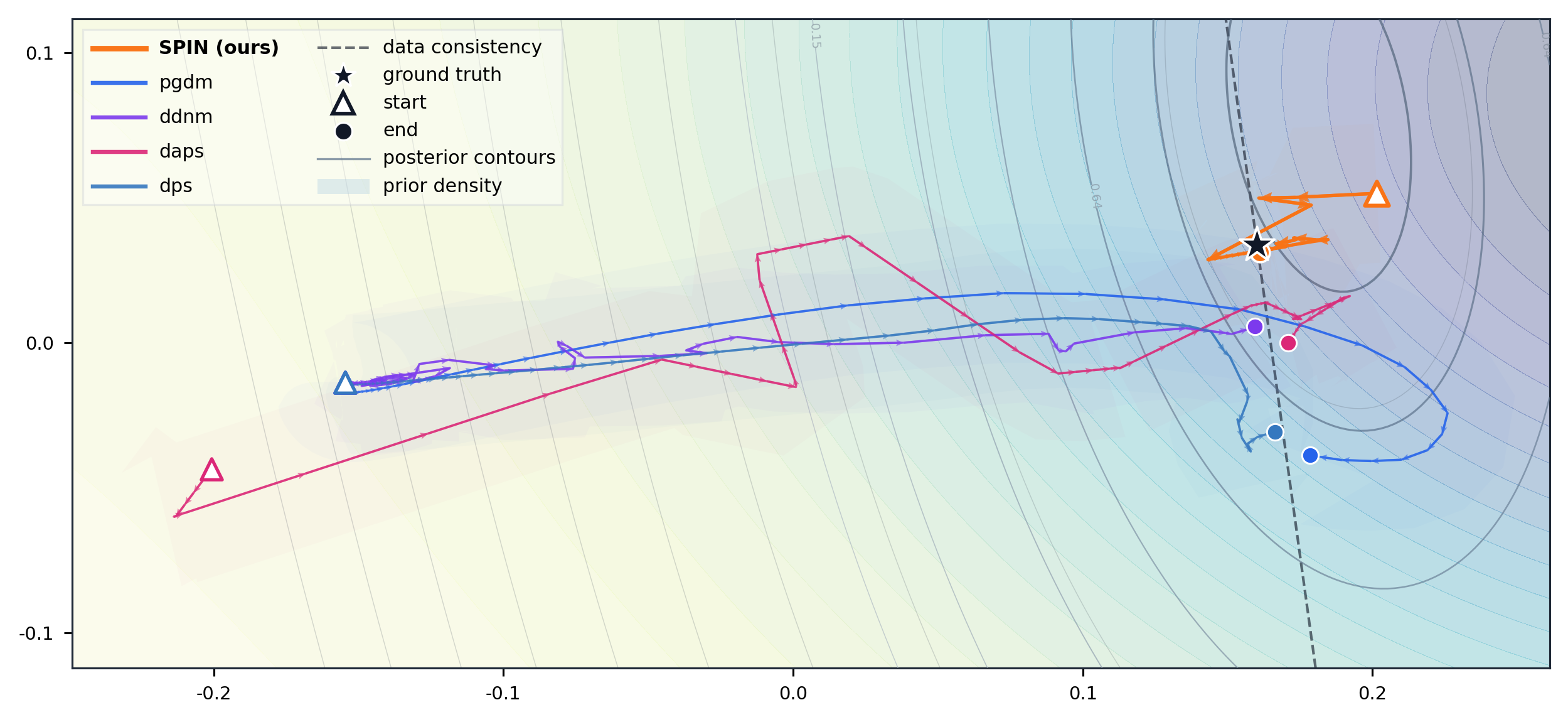}
    \caption{Mean trajectories of posterior samplers on a 2D inverse problem, averaged over $300$ independent runs. \textsc{Spin} first builds a warm start near the data-consistent region, then uses its scheduler to take shorter guided steps toward the posterior mode; the baselines instead follow longer diffusion trajectories from Gaussian noise, leading to more variable paths.}
    \label{fig:real_sampler_trajectory_landscape}
    \vspace{-15pt}
\end{figure}

\section{Method}
\vspace{-5pt}
\label{sec:method}

\paragraph{Motivation.}
We view posterior sampling as a split inference problem: the diffusion model performs prior refinement through ordinary denoising transitions, while data consistency is imposed by explicit optimization in image space. \textsc{Spin} uses this split to avoid spending the full sampling budget on a long reverse trajectory from near-pure noise. It first constructs a warm-start state at an intermediate timestep $t_*$, as illustrated in Figure~\ref{fig:real_sampler_trajectory_landscape}, then runs a truncated reverse trajectory from $t_*$ to $0$ with data-consistency guidance only at a sparse set of scheduled timesteps.

Standard posterior samplers begin from the terminal Gaussian prior and must traverse the full reverse path, which can be expensive. \textsc{Spin} instead uses an initial reconstruction to move the sampler near the data-consistent region at $\tk{*}$, then spends its denoising budget only on the shorter segment $\tk{*}\to0$. The method is therefore useful when the error of constructing the warm start is smaller than the reverse-transition error that would have accumulated over the skipped segment $1\to\tk{*}$. Lemma~\ref{lem:warm_start_truncation} formalizes this trade-off. 

\begin{warmstartlemmabox}
\begin{lemma}[Warm-Start Truncation]
\label{lem:warm_start_truncation}
Let $\mu_{\tk{*}}=\law(\bx^{(N)}_{\tk{*}})$ be the warm-start output law, let $\pi_{\tk{*}}(\cdot\mid\obs)$ denote the target posterior marginal at noise level $\tk{*}$, and let $W_1$ denote the 1-Wasserstein distance between probability laws. Define
\begin{equation}
    \bm{\epsilon_*} \coloneqq W_1\left(\mu_{\tk{*}},\pi_{\tk{*}}(\cdot\mid\obs)\right)
\end{equation}
as the warm-start error. Under a uniform $W_1$-Lipschitz assumption on the approximate reverse kernels (Assumption~\ref{ass:warm_start_stability} in Appendix~\ref{sec:warm_start_error}), let $L$ be the cumulative Lipschitz constant of those kernels along the reverse segment $\tk{*}\to0$, and let $\bm{e_{tr}}$ and $\bm{e_{sk}}$ be the local transition errors accumulated along the truncated reverse segment $\tk{*}\to0$ and the skipped reverse segment $1\to\tk{*}$, respectively, each propagated to time~$0$. Let $\mu_0$ be the output law after the truncated reverse segment. Then
\begin{equation}
    W_1\left(\mu_0,\pi_0(\cdot\mid\obs)\right)
    \le
    L\cdot\bm{\epsilon_*} + \bm{e_{tr}}.
\end{equation}
Running the same approximate reverse sampler from the terminal Gaussian marginal $\pi_1(\cdot\mid\obs)=\gauss(\zero,\Id)$ gives a full-sampler law $\nu_0$ satisfying
\begin{equation}
    W_1\left(\nu_0,\pi_0(\cdot\mid\obs)\right)
    \le
    \bm{e_{tr}} + \bm{e_{sk}}.
\end{equation}
Thus warm starting improves the bound when the propagated warm-start error is smaller than the transition error avoided by skipping the high-noise segment. The warm-started sampler has the tighter upper bound whenever
\begin{equation}
    L\cdot\bm{\epsilon_*} < \bm{e_{sk}}.
\end{equation}
In particular, an exact warm start at $\tk{*}$ avoids the entire accumulated transition error over the skipped reverse segment $1\to\tk{*}$. Proof is given in Appendix~\ref{sec:warm_start_error}.
\end{lemma}
\end{warmstartlemmabox}

\paragraph{Where baselines fail.}
Figure~\ref{fig:real_sampler_trajectory_landscape} shows that the warm-start--truncate split is a structural choice rather than a numerical tweak. Even with an \emph{exact} score (Appendix~\ref{sec:fig2-extended}), the baselines fall into one of two failure modes: drifting along the data-consistency line into low-density posterior regions, or being pulled toward a posterior mode at the cost of measurement fit. We refer to the fixed-$t_*$ warm-start construction as Phase~1 and the subsequent truncated guided reverse process as Phase~2. \textsc{Spin} escapes both because the conflict is resolved \emph{at $\tk{*}$, not along the reverse path}: Phase~1 absorbs data consistency where the prior is still broad, leaving Phase~2 only a short, sparsely guided segment. The condition $L\bm{\epsilon_*}<\bm{e_{sk}}$ of Lemma~\ref{lem:warm_start_truncation} is therefore realized empirically \emph{in the absence of model-approximation noise} --- the advantage is intrinsic to the trajectory architecture, not an artifact of score quality.

\begin{wrapfigure}[15]{r}{0.56\linewidth}
\vspace{-2em}
\refstepcounter{algorithm}\label{alg:spin-highlevel}
\begin{tcolorbox}[
    colback=orange!5!white,
    colframe=orange!5!white,
    boxrule=0pt,
    arc=3pt,
    left=2pt,
    right=2pt,
    top=3pt,
    bottom=3pt,
    before upper=\setlength{\parindent}{0pt}]
\rule{\linewidth}{0.8pt}\\[2pt]
{\small\bfseries Algorithm~\thealgorithm\enspace \textsc{Spin} --- \texttt{Pseudocode}}\\[-2pt]
\rule{\linewidth}{0.3pt}
\small
{\algrenewcommand\alglinenumber[1]{\scriptsize\textcolor{black!40}{#1}}%
\begin{algorithmic}[1]
\renewcommand{\algorithmiccomment}[1]{\hfill#1}
\For{$N$} \Comment{\textcolor{orange}{\scriptsize\texttt{// warm-start; fresh-noise proxy}}}
    \State $\hat{\bx}_0 \gets \denoiser{\bm{t_*}}{}{\bx_{\bm{t_*}}}[\param]$
    \State $\bx_0^* \gets \textsf{Optimize}\big(\hat{\bx}_0 \mapsto \|\obs - \opA(\hat{\bx}_0)\|^2\big)$
    \State $\bx_{\bm{t_*}} = \alpha_{\bm{t_*}}\bx_0^* + \sigma_{\bm{t_*}}\boldsymbol{\epsilon},\quad \boldsymbol{\epsilon}\sim\gauss(\mathbf{0},\Id)$
\EndFor
\For{$k = M, \ldots, 1$} \Comment{\textcolor{orange}{\scriptsize\texttt{// guided reverse}}}
    \State $\tilde{\bx}_{k} \sim \pdata{k\tbar k+1}{\bx_{k+1}}{}[\param]$
    \State $\hat{\bx}_0 \gets \denoiser{k}{}{\tilde{\bx}_{k}}[\param]$
    \State $\bx^*_0 \gets \textsf{Optimize}\big(\bx \mapsto \|\obs - \opA(\bx)\|^2 + \lambda\|\bx - \hat{\bx}_0 \|^2\bigl)$
    \State $\bx_{k+1} = \alpha_{k+1}\bx^*_0 + \sigma_{k+1}\alpha_{k+1}\epsilon_{k}$
    \State $\bx_{k} \gets \textsf{DDIM} \bigl(\bx_{k+1},\;k+1 \xrightarrow{n} k\bigr)$
\EndFor
\State \Return $\denoiser{0}{}{\bx_{1}}[\param]$
\end{algorithmic}
}
\rule{\linewidth}{0.8pt}
\end{tcolorbox}
\end{wrapfigure}

\vspace{-5pt}
\paragraph{Algorithm.}
The remaining part of this section describes how \textsc{Spin} implements the two quantities in this bound: Phase~1 targets $\pi_{\tk{*}}(\cdot\mid\obs)$ to make $\bm{\epsilon_*}$ small, while Phase~2 realizes the truncated reverse sampler whose accumulated error is $\bm{e_{tr}}$. Algorithm~\ref{alg:spin-highlevel} shows simplified pseudocode. The methods described in this section assume pixel space, while the full latent-space algorithm is given in Appendix~\ref{sec:algorithm}. 

\vspace{-5pt}
\paragraph{Warm-start phase.} This phase tries to produce a noisy image-space state that is close to the posterior marginal $\pi_{\tk{*}}(\cdot\mid\obs)$. In the ideal case, we would first draw a clean posterior sample $X^\obs_0 \sim \pi_0(\cdot\mid\obs)$ and then add forward noise,
$X^\obs_{\tk{*}} = \alpha_{\tk{*}} X^\obs_0 + \sigma_{\tk{*}} \epsilon$ with $\epsilon\sim\gauss(\zero,\Id)$.
This is not available in practice, because drawing $X^\obs_0$ is the original posterior sampling problem. Instead, we start from a simple observation-based initializer $\bx_{\mathrm{init}}(\obs)$, noise it to timestep $\tk{*}$:
\begin{equation*}
    \bx^{0}_{\tk{*}} \sim \gauss(\alpha_{\tk{*}}\bx_{\mathrm{init}}(\obs), \sigma^2_{\tk{*}}\Id).
\end{equation*}
For $N$ iterations at fixed timestep $\tk{*}$, we denoise $\bx^{k}_{\tk{*}}$ to $\hat{\bx}^{k}_0$, and solve the unregularized data-consistency problem
$\bx_0^k \in \argmin_{\bx}\; \|\obs-\opA(\bx)\|^2$,
initialized at $\hat{\bx}^{k}_0$. Algorithm~\ref{alg:spin-highlevel} displays the return to $\tk{*}$ as the familiar fresh-noise update $\bx_{\tk{*}}=\alpha_{\tk{*}}\bx^*_0+\sigma_{\tk{*}}\epsilon$ to keep the high-level loop operational. In the method, we use the conservative noise-preserving version of this step: it reuses the current noise estimate and injects only residual fresh noise. Concretely, we renoise the result $\bx^k_0$ back to $\tk{*}$ to obtain the next warm-start iterate,
\begin{equation}
    \label{eq:phase1_renoise}
    \bx^{k+1}_{\tk{*}} \sim \gauss(\alpha_{\tk{*}}\bmu^{k}, \sigma^2_{\tk{*}}\sigma^{k}),
\end{equation}
where $\bmu^{k} := \alpha_{\tk{*}}\bx^{k}_0 + \sigma_{\tk{*}}\hat{\bx}^{k}_0 + \sigma_{\tk{*}}\epsilon^{k}, \qquad \sigma^{k} := \sigma^2_{\tk{*}}\Id, \qquad \epsilon^{k} := (\bx^{k}_{\tk{*}}-\alpha_{\tk{*}}\hat{\bx}^{k}_0)/\sigma_{\tk{*}}$.
The repeated denoising/re-noising loop supplies the prior stabilization, so Phase~1 can use this unanchored correction. Its output $\bx^{(N)}_{\tk{*}}$ initializes the guided reverse process.
\vspace{-5pt}
\paragraph{Sparse guided denoising phase.}
Given the warm start, the second phase runs only the shortened reverse trajectory. Let $\{\tk{1},\dotsc,\tk{M}\}$ be the guidance timesteps, with $\tk{M}=\tk{*}$ and $\tk{1}\approx 0$. The scheduler is the rule that constructs this ordered grid from the fixed guidance budget $M$. We assign a schedule density over the remaining noise interval and place the $M$ guidance points according to its cumulative mass: uniform density gives equally spaced corrections, while our default Gaussian density, with center $\mu_G$ and width $\sigma_G$, clusters corrections in the intermediate-noise region. Thus $M$ controls how often we solve the data-consistency objective, and the schedule parameters control where those solves occur; all intermediate transitions between scheduled points are ordinary DDIM denoising steps without optimization. This sparse schedule is a key design choice: instead of applying data consistency at every reverse step, \textsc{Spin} allocates the guidance budget to timesteps where the denoiser prediction is informative but still has enough remaining trajectory to absorb optimization artifacts. Starting from $\bx^\obs_{\tk{M}}=\bx^{(N)}_{\tk{*}}$, each guidance step first probes the lower-noise state by sampling
$\tilde{\bx}_{\tk{k}}\sim \pdata{\tk{k}\tbar \tk{k+1}}{\bx^\obs_{\tk{k+1}}}{}[\param]$
and denoising it to $\hat{\bx}_0=\denoiser{\tk{k}}{}{\tilde{\bx}_{\tk{k}}}[\param]$. We then solve the anchored data-consistency problem
\begin{equation*}
    \bx^*_0 \in \argmin_{\bx}\; \|\obs-\opA(\bx)\|^2 + \lambda\|\bx-\hat{\bx}_0\|^2
\end{equation*}
initialized at $\hat{\bx}_0$. Algorithm~\ref{alg:spin-highlevel} writes the return to $\tk{k+1}$ as a deterministic mean update using the current noise estimate $\epsilon_{\tk{k}}$. In practice, we use a conservative stochastic version of this update, which keeps this preserved-noise mean and adds only residual fresh noise. Specifically, we renoise the result $\bx^*_0$ back to $\tk{k+1}$ to obtain the next guided denoising iterate,
\begin{equation}
    \label{eq:phase2_renoise}
    \tilde{\bx}_{\tk{k+1}} \sim \gauss(\bmu_{\tk{k+1}}, \sigma^2_{\tk{k+1}}\sigma'_{\tk{k+1}}),
\end{equation}
with $\bmu_{\tk{k+1}} := \alpha_{\tk{k+1}}\bx^*_0 + \sigma_{\tk{k+1}}\alpha_{\tk{k+1}}\epsilon_{\tk{k}}, \qquad \sigma'_{\tk{k+1}} := \sigma^2_{\tk{k+1}}\Id, \qquad \epsilon_{\tk{k}} := (\tilde{\bx}_{\tk{k}}-\alpha_{\tk{k}}\hat{\bx}_0)/\sigma_{\tk{k}}$.
Finally, we run $n$ steps in DDIM style from $\tk{k+1}$ to $\tk{k}$ to obtain the accepted state $\bx^\obs_{\tk{k}}$. Thus the lower-noise sample acts only as a guidance probe: the method corrects the clean component, re-noises it with preserved noise, and lets the diffusion dynamics absorb any artifacts before the next scheduled correction. Appendix~\ref{sec:hyperparameters} reports the task-specific choices of $t_*$, $N$, $M$, and optimizer budgets.

\begin{callout}
\textbf{Takeaway:} \textsc{Spin} \contribnum{1} first builds an iterative warm start at $\bm{t_*}$, \contribnum{2} then runs a truncated reverse trajectory $\bm{t_*\to0}$ with sparse data-consistency corrections. These corrections avoid propagated-likelihood scores and denoiser VJPs. Conservative re-noising returns corrected estimates to the diffusion trajectory while preserving accumulated information.
\end{callout}

\vspace{-5pt}
\section{Related Work}
\vspace{-5pt}
\begin{table}[H]
    \centering
    \caption{Comparison of training-free diffusion/consistency/flow inverse-problem solvers. VJP: backprop through denoiser $\denoiser{\cdot}{}{\cdot}$ and encoder/decoder $\encoder$/$\decoder$ ($^\dagger$latent models only). DC Guidance Density: frequency of data-consistency or likelihood guidance application.}
    \label{tab:method_comparison}
    \fontsize{7.3}{7.5}\selectfont
    \setlength{\tabcolsep}{2pt}
    \renewcommand{\arraystretch}{0.88}
\resizebox{\textwidth}{!}{%
\begin{tabular}{@{}
	        >{\raggedright\arraybackslash}p{2.5cm}
	        >{\raggedright\arraybackslash}p{5.5cm}
	        >{\raggedright\arraybackslash\leavevmode}p{2.5cm}
	        >{\centering\arraybackslash\leavevmode}p{1.2cm}
	        >{\centering\arraybackslash\leavevmode}p{1.4cm}
	        >{\raggedright\arraybackslash\leavevmode}p{2.0cm}
	        @{}}
	    \toprule
	    \textbf{Method} & \textbf{Inference Objective} & \textbf{\shortstack{Guidance Density}} & \textbf{\shortstack{VJP\\$\denoiser{\cdot}{}{\cdot}[\param]$?}} & \textbf{\shortstack{VJP\\$\encoder$/$\decoder$?}} & \textbf{\shortstack{Assumptions\\on $\mathcal{A}$}} 
	    \\
	    \midrule
	    \textsc{DPS}, \textsc{PGDM}&
	    Score-based posterior approximation &
	    {\color{red}Dense} &
	    {\color{red}Yes} & {\color{red}Yes}$^\dagger$ &
	    {\color{Green}General} / known $h^\dagger$ 
	    \\
	    \textsc{DAPS}  &
	    Decoupled posterior annealing &
	    {\color{orange}Moderate} + inner MCMC &
	    {\color{Green}No} & {\color{orange}Often}$^\dagger$ &
	    {\color{Green}General} 
	    \\
	    \textsc{RED-Diff}, \textsc{DDS} &
	    Variational / Krylov data consistency &
	    {\color{red}Dense} / iterative solve &
	    {\color{Green}No} & {\color{Green}No} &
	    {\color{orange}General / linear} 
	    \\
	    \textsc{DiffPIR}, \textsc{PnP-DM} &
	    Plug-and-play / proximal posterior &
	    {\color{red}Dense} + inner loops &
	    {\color{Green}No} & {\color{Green}No} &
	    {\color{Green}General} 
	    \\
	    \textsc{PSLD}, \textsc{ReSample}, \newline \textsc{P2L}, \textsc{TReg}, \textsc{LATINO-PRO} &
	    Latent/text/consistency-model inverse solvers &
	    {\color{orange}Moderate--dense} + inner solve &
	    {\color{orange}Varies} & {\color{orange}Varies}$^\dagger$ &
	    {\color{orange}Varies} 
	    \\
	    \textsc{FlowDPS}, \textsc{PnP-Flow}, \newline\textsc{FLAIR} &
	    Flow / rectified-flow prior guidance &
	    {\color{red}Dense} / iterative solve &
	    {\color{orange}Varies} & {\color{orange}Often}$^\dagger$ &
	    {\color{orange}General / linear} 
	    \\
	    \textsc{CCDF} &
	    Forward-diffused warm initialization &
	    {\color{red}Dense} &
	    {\color{Green}No} & {\color{Green}No} &
	    {\color{orange}Linear + non-expansive} 
	    \\
	    \textsc{DDNM}, \textsc{DDRM} &
	    Subspace projection &
	    {\color{red}Dense} &
	    {\color{Green}No} & {\color{Green}No} &
	    {\color{red}Linear} + pseudo-inverse/SVD 
	    \\
	    \textsc{Spin} (ours) &
	    Sparse time-marginal guidance &
	    {\color{Green}Sparse} &
	    {\color{Green}No} & {\color{Green}No} &
	    {\color{Green}General} 
	    \\
	    \bottomrule
	    \end{tabular}%
}
\vspace{-10pt}
\end{table}

\textbf{Likelihood-score guidance with VJPs.}
Diffusion Posterior Sampling (DPS)~\citep{chung2022diffusion} and Pseudoinverse-Guided Diffusion Models (PGDM)~\citep{song2023pseudoinverse} approximate the intractable posterior score in \eqref{eq:posterior-denoiser} to guide reverse diffusion, but require vector--Jacobian products (VJPs) through the denoiser (and through the decoder for latent diffusion models), incurring large memory overhead and slow inference. Several works reduce this cost without using denoiser VJPs, e.g., decoupled posterior annealing (DAPS)~\citep{zhang2025improving}, variational reformulations (RED-Diff)~\citep{mardani2023variational}, and Krylov/CG data-consistency updates (DDS)~\citep{chung2024decomposed}. \textsc{Spin} likewise avoids denoiser VJPs but additionally avoids encoder/decoder VJPs and applies data consistency only at sparsely scheduled timesteps.

\textbf{Gradient-free and plug-and-play solvers.}
DDNM~\citep{wang2022zero} and DDRM~\citep{kawar2022denoising} avoid VJPs through closed-form projection updates, but are restricted to linear operators: DDNM requires an explicit pseudo-inverse, while DDRM uses an SVD-based spectral decomposition. Plug-and-play and RED methods use off-the-shelf denoisers as implicit priors~\citep{venkatakrishnan2013plugandplay,chan2016plugandplay,romano2017little}; DiffPIR~\citep{diffpir} and PnP-DM~\citep{wu2024principled} adapt the denoise/data-consistency decomposition to diffusion priors. \textsc{Spin} is closest to these alternating schemes, but uses only gradients through $\mathcal{A}$, applies data consistency sparsely, and warm-starts a truncated reverse trajectory rather than running dense conditional reverse diffusion.

\textbf{Warm starts.}
CCDF~\citep{chung2022comecloser} relies on an external pretrained initializer network and is tailored to task-specific conditional updates: it forward-diffuses an initial reconstruction and then runs dense conditional reverse diffusion, with contraction guarantees for non-expansive linear consistency maps. SDEdit~\citep{meng2021sdedit} denoises from a noisy guide image, but does not iteratively refine an inverse-problem warm start at a fixed intermediate noise level. \textsc{Spin} is a general inverse-problem solver that starts from a simple observation-based initialization, constructs the intermediate state by repeated fixed-$t_*$ denoise--optimize--re-noise steps, and then applies sparse scheduled guidance along a truncated DDIM trajectory without requiring an additional learned model.

Table~\ref{tab:method_comparison} summarizes the main differences in targets, gradient requirements, and transition designs. Additional discussion of latent inverse solvers (PSLD, ReSample, P2L, TReg, LATINO-PRO), other warm-start strategies (SDEdit, ILVR), and consistency / flow-based guidance (CM4IR, FlowDPS, PnP-Flow, FLAIR) is deferred to Appendix~\ref{sec:related_work_ext}.

\vspace{-5pt}
\section{Experiments}
\vspace{-5pt}
\label{sec:experiments}

\begin{table*}[t]
    \centering
    \caption[FFHQ LPIPS comparison]{Mean LPIPS for linear/nonlinear imaging tasks on the \textbf{\texttt{FFHQ}} datasets with $\sigma_y = 0.05$. Lower metrics are better. HDR denotes High Dynamic Range. \textcolor{RoyalBlue}{Blue} factors are relative to \textsc{Spin}; larger values mean greater time or memory savings.}
    \label{tab:ffhq}
    \fontsize{7.3}{8.5}\selectfont
    \setlength{\tabcolsep}{2.8pt}
    \begin{tabular}{lcccccccc}
    \toprule
    \multirow{2}{*}{} & \multicolumn{8}{c}{$\downarrow$ \textbf{LPIPS FFHQ}} \\
    \cmidrule(lr){2-9}
    \textbf{Task}                     & \multicolumn{1}{c}{\textsc{Spin}}       & \multicolumn{1}{c}{\textsc{DAPS}} 
          & \multicolumn{1}{c}{\textsc{Red-Diff}} 
           & \multicolumn{1}{c}{\textsc{PNP-DM}} 
            & \multicolumn{1}{c}{\textsc{DPS}} 
               & \multicolumn{1}{c}{\textsc{DDNM}} 
                & \multicolumn{1}{c}{\textsc{PGDM}} 
                  & \multicolumn{1}{c}{\textsc{Diffpir}} 
                   \\
    \midrule
    Gaussian Deblur                    & \tabhi{BurntOrange!40}{\tabval{0.17}{0.06}}    & \tabval{0.19}{0.06}  & \tabval{0.25}{0.07}  & \tabval{0.20}{0.07}  & \tabhi{BurntOrange}{\tabbest{0.16}{0.05}}  & \tabval{0.84}{0.04}  & \tabval{0.87}{0.14}  & --              \\
    Motion Deblur                      & \tabhi{BurntOrange}{\tabbest{0.13}{0.04}}    & \tabhi{BurntOrange!40}{\tabval{0.19}{0.06}}  & \tabval{0.20}{0.06}  & \tabval{0.21}{0.06}  & \tabval{0.21}{0.06}  & --               & --               & --              \\
    SR ($\times$4)                     & \tabhi{BurntOrange!40}{\tabval{0.20}{0.07}}    & \tabhi{BurntOrange}{\tabbest{0.19}{0.06}}  & \tabval{0.44}{0.08}  & \tabval{0.21}{0.06}  & \tabval{0.22}{0.07}  & \tabval{0.36}{0.08}  & \tabval{0.30}{0.07}  & --              \\
    SR ($\times$16)                    & \tabhi{BurntOrange}{\tabbest{0.35}{0.08}}    & \tabval{0.45}{0.10}  & \tabval{0.59}{0.08}  & \tabval{0.60}{0.15}  & \tabhi{BurntOrange!40}{\tabval{0.36}{0.08}}  & \tabval{0.52}{0.09}  & \tabval{0.42}{0.06}  & --              \\
    Box Inpainting                     & \tabhi{BurntOrange!40}{\tabval{0.14}{0.04}}    & \tabhi{BurntOrange}{\tabbest{0.12}{0.04}}  & \tabval{0.18}{0.05}  & \tabval{0.55}{0.07}  & \tabval{0.20}{0.08}  & \tabval{0.18}{0.05}  & \tabval{0.17}{0.04}  & \tabhi{BurntOrange!40}{\tabval{0.14}{0.04}} \\
    Half Inpainting                    & \tabhi{BurntOrange}{\tabbest{0.24}{0.06}}    & \tabhi{BurntOrange}{\tabbest{0.24}{0.06}}  & \tabval{0.29}{0.07}  & \tabval{0.60}{0.04}  & \tabval{0.26}{0.06}  & \tabval{0.28}{0.06}  & \tabhi{BurntOrange!40}{\tabval{0.25}{0.05}}  & \tabhi{BurntOrange!40}{\tabval{0.25}{0.05}} \\
    JPEG (QF=2)                        & \tabhi{BurntOrange}{\tabbest{0.16}{0.06}}    & \tabhi{BurntOrange!40}{\tabval{0.22}{0.07}}  & \tabval{0.32}{0.08}  & \tabval{0.27}{0.07}  & \tabval{0.28}{0.07}  & --               & --               & --              \\
    Phase Retrieval                    & \tabhi{BurntOrange!40}{\tabval{0.35}{0.24}}    & \tabhi{BurntOrange}{\tabbest{0.30}{0.25}}  & \tabhi{BurntOrange!40}{\tabval{0.35}{0.24}}  & \tabval{0.44}{0.23}  & \tabval{0.48}{0.18}  & --               & --               & --              \\
    HDR                                & \tabhi{BurntOrange!40}{\tabval{0.13}{0.05}}    & \tabhi{BurntOrange}{\tabbest{0.08}{0.05}}  & \tabval{0.19}{0.06}  & \tabval{0.18}{0.07}  & \tabval{0.64}{0.36}  & --               & --               & --              \\
    \midrule
    \textbf{Memory (MB)}               & \tabhi{BurntOrange}{\textbf{1983}}               & \factorcell{2095}{1.1}             & \colorfactorcell{BurntOrange!40}{1985}{1.0}             & \factorcell{1985}{1.0}             & \factorcell{3309}{1.7}             & \factorcell{2019}{1.0}             & \factorcell{3409}{1.7}             & \colorfactorcell{BurntOrange!40}{1985}{1.0}            \\
    \textbf{Runtime (sec)}             & \tabhi{BurntOrange}{\textbf{25}}                 & \factorcell{75}{3.0}               & \factorcell{50}{2.0}               & \factorcell{194}{7.8}              & \factorcell{105}{4.2}              & \colorfactorcell{BurntOrange!40}{47}{1.9}               & \factorcell{101}{4.0}              & \factorcell{50}{2.0}              \\
    \bottomrule
    \end{tabular}
    \vspace{-10pt}
\end{table*}

\textbf{Datasets, Models and Baselines.} We use the pixel-space \texttt{FFHQ} diffusion model of \citet{choi2021ilvr}, the \texttt{ImageNet} model of \citet{dhariwal2021diffusion}, and the \texttt{FFHQ} latent diffusion model of \citet{rombach2022high}. Following the protocol of \textsc{PnP-DM}~\citep{wu2024principled}, we evaluate on 100 validation samples per task and report mean $\pm$ std for every metric. We use the first 100 samples from \texttt{FFHQ}~\citep{karras2019style} and a random subset of \texttt{ImageNet}~\citep{ILSVRC15} to mitigate class bias. Pixel-space experiments use $256 \times 256$ resolution and we compare against \textsc{DPS}~\citep{chung2022diffusion}, \textsc{PGDM}~\citep{song2023pseudoinverse}, \textsc{DDNM}~\citep{wang2022zero}, \textsc{DiffPIR}~\citep{diffpir}, \textsc{Red-Diff}~\citep{mardani2023variational}, \textsc{DAPS}~\citep{zhang2025improving}, and \textsc{PnP-DM}~\citep{wu2024principled}; for latent-space, against \textsc{DAPS}~\citep{zhang2025improving}, \textsc{ReSample}~\citep{song2023resample}, \textsc{PSLD}~\citep{rout2023solving}, and \textsc{PnP-DM}~\citep{wu2024principled}. Implementation details and hyperparameters for each competitor are in Appendix~\ref{sec:competitors}.

\textbf{Inverse Tasks and Settings.} Following the protocol of \citet{janati2025mixture}, we evaluate at noise level $\sigma_y = 0.05$ on the following tasks. \textit{Linear:} super-resolution ($\times4$, $\times16$); inpainting with a $150 \times 150$ central box mask and a right-half mask; Gaussian and motion deblurring with $61 \times 61$ kernels~\citep{chung2022diffusion}. \textit{Nonlinear:} JPEG dequantization at quality factor $2\%$ via the differentiable operator of \citet{shin2017jpeg}; phase retrieval with $\times2$ oversampling; non-uniform deblurring~\citep{tran2021explore}; and high dynamic range (HDR) reconstruction~\citep{mardani2023variational}.

\begin{table*}[t]
    \centering
    \caption[ImageNet LPIPS comparison]{Mean LPIPS for linear/nonlinear imaging tasks on the \textbf{\texttt{ImageNet}} datasets with $\sigma_y = 0.05$. Lower metrics are better. HDR denotes High Dynamic Range. \textcolor{RoyalBlue}{Blue} factors are relative to \textsc{Spin}; larger values mean greater time or memory savings.}
    \label{tab:imagenet}

    \fontsize{7.3}{8.5}\selectfont
    \setlength{\tabcolsep}{2.8pt}
    \begin{tabular}{lcccccccc}
    \toprule
    \multirow{2}{*}{} & \multicolumn{8}{c}{$\downarrow$ \textbf{LPIPS ImageNet}} \\
    \cmidrule(lr){2-9}

    \textbf{Task}                     & \multicolumn{1}{c}{\textsc{Spin}}          & \multicolumn{1}{c}{\textsc{DAPS}} 
          & \multicolumn{1}{c}{\textsc{Red-Diff}} 
           & \multicolumn{1}{c}{\textsc{PNP-DM}} 
            & \multicolumn{1}{c}{\textsc{DPS}} 
               & \multicolumn{1}{c}{\textsc{DDNM}} 
                & \multicolumn{1}{c}{\textsc{PGDM}} 
                  & \multicolumn{1}{c}{\textsc{Diffpir}} 
                   \\
    \midrule
    Gaussian Deblur          & \tabhi{BurntOrange}{\tabbest{0.34}{0.10}}    & \tabval{0.41}{0.12}  & \tabval{0.44}{0.11}  & \tabval{0.47}{0.13}  & \tabhi{BurntOrange!40}{\tabval{0.35}{0.14}}  & \tabval{0.42}{0.13}  & \tabval{1.06}{0.10}  & --              \\
    Motion Deblur            & \tabhi{BurntOrange}{\tabbest{0.28}{0.09}}    & \tabval{0.40}{0.11}  & \tabhi{BurntOrange!40}{\tabval{0.35}{0.09}}  & \tabval{0.50}{0.13}  & \tabval{0.41}{0.13}  & --               & --               & --              \\
    SR ($\times$4)           & \tabhi{BurntOrange!40}{\tabval{0.39}{0.11}}    & \tabval{0.41}{0.12}  & \tabval{0.53}{0.13}  & \tabval{0.49}{0.14}  & \tabval{0.47}{0.14}  & \tabhi{BurntOrange}{\tabbest{0.27}{0.09}}  & \tabval{0.63}{0.12}  & --              \\
    SR ($\times$16)          & \tabval{0.68}{0.13}    & \tabval{0.80}{0.13}  & \tabval{0.82}{0.11}  & \tabval{0.88}{0.10}  & \tabhi{BurntOrange}{\tabbest{0.58}{0.10}}  & \tabval{0.68}{0.17}  & \tabhi{BurntOrange!40}{\tabval{0.59}{0.08}}  & --              \\
    Box Inpainting           & \tabval{0.32}{0.05}    & \tabhi{BurntOrange!40}{\tabval{0.31}{0.05}}  & \tabval{0.36}{0.06}  & \tabval{0.71}{0.14}  & \tabval{0.46}{0.13}  & \tabhi{BurntOrange}{\tabbest{0.28}{0.06}}  & \tabhi{BurntOrange}{\tabbest{0.28}{0.05}}  & \tabhi{BurntOrange}{\tabbest{0.28}{0.05}}              \\
    Half Inpainting          & \tabval{0.38}{0.08}    & \tabval{0.40}{0.05}  & \tabval{0.46}{0.05}  & \tabval{0.70}{0.12}  & \tabval{0.47}{0.10}  &  \tabval{0.38}{0.07}  & \tabhi{BurntOrange}{\tabbest{0.33}{0.05}}  & \tabhi{BurntOrange!40}{\tabval{0.34}{0.07}}              \\
    JPEG (QF=2)              & \tabhi{BurntOrange}{\tabbest{0.30}{0.11}}    & \tabhi{BurntOrange!40}{\tabval{0.44}{0.12}}  & \tabval{0.47}{0.12}  & \tabval{0.54}{0.13}  & \tabval{0.57}{0.12}  & --               & --               & --              \\
    Phase Retrieval          & \tabval{0.65}{0.14}    & \tabhi{BurntOrange}{\tabbest{0.60}{0.18}}  & \tabhi{BurntOrange!40}{\tabval{0.64}{0.12}}  & \tabval{0.73}{0.11}  & \tabval{0.67}{0.08}  & --               & --               & --              \\
    HDR & \tabhi{BurntOrange!40}{\tabval{0.17}{0.10}}    & \tabhi{BurntOrange}{\tabbest{0.14}{0.11}}  & \tabval{0.21}{0.13}  & \tabval{0.34}{0.18}  & \tabval{0.88}{0.16}  & --               & --               & --              \\
    \midrule
    \textbf{Memory (MB)}               & \tabhi{BurntOrange}{\textbf{4991}}              & \colorfactorcell{BurntOrange!40}{4993}{1.0}             & \factorcell{4995}{1.0}             & \factorcell{4993}{1.0}             & \factorcell{8701}{1.7}             & \factorcell{5031}{1.0}             & \factorcell{8741}{1.8}             & \factorcell{5007}{1.0}              \\
    \textbf{Run time (sec)}            & \tabhi{BurntOrange}{\textbf{83}}                & \factorcell{219}{2.6}              & \colorfactorcell{BurntOrange!40}{165}{2.0}              & \factorcell{636}{7.7}              & \factorcell{360}{4.3}              & \factorcell{296}{3.6}              & \factorcell{371}{4.5}              & \factorcell{169}{2.0}              \\

    \bottomrule
    \end{tabular}
    \vspace{-10pt}
\end{table*}

\begin{figure*}[t] 
    \centering
    \includegraphics[width=0.95\linewidth]{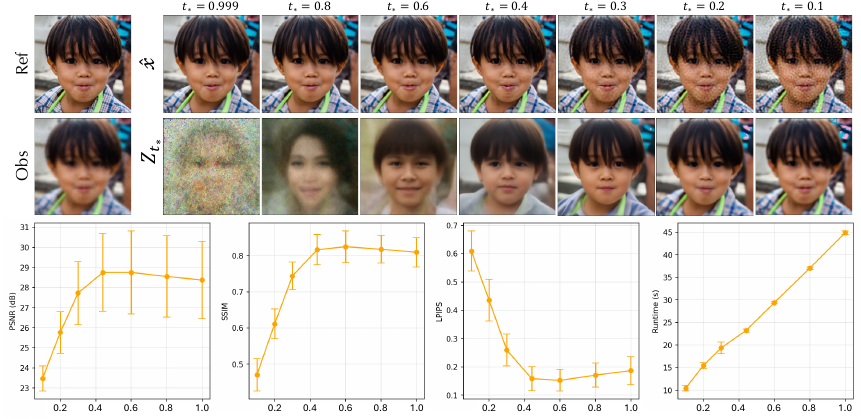}
   \caption{Impact of the initial timestep $t_*$ on \textsc{Spin}. The first row shows reconstruction quality for different $t_*$ values, while the second row shows the corresponding initial guesses. The plots report how quality metrics and runtime vary with $t_*$.}
    \label{fig:ablation}
    \vspace{-17pt}
\end{figure*} 

\subsection{Results}
\vspace{-5pt}

\textsc{Spin} delivers a strong quality--compute trade-off across all three experimental settings: it is competitive with the best-performing baselines in LPIPS while consistently reducing runtime and memory. On pixel-space \texttt{FFHQ}, \textsc{Spin} is best or second-best on most tasks; on \texttt{ImageNet} and latent \texttt{FFHQ}, it remains close to the strongest baselines while providing substantially lower wall-clock cost. The efficiency gains come from the warm start, which skips the high-noise reverse segment from $1$ to $t_*$, and from applying measurement-driven updates only at scheduled high-impact timesteps, avoiding backpropagation through the denoiser (and the encoder--decoder in latent-space models). Concretely, on \texttt{FFHQ} pixel-space, \textsc{Spin} requires only 1983 MB and 25 sec, compared to \textsc{DPS} (3309 MB, 105 sec) and \textsc{PnP-DM} (194 sec). The efficiency gains are largest in the latent setting, where \textsc{Spin} runs over \textcolor{RoyalBlue}{\textbf{$20\times$}} faster than \textsc{ReSample} and over \textcolor{RoyalBlue}{\textbf{$50\times$}} faster than \textsc{DAPS}, while avoiding denoiser and encoder--decoder VJPs.

Tables \ref{tab:ffhq}, \ref{tab:ffhq_latent}, and \ref{tab:imagenet} together with Figure \ref{fig:main_results} provide the quantitative and qualitative comparison on FFHQ and ImageNet; SSIM and PSNR results (Tables \ref{tab:imagenet_ssim}, \ref{tab:imagenet_psnr}) and additional samples are reported in Appendix \ref{sec:additional-experiments}. We report LPIPS~\citep{zhang2018unreasonable} as our primary metric, with PSNR and SSIM provided in the supplementary material since they perform pixel-wise comparisons and favor overly smooth images (see \citet{janati2025mixture}, Appendix B.6).
The \textbf{best performing method} in each row is highlighted as \colorbox{BurntOrange}{\phantom{XX}} and the \textbf{$2^{\text{nd}}$ best} as \colorbox{BurntOrange!40}{\phantom{XX}}; '--' denotes a baseline that is not applicable or not working on the given task.

\textbf{Hyperparameters.} For the warm-start we set the initial timestep $t_*$ depending on the task family: deblurring tasks use $t_* \in [0.44, 0.50]$, super-resolution uses $t_* \in [0.50, 0.60]$, and inpainting uses larger values $t_* \in [0.70, 0.80]$. We use $N=10$ warm-start iterations and $\eta_{\text{init}} = 10^{-4}$ as the default pixel-space setting, with each iteration performing $N_{\texttt{OPT}}^{\texttt{int}} = 50$ gradient steps on the observation likelihood; more details on the settings are reported in Appendix~\ref{sec:hyperparameters}.

For the guided denoising phase (Phase 2), our default pixel-space setting uses $M=30$ scheduled guidance timesteps, a Gaussian schedule, $\eta_{\text{main}} = 10^{-3}$, and $\lambda=0$ in the pixel-space objective. We use $G_{\texttt{OPT}}=50$ as a representative guidance budget, with easier or harder tasks adjusting only $G_{\texttt{OPT}}$ and, in a few cases, $M$; the full task-specific values are reported in Appendix~\ref{sec:hyperparameters}.

\begin{table*}[t]
    \centering
    \begin{minipage}[t]{0.60\textwidth}%
        \vspace{0pt}
        \centering
        \begingroup
    \fontsize{6}{6.0}\selectfont
    \setlength{\tabcolsep}{1.2pt}
    \def\latentnumfont{\fontsize{6.45}{6.2}\selectfont}
    \resizebox{\linewidth}{!}{%
    \begin{tabular}{l>{\latentnumfont}c>{\latentnumfont}c>{\latentnumfont}c>{\latentnumfont}c>{\latentnumfont}c}
    \toprule
    \multirow{2}{*}{} & \multicolumn{4}{c}{$\downarrow$\textbf{LPIPS FFHQ Latent}} \\
    \cmidrule(lr){2-6}
    \text{Task}                      & \multicolumn{1}{c}{\textsc{Spin}}                              & \multicolumn{1}{c}{\textsc{ReSample} 
                              }                               & \multicolumn{1}{c}{\textsc{PSLD} 
                                             }                & \multicolumn{1}{c}{\textsc{DAPS} 
                                                                 }                    & \multicolumn{1}{c}{\textsc{PNP-DM} 
                                                                                     }    \\
    \midrule
    Gauss. Deblur                  & \tabhi{orange!40}{\tabval{0.18}{0.02}}                & \tabhi{BurntOrange}{\tabbest{0.16}{0.04}}          & \tabval{0.59}{0.03}                        & \tabval{0.32}{0.07}                            & \tabval{0.32}{0.07}            \\
    Mot. Deblur                    & \tabhi{orange!40}{\tabval{0.30}{0.05}}               & \tabhi{BurntOrange}{\tabbest{0.20}{0.06}}          & \tabval{0.70}{0.02}                        & \tabval{0.36}{0.07}                            & \tabval{0.36}{0.08}            \\
    SR ($\times$4)                   & \tabval{0.20}{0.07}              & \tabhi{orange!40}{\tabval{0.22}{0.05}}            & \tabhi{BurntOrange}{\tabbest{0.21}{0.03}}  & \tabval{0.28}{0.06}                            & \tabval{0.40}{0.07}            \\
    SR ($\times$16)                  & \tabval{0.45}{0.10}                                      & \tabhi{orange!40}{\tabval{0.38}{0.12}}            & \tabhi{BurntOrange}{\tabbest{0.36}{0.05}}  & \tabval{0.52}{0.09}                            & \tabval{0.71}{0.11}           \\
    Box Inp.                   & \tabhi{BurntOrange}{\tabbest{0.19}{0.06}}               & \tabhi{orange!40}{\tabval{0.22}{0.04}}            & \tabval{0.27}{0.13}                        & \tabval{0.37}{0.05}                            & \tabval{0.31}{0.12}            \\
    Half Inp.                  & \tabhi{BurntOrange}{\tabbest{0.30}{0.04}}              & \tabhi{orange!40}{\tabval{0.30}{0.03}}          & \tabval{0.32}{0.02}   & \tabval{0.49}{0.01}                            & \tabval{0.44}{0.08}          \\
    JPEG (QF=2)                      & \tabval{0.34}{0.10}           & \tabhi{BurntOrange}{\tabbest{0.26}{0.05}}          & --                           & \tabhi{orange!40}{\tabval{0.32}{0.06}}       & \tabval{0.36}{0.03}              \\
    Phase Retr.                  & \tabval{0.55}{0.07}           & \tabhi{orange!40}{\tabval{0.39}{0.17}}            & --                           & \tabhi{BurntOrange}{\tabbest{0.25}{0.21}}   & \tabval{0.50}{0.14}           \\
    HDR & \tabval{0.27}{0.04}            & \tabhi{BurntOrange}{\tabbest{0.12}{0.04}}          & --                           & \tabval{0.24}{0.09}       &\tabhi{orange!40}{\tabval{0.24}{0.05}}            \\
    \midrule
    \textbf{Memory (MB)}             & \tabhi{BurntOrange}{\textbf{4525}}             & \widefactorcell{6238}{1.4}                                   & \widefactorcell{6919}{1.5}                             & \widefactorcell{5885}{1.3}                             & \widefactorcell{5885}{1.3}             \\
    \textbf{Runtime (sec)}           & \tabhi{BurntOrange}{\textbf{24}}               & \widefactorcell{509}{21.2}                                    & \widefactorcell{244}{10.2}                             & \widefactorcell{1254}{52.3}                             & \widefactorcell{1323}{55.1}             \\
    \bottomrule
    \end{tabular}%
    }
\par\endgroup

        \captionof{table}[FFHQ latent LPIPS comparison]{Mean LPIPS for linear/nonlinear imaging tasks on the \textbf{\texttt{FFHQ}} datasets with LDM prior and $\sigma_y = 0.05$. \textcolor{RoyalBlue}{Blue} factors are relative to \textsc{Spin}; larger values mean greater time or memory savings.}
        \label{tab:ffhq_latent}
    \end{minipage}%
    \hspace{0.005\textwidth}%
    \begin{minipage}[t]{0.385\textwidth}%
        \vspace{0pt}
        \centering
        \begingroup
    \fontsize{8.4}{8}\selectfont
    \setlength{\tabcolsep}{3.2pt}
    \def\schedulenumfont{\fontsize{8.9}{8.3}\selectfont}
    \resizebox{\linewidth}{!}{%
    \begin{tabular}{l>{\schedulenumfont}c>{\schedulenumfont}c>{\schedulenumfont}c}
    \toprule
    Schedule & \multicolumn{1}{c}{\textsc{Deblur}} & \multicolumn{1}{c}{\textsc{Motion Deblur}} & \multicolumn{1}{c}{\textsc{SR4}} \\
    \midrule
    Uniform         & \tabval{0.16}{0.05} & \tabval{0.14}{0.04} & \tabhi{BurntOrange!40}{\tabval{0.15}{0.04}} \\
    Linear          & \tabval{0.18}{0.05} & \tabval{0.19}{0.06} & \tabval{0.18}{0.05} \\
    Polynomial      & \tabval{0.16}{0.04} & \tabval{0.14}{0.04} & \tabval{0.21}{0.05} \\
    Exponential     & \tabhi{BurntOrange!40}{\tabval{0.15}{0.04}} & \tabhi{BurntOrange}{\tabbest{0.11}{0.03}} & \tabval{0.19}{0.05} \\
    Beta            & \tabval{0.22}{0.05} & \tabval{0.32}{0.06} & \tabval{0.28}{0.06} \\
    Gaussian        & \tabhi{BurntOrange}{\tabbest{0.15}{0.03}} & \tabhi{BurntOrange!40}{\tabval{0.12}{0.03}} & \tabhi{BurntOrange}{\tabbest{0.15}{0.03}} \\
    \bottomrule
    \end{tabular}%
    }
\par\endgroup
        \captionof{table}{LPIPS scores for different noise schedules across tasks.}
        \label{tab:schedule_comparison}
        \vspace{3pt}
        \begingroup
    \fontsize{8.4}{8}\selectfont
    \setlength{\tabcolsep}{4.4pt}
    \resizebox{0.96\linewidth}{!}{%
    \begin{tabular}{lccccc}
    \toprule
     & \textsc{Spin} & \textsc{DAPS} & \textsc{Red-Diff} & \textsc{DPS} & \textsc{PNP-DM} \\
    \midrule
    NFE $\downarrow$ & \tabhi{BurntOrange}{\textbf{631}} & 1000 & \tabhi{BurntOrange!40}{999} & \tabhi{BurntOrange!40}{999} & 3848 \\
    \bottomrule
    \end{tabular}%
    }
\par\endgroup

        \captionof{table}{Average FFHQ NFE over valid task cells for the strongest complete methods in Table~\ref{tab:ffhq}. NFE counts only denoiser/score-network forward passes, not optimizer steps.}
        \label{tab:ffhq_nfe}
    \end{minipage}%
\end{table*}

The spacing of the guidance/optimization steps grid is controlled by our iteration scheduling strategy. Our default pixel-space setting uses a Gaussian schedule with center $\mu_G=0.4$ and width $\sigma_G=10.0$, which concentrates optimization steps in the intermediate noise regime where the denoiser can still correct artifacts introduced by measurement updates. For interleaving DDIM steps between guidance operations, we use $\eta_{\text{interleave}}=1.0$, corresponding to ancestral sampling. The small number of task-specific schedule adjustments is reported in Appendix~\ref{sec:hyperparameters}.
\vspace{-5pt}
\subsection{Ablation Studies}
\vspace{-5pt}

Ablations examine \textsc{Spin}'s allocation of computation along the diffusion trajectory. The initial timestep $t_*$ controls both the Phase~1 warm-start quality and the trajectory left for Phase~2 refinement. Figure \ref{fig:ablation} shows performance peaks at $t_* \in [0.4, 0.6]$, an intermediate-noise regime. Starting too early ($t_* \gg 0.6$) increases runtime with little gain, as quality saturates and outputs become overly smooth; starting too late ($t_* \ll 0.4$) leaves insufficient denoising to absorb data-consistency artifacts. Additional metrics are in Appendix \ref{sec:ablation}.
The Phase~2 schedule controls where the remaining guidance budget is spent. Table \ref{tab:schedule_comparison} shows uniform spacing is suboptimal: data-consistency updates are not equally useful at all noise levels. Schedules concentrating guidance in the intermediate regime perform better, with the Gaussian schedule most consistent across tasks. This matches Figure \ref{fig:ablation}: guidance is most effective when the denoiser prediction is informative while leaving enough reverse steps to correct optimization artifacts. We default to $\mu_{G} = 0.4$, $\sigma_G = 10.0$, concentrating optimization around $t_k \in [0.4, 0.6]$ (Appendix \ref{sec:sampling-schedule}).

\textbf{Limitations.} The warm-start timestep $t_*$, Phase~2 schedule parameters, and optimization budget are sensitive to the task and may need to be re-calibrated for a new degradation operator or dataset to recover the best speed--quality trade-off. Our experiments also focus on pretrained DDPM/DDIM-style image priors on \texttt{FFHQ} and \texttt{ImageNet} at $256 \times 256$ resolution; extending the same two-phase design to higher-resolution settings and to flow-matching or consistency-model priors remains future work.

\vspace{-5pt}
\section{Conclusion}
\vspace{-5pt}
\label{sec:conclusion}

We introduced \textsc{Spin}, a two-phase framework for efficient diffusion-based inverse problem solving that avoids starting posterior sampling from pure noise. The method first constructs a task-informed warm start at an intermediate timestep $t_*$, then runs a truncated reverse trajectory with data-consistency corrections applied only at scheduled timesteps where the denoiser can still absorb optimization artifacts; gradients are required only through the forward operator $\mathcal{A}$, avoiding backpropagation through the denoiser, encoder, and decoder. Across linear and nonlinear inverse problems on \texttt{FFHQ} and \texttt{ImageNet}, \textsc{Spin} maintains competitive reconstruction quality while substantially reducing memory and inference time, achieving at least 2$\times$ speedup over competing baselines. These results suggest that posterior sampling can be made more practical by allocating computation selectively along the trajectory rather than applying dense guidance from pure noise, opening a path toward faster inverse solvers that preserve the flexibility of diffusion priors while reducing their deployment cost.

\newpage
\bibliographystyle{plainnat}
\bibliography{bibliography/bibliography}

@String(IJCV = {Int. J. Comput. Vis.})

@String(CVPR= {IEEE Conf. Comput. Vis. Pattern Recog.})

@String(NIPS= {Adv. Neural Inform. Process. Syst.})

@String(IJCV = {IJCV})

@String(CVPR = {CVPR})

@String(NIPS = {NeurIPS})

@inproceedings{rombach2022high,
  title     = {High-resolution image synthesis with latent diffusion models},
  author    = {Rombach, Robin and Blattmann, Andreas and Lorenz, Dominik and Esser, Patrick and Ommer, Bj{"o}rn},
  booktitle = {Proceedings of the IEEE/CVF conference on computer vision and pattern recognition},
  pages     = {10684--10695},
  year      = {2022}
}

@article{ddpm,
  title   = {Denoising diffusion probabilistic models},
  author  = {Ho, Jonathan and Jain, Ajay and Abbeel, Pieter},
  journal = {Advances in neural information processing systems},
  volume  = {33},
  pages   = {6840--6851},
  year    = {2020}
}

@article{ddim,
  title   = {Denoising diffusion implicit models},
  author  = {Song, Jiaming and Meng, Chenlin and Ermon, Stefano},
  journal = {arXiv preprint arXiv:2010.02502},
  year    = {2020}
}

@article{chung2022diffusion,
  title   = {Diffusion posterior sampling for general noisy inverse problems},
  author  = {Chung, Hyungjin and Kim, Jeongsol and Mccann, Michael T and Klasky, Marc L and Ye, Jong Chul},
  journal = {arXiv preprint arXiv:2209.14687},
  year    = {2022}
}

@inproceedings{janati2025mixture,
  title     = {A Mixture-Based Framework for Guiding Diffusion Models},
  author    = {Janati, Yazid and Moufad, Badr and Abou El Qassime, Mehdi and Durmus, Alain Oliviero and Moulines, Eric and Olsson, Jimmy},
  booktitle = {Forty-second International Conference on Machine Learning},
  year      = {2025}
}

@inproceedings{diffpir,
  author    = {Zhu, Yuanzhi and Zhang, Kai and Liang, Jingyun and Cao, Jiezhang and Wen, Bihan and Timofte, Radu and Van Gool, Luc},
  title     = {Denoising Diffusion Models for Plug-and-Play Image Restoration},
  booktitle = {Proceedings of the IEEE/CVF Conference on Computer Vision and Pattern Recognition (CVPR) Workshops},
  month     = {June},
  year      = {2023},
  pages     = {1219-1229}
}

@article{song2023resample,
  title   = {Solving inverse problems with latent diffusion models via hard data consistency},
  author  = {Song, Bowen and Kwon, Soo Min and Zhang, Zecheng and Hu, Xinyu and Qu, Qing and Shen, Liyue},
  journal = {arXiv preprint arXiv:2307.08123},
  year    = {2023}
}

@article{mardani2023variational,
  title   = {A variational perspective on solving inverse problems with diffusion models},
  author  = { Mardani, Morteza and Song, Jiaming and Kautz, Jan and Vahdat, Arash},
  journal = {arXiv preprint arXiv:2305.04391},
  year    = {2023}
}

@article{dhariwal2021diffusion,
  title   = {Diffusion models beat gans on image synthesis},
  author  = {Dhariwal, Prafulla and Nichol, Alexander},
  journal = {Advances in neural information processing systems},
  volume  = {34},
  pages   = {8780--8794},
  year    = {2021}
}

@article{lipman2022flow,
  title   = {Flow matching for generative modeling},
  author  = {Lipman, Yaron and Chen, Ricky TQ and Ben-Hamu, Heli and Nickel, Maximilian and Le, Matt},
  journal = {arXiv preprint arXiv:2210.02747},
  year    = {2022}
}

@inproceedings{esser2024scaling,
  title     = {Scaling rectified flow transformers for high-resolution image synthesis},
  author    = {Esser, Patrick and Kulal, Sumith and Blattmann, Andreas and Entezari, Rahim and M{"u}ller, Jonas and Saini, Harry and Levi, Yam and Lorenz, Dominik and Sauer, Axel and Boesel, Frederic and others},
  booktitle = {Forty-first international conference on machine learning},
  year      = {2024}
}

@article{gao2024diffusion,
  title   = {Diffusion meets flow matching: Two sides of the same coin. 2024},
  author  = {Gao, Ruiqi and Hoogeboom, Emiel and Heek, Jonathan and De Bortoli, Valentin and Murphy, Kevin P and Salimans, Tim},
  journal = {URL https://diffusionflow. github. io},
  year    = {2024}
}

@article{kadkhodaie2020solving,
  title   = {Solving linear inverse problems using the prior implicit in a denoiser},
  author  = {Kadkhodaie, Zahra and Simoncelli, Eero P},
  journal = {arXiv preprint arXiv:2007.13640},
  year    = {2020}
}

@article{daras2024survey,
  title   = {A survey on diffusion models for inverse problems},
  author  = {Daras, Giannis and Chung, Hyungjin and Lai, Chieh-Hsin and Mitsufuji, Yuki and Ye, Jong Chul and Milanfar, Peyman and Dimakis, Alexandros G and Delbracio, Mauricio},
  journal = {arXiv preprint arXiv:2410.00083},
  year    = {2024}
}

@article{kawar2022denoising,
  title   = {Denoising diffusion restoration models},
  author  = {Kawar, Bahjat and Elad, Michael and Ermon, Stefano and Song, Jiaming},
  journal = {Advances in neural information processing systems},
  volume  = {35},
  pages   = {23593--23606},
  year    = {2022}
}

@inproceedings{song2023pseudoinverse,
  title     = {Pseudoinverse-guided diffusion models for inverse problems},
  author    = {Song, Jiaming and Vahdat, Arash and Mardani, Morteza and Kautz, Jan},
  booktitle = {International Conference on Learning Representations},
  year      = {2023}
}

@inproceedings{spagnoletti2025latino-pro,
  title     = {{LATINO-PRO}: {LAtent} consisTency INverse sOlver with {PRompt} Optimization},
  author    = {Spagnoletti, Alessio and Prost, Jean and Almansa, Andr{\'e}s and Papadakis, Nicolas and Pereyra, Marcelo},
  booktitle = {Proceedings of the IEEE/CVF International Conference on Computer Vision},
  pages     = {19597--19607},
  year      = {2025}
}

@inproceedings{chung2024decomposed,
  title     = {Decomposed Diffusion Sampler for Accelerating Large-Scale Inverse Problems},
  author    = {Chung, Hyungjin and Lee, Suhyeon and Ye, Jong Chul},
  booktitle = {International Conference on Learning Representations},
  year      = {2024}
}

@inproceedings{chung2024prompt,
  title     = {Prompt-tuning Latent Diffusion Models for Inverse Problems},
  author    = {Chung, Hyungjin and Ye, Jong Chul and Milanfar, Peyman and Delbracio, Mauricio},
  booktitle = {Proceedings of the 41st International Conference on Machine Learning},
  pages     = {8941--8967},
  year      = {2024},
  volume    = {235},
  series    = {Proceedings of Machine Learning Research},
  publisher = {PMLR}
}

@inproceedings{kim2025regularization,
  title     = {Regularization by Texts for Latent Diffusion Inverse Solvers},
  author    = {Kim, Jeongsol and Park, Geon Yeong and Chung, Hyungjin and Ye, Jong Chul},
  booktitle = {International Conference on Learning Representations},
  year      = {2025}
}

@article{erbach2025flair,
  title   = {Solving Inverse Problems with {FLAIR}},
  author  = {Erbach, Julius and Narnhofer, Dominik and Dombos, Andreas and Schiele, Bernt and Lenssen, Jan Eric and Schindler, Konrad},
  journal = {arXiv preprint arXiv:2506.02680},
  year    = {2025}
}

@article{lipman2024flowmatchingguidecode,
  title   = {Flow Matching Guide and Code},
  author  = {Lipman, Yaron and Havasi, Marton and Holderrieth, Peter and Shaul, Neta and Le, Matt and Karrer, Brian and Chen, Ricky T. Q. and Lopez-Paz, David and Ben-Hamu, Heli and Gat, Itai},
  journal = {arXiv preprint arXiv:2412.06264},
  year    = {2024}
}

@article{kim2025flowdps,
  title   = {Flowdps: Flow-driven posterior sampling for inverse problems},
  author  = {Kim, Jeongsol and Kim, Bryan Sangwoo and Ye, Jong Chul},
  journal = {arXiv preprint arXiv:2503.08136},
  year    = {2025}
}

@article{choi2021ilvr,
  title   = {Ilvr: Conditioning method for denoising diffusion probabilistic models},
  author  = {Choi, Jooyoung and Kim, Sungwon and Jeong, Yonghyun and Gwon, Youngjune and Yoon, Sungroh},
  journal = {arXiv preprint arXiv:2108.02938},
  year    = {2021}
}

@inproceedings{karras2019style,
  title     = {A style-based generator architecture for generative adversarial networks},
  author    = {Karras, Tero and Laine, Samuli and Aila, Timo},
  booktitle = {Proceedings of the IEEE/CVF conference on computer vision and pattern recognition},
  pages     = {4401--4410},
  year      = {2019}
}

@article{wang2022zero,
  title   = {Zero-shot image restoration using denoising diffusion null-space model},
  author  = {Wang, Yinhuai and Yu, Jiwen and Zhang, Jian},
  journal = {arXiv preprint arXiv:2212.00490},
  year    = {2022}
}

@inproceedings{zhang2025improving,
  title     = {Improving diffusion inverse problem solving with decoupled noise annealing},
  author    = {Zhang, Bingliang and Chu, Wenda and Berner, Julius and Meng, Chenlin and Anandkumar, Anima and Song, Yang},
  booktitle = {Proceedings of the Computer Vision and Pattern Recognition Conference},
  pages     = {20895--20905},
  year      = {2025}
}

@article{wu2024principled,
  title   = {Principled probabilistic imaging using diffusion models as plug-and-play priors},
  author  = {Wu, Zihui and Sun, Yu and Chen, Yifan and Zhang, Bingliang and Yue, Yisong and Bouman, Katherine},
  journal = {Advances in Neural Information Processing Systems},
  volume  = {37},
  pages   = {118389--118427},
  year    = {2024}
}

@article{rout2023solving,
  title   = {Solving linear inverse problems provably via posterior sampling with latent diffusion models},
  author  = {Rout, Litu and Raoof, Negin and Daras, Giannis and Caramanis, Constantine and Dimakis, Alex and Shakkottai, Sanjay},
  journal = {Advances in Neural Information Processing Systems},
  volume  = {36},
  pages   = {49960--49990},
  year    = {2023}
}

@article{martin2024pnp,
  title   = {Pnp-flow: Plug-and-play image restoration with flow matching},
  author  = {Martin, S{\'e}gol{\`e}ne and Gagneux, Anne and Hagemann, Paul and Steidl, Gabriele},
  journal = {arXiv preprint arXiv:2410.02423},
  year    = {2024}
}

@inproceedings{shin2017jpeg,
  title     = {Jpeg-resistant adversarial images},
  author    = {Shin, Richard and Song, Dawn and others},
  booktitle = {NIPS 2017 workshop on machine learning and computer security},
  volume    = {1},
  pages     = {8},
  year      = {2017}
}

@inproceedings{tran2021explore,
  title     = {Explore image deblurring via encoded blur kernel space},
  author    = {Tran, Phong and Tran, Anh Tuan and Phung, Quynh and Hoai, Minh},
  booktitle = {Proceedings of the IEEE/CVF conference on computer vision and pattern recognition},
  pages     = {11956--11965},
  year      = {2021}
}

@inproceedings{meng2021sdedit,
  title     = {SDEdit: Guided Image Synthesis and Editing with Stochastic Differential Equations},
  author    = {Meng, Chenlin and He, Yutong and Song, Yang and Song, Jiaming and Wu, Jiajun and Zhu, Jun-Yan and Ermon, Stefano},
  booktitle = {International Conference on Learning Representations},
  year      = {2022},
  url       = {https://openreview.net/forum?id=aBsCjcPu_tE}
}

@inproceedings{chung2022comecloser,
  title     = {Come-Closer-Diffuse-Faster: Accelerating Conditional Diffusion Models for Inverse Problems Through Stochastic Contraction},
  author    = {Chung, Hyungjin and Sim, Byeongsu and Ye, Jong Chul},
  booktitle = {Proceedings of the IEEE/CVF Conference on Computer Vision and Pattern Recognition},
  pages     = {12413--12422},
  year      = {2022},
  doi       = {10.1109/CVPR52688.2022.01209}
}

@inproceedings{song2023consistency,
  title        = {Consistency Models},
  author       = {Song, Yang and Dhariwal, Prafulla and Chen, Mark and Sutskever, Ilya},
  booktitle    = {Proceedings of the 40th International Conference on Machine Learning},
  pages        = {32211--32252},
  year         = {2023},
  organization = {PMLR}
}

@inproceedings{garber2025zeroshot,
  title     = {Zero-Shot Image Restoration Using Few-Step Guidance of Consistency Models (and Beyond)},
  author    = {Garber, Tomer and Tirer, Tom},
  booktitle = {Proceedings of the IEEE/CVF Conference on Computer Vision and Pattern Recognition},
  pages     = {2398--2407},
  year      = {2025}
}

@inproceedings{venkatakrishnan2013plugandplay,
  title        = {Plug-and-Play Priors for Model Based Reconstruction},
  author       = {Venkatakrishnan, Singanallur V. and Bouman, Charles A. and Wohlberg, Brendt},
  booktitle    = {2013 IEEE Global Conference on Signal and Information Processing},
  pages        = {945--948},
  year         = {2013},
  organization = {IEEE}
}

@article{chan2016plugandplay,
  title   = {Plug-and-Play {ADMM} for Image Restoration: Fixed-Point Convergence and Applications},
  author  = {Chan, Stanley H. and Wang, Xiran and Elgendy, Omar A.},
  journal = {IEEE Transactions on Computational Imaging},
  volume  = {3},
  number  = {1},
  pages   = {84--98},
  year    = {2017}
}

@article{romano2017little,
  title   = {The Little Engine That Could: Regularization by Denoising ({RED})},
  author  = {Romano, Yaniv and Elad, Michael and Milanfar, Peyman},
  journal = {SIAM Journal on Imaging Sciences},
  volume  = {10},
  number  = {4},
  pages   = {1804--1844},
  year    = {2017}
}

@inproceedings{zhang2018unreasonable,
  title     = {The Unreasonable Effectiveness of Deep Features as a Perceptual Metric},
  author    = {Zhang, Richard and Isola, Phillip and Efros, Alexei A and Shechtman, Eli and Wang, Oliver},
  booktitle = {CVPR},
  year      = {2018}
}

@article{ILSVRC15,
  author  = {Olga Russakovsky and Jia Deng and Hao Su and Jonathan Krause and Sanjeev Satheesh and Sean Ma and Zhiheng Huang and Andrej Karpathy and Aditya Khosla and Michael Bernstein and Alexander C. Berg and Li Fei-Fei},
  title   = {{ImageNet Large Scale Visual Recognition Challenge}},
  year    = {2015},
  journal = {International Journal of Computer Vision (IJCV)},
  doi     = {10.1007/s11263-015-0816-y},
  volume  = {115},
  number  = {3},
  pages   = {211-252}
}

@article{moufad2025efficient,
  title={Efficient Zero-Shot Inpainting with Decoupled Diffusion Guidance},
  author={Moufad, Badr and Shouraki, Navid Bagheri and Durmus, Alain Oliviero and Hirtz, Thomas and Moulines, Eric and Olsson, Jimmy and Janati, Yazid},
  journal={arXiv preprint arXiv:2512.18365},
  year={2025}
}

\appendix
\onecolumn
\begin{center}
\LARGE \textbf{Supplementary Material}
\end{center}

\section{Algorithm}
\label{sec:algorithm}

This appendix complements Section~\ref{sec:method} with the implementation-level details of \textsc{Spin}. We (\textbf{i}) lift the pixel-space description of the main paper to the latent-space algorithm actually used in our experiments, (\textbf{ii}) present the full algorithm and walk through its line groups, and (\textbf{iii}) derive in closed form the conservative re-noising step that returns each phase to the noisy manifold. The conceptual justification of conservative re-noising as an optimal coupling of Gaussians is given separately in Appendix~\ref{sec:renoising_app}.

\subsection{From the pixel-space description to the latent-space implementation}
\label{sec:algorithm-latent}

Section~\ref{sec:method} writes the method in pixel space ($\bx$) for clarity, but our experiments also use a latent diffusion prior, so the implemented algorithm operates on a latent variable $\bz$ with $\bz = \encoder(\bx)$ and $\bx = \decoder(\bz)$. The denoiser $\denoiser{t}{}{\cdot}[\param]$ maps a noisy latent to its predicted clean latent, and forward noising follows the latent variance-preserving schedule $\bz_t = \alpha_t \bz_0 + \sigma_t \epsilon$ with $\alpha_t^2 + \sigma_t^2 = 1$.

The pixel/latent boundary is crossed only at the data-consistency step. The forward operator $\opA$ is defined in pixel space, so each \textsf{Optimize} call decodes the predicted clean latent, performs gradient descent on
$\|\obs - \opA(\bx)\|^2$ (Phase~1) or its anchored version
$\|\obs - \opA(\bx)\|^2 + \lambda\|\bx - \hat{\bx}_0\|^2$ (Phase~2),
and re-encodes the result. Crucially, the gradient flows only through $\opA$: \emph{no} VJPs are taken through $\denoiser{}{}{}[\param]$, $\encoder$, or $\decoder$. This is the source of the wall-clock and memory savings reported in Section~\ref{sec:experiments} relative to score-correction methods such as DPS.

The symbol map between Section~\ref{sec:method} and Algorithm~\ref{alg:spin-current} is direct: pixel-space iterates $\bx^k_0$, $\hat{\bx}^k_0$, $\bx^*_0$ correspond, after encoding/decoding, to the latent iterates $\encoder(\bx^k_0)$, $\hat{\bz}^k_0$, $\bz^*_0=\encoder(\bx^*_0)$ used in the algorithm. All re-noising algebra below is therefore identical to the pixel-space derivation up to relabeling.

\subsection{Full algorithm}
\label{sec:algorithm-full}

\begin{figure}[t]
\refstepcounter{algorithm}\label{alg:spin-current}
\begin{tcolorbox}[
    colback=orange!5!white,
    colframe=orange!5!white,
    boxrule=0pt,
    arc=3pt,
    left=2pt,
    right=2pt,
    top=3pt,
    bottom=3pt,
    breakable,
    before upper=\setlength{\parindent}{0pt}]
\rule{\linewidth}{0.8pt}\\[2pt]
{\small\bfseries Algorithm~\thealgorithm\enspace \textsc{Spin} --- \texttt{Full}}\\[-2pt]
\rule{\linewidth}{0.3pt}
\small
{\algrenewcommand\alglinenumber[1]{\scriptsize\textcolor{black!40}{#1}}%
\begin{algorithmic}[1]
\Require Observation $\obs$, initializer $\bx_{\mathrm{init}}(\obs)$, forward operator $\opA$, encoder $\encoder$, decoder $\decoder$, denoiser $\denoiser{\cdot}{}{\cdot}[\param]$, timesteps $\{\tk{1}, \dots, \tk{M}\}$, warm-start iterations $N$, DDIM steps $n$  
\State \textcolor{orange}{\small{\texttt{Phase 1: Initial Guess}}}
\State $\bz^{0}_{\tk{*}} \sim \gauss(\alpha_{\tk{*}}\encoder(\bx_{\mathrm{init}}(\obs)), \sigma^2_{\tk{*}}\Id)$ 
\For{$k = 0, \dots, N-1$}
    \State $\hat{\bz}^{k}_0 \gets \denoiser{\tk{*}}{}{\bz^{k}_{\tk{*}}}[\param]$ \Comment{\textcolor{orange}{\scriptsize\texttt{Predict clean latent}}}
    \State $\epsilon^{k} \gets (\bz^{k}_{\tk{*}} - \alpha_{\tk{*}} \hat{\bz}^{k}_0) / \sigma_{\tk{*}}$
    \State $\bx^{k}_0 \gets \textsf{Optimize}(\bx \mapsto \|\obs - \opA(\bx)\|^2, \bx_0 = \decoder(\hat{\bz}^{k}_0))$
    \State $\bmu^{k} \gets \alpha_{\tk{*}} \encoder(\bx^{k}_0) + \sigma_{\tk{*}} \hat{\bz}^{k}_0 + \sigma_{\tk{*}} \epsilon^{k}$
    \State $\sigma^{k} \gets \sigma^2_{\tk{*}} \Id$
    \State $\bz^{k+1}_{\tk{*}} \sim  \gauss(\alpha_{\tk{*}}\bmu^{k}, \sigma^2_{\tk{*}}\sigma^{k})$  \Comment{\textcolor{orange}{\scriptsize{\texttt{Re-noise}}}}
\EndFor
\State $\bz^\obs_{\tk{M}} \gets \bz^{N}_{\tk{*}}$
\State \textcolor{orange}{\small{\texttt{Phase 2: Guided denoising}}}
\For{$k = M-1, \dots, 1$}
    \State $\tilde{\bz}_{\tk{k}} \sim \pdata{\tk{k}\tbar\tk{k+1}}{\bz^\obs_{\tk{k+1}}}{}[\param]$
    \State $\hat{\bz}_0 \gets \denoiser{\tk{k}}{}{\tilde{\bz}_{\tk{k}}}[\param]$
    \State $\epsilon_{\tk{k}} \gets (\tilde{\bz}_{\tk{k}} - \alpha_{\tk{k}} \hat{\bz}_0) / \sigma_{\tk{k}}$
    \State $\hat{\bx}_0 \gets \decoder(\hat{\bz}_0)$
    \State $\bx^*_0 \gets \textsf{Optimize}(\bx \mapsto \| \obs - \opA(\bx) \|^2 + \lambda \| \bx - \hat{\bx}_0 \|^2, \bx_0 = \hat{\bx}_0)$
    \State $\bmu_{\tk{k+1}} \gets \alpha_{\tk{k+1}} \encoder(\bx^*_0) + \sigma_{\tk{k+1}} \alpha_{\tk{k+1}} \epsilon_{\tk{k}}$
    \State $\sigma'_{\tk{k+1}} \gets \sigma^2_{\tk{k+1}} \Id$ 
    \State $\tilde{\bz}_{\tk{k+1}} \sim \gauss(\bmu_{\tk{k+1}}, \sigma^2_{\tk{k+1}} \sigma'_{\tk{k+1}})$ 
    \State $\bz^\obs_{\tk{k}} \gets \textsf{DDIM}(\tilde{\bz}_{\tk{k+1}}, \tk{k+1} \xrightarrow[\text{ }]{n}{\tk{k}})$ \Comment{\textcolor{orange}{\tiny{\texttt{DDIM in n steps}}}}

\EndFor
\State \Return $\hat{\bx} = \decoder(\denoiser{0}{}{\bz^\obs_{\tk{1}}}[\param])$
\end{algorithmic}
}
\rule{\linewidth}{0.8pt}
\end{tcolorbox}
\end{figure}

\paragraph{Line-group walkthrough.} The two phases of Algorithm~\ref{alg:spin-current} mirror the structure of Section~\ref{sec:method}:
\begin{itemize}
    \item \textbf{Phase~1, line~2:} initialize the warm-start state at $\tk{*}$ from the observation-based initializer $\bx_{\mathrm{init}}(\obs)$ (a task-specific zero-fill / pseudo-inverse; see Appendix~\ref{sec:hyperparameters}).
    \item \textbf{Phase~1, lines~4--5:} predict the clean latent $\hat{\bz}^k_0$ and extract the implied noise $\epsilon^k$ from the denoiser output.
    \item \textbf{Phase~1, line~6:} decode and solve the unregularized data-consistency objective in pixel space.
    \item \textbf{Phase~1, lines~7--9:} re-noise back to $\tk{*}$. The mean $\bmu^k$ and the sampling distribution are derived in Appendix~\ref{sec:phase1-renoise}.
    \item \textbf{Phase~2, lines~14--17:} probe a lower-noise state along the reverse trajectory, denoise it to $\hat{\bz}_0$, extract its noise term $\epsilon_{\tk{k}}$, and decode it to $\hat{\bx}_0$.
    \item \textbf{Phase~2, line~18:} solve the anchored data-consistency objective initialized at $\hat{\bx}_0$.
    \item \textbf{Phase~2, lines~19--21:} re-noise back to $\tk{k+1}$ (Appendix~\ref{sec:phase2-renoise}).
    \item \textbf{Phase~2, line~22:} run $n$ DDIM steps from $\tk{k+1}$ to $\tk{k}$ to reach the next accepted state.
\end{itemize}
Intermediate transitions between scheduled timesteps $\{\tk{1},\dots,\tk{M}\}$ are vanilla denoiser steps; data consistency is solved only at scheduled timesteps. The schedule itself is described in Appendix~\ref{sec:sampling-schedule}.

\paragraph{Computational cost.} Algorithm~\ref{alg:spin-current} performs $N + M + nM$ denoiser evaluations (warm-start iterations plus probe and DDIM substeps in Phase~2) and $N+M$ pixel-space optimization passes through the forward operator $\opA$. Every gradient step backpropagates only through $\opA$; no Jacobian-vector products are taken through the denoiser, encoder, or decoder. This contrasts with DPS-style samplers, which take a denoiser VJP at every reverse step.

\subsection{Conservative re-noising: Phase~1}
\label{sec:phase1-renoise}

After the optimization in Algorithm~\ref{alg:spin-current} (line~6) yields $\bz^k_0=\encoder(\bx^k_0)$, we must return to a noisy state at timestep $\tk{*}$ for the next iteration. A standard forward-noising step $\bz^{k+1}_{\tk{*}} = \alpha_{\tk{*}}\bz^k_0 + \sigma_{\tk{*}}\epsilon$ with $\epsilon\sim\gauss(\zero,\Id)$ would discard the current noise realization. Lemma~\ref{lem:optimal_coupling} (Appendix~\ref{sec:renoising_app}) shows that the least-disruptive return to the noisy manifold instead reuses the predicted noise $\epsilon^k$ extracted from the denoiser output; we construct such a step explicitly here. Throughout we use the variance-preserving identity $\alpha_t^2 + \sigma_t^2 = 1$; for non-VP schedules the same construction applies after normalizing the mixing coefficients.

We treat the prior-predicted clean latent $\hat{\bz}^k_0$ as a second clean estimate and form a forward-noising-style combination of the two,
\begin{equation*}
\bz^{\text{mix}}_0 := \alpha_{\tk{*}} \bz^{k}_0 + \sigma_{\tk{*}} \hat{\bz}^{k}_0.
\end{equation*}
The coefficients are chosen so that the limit $\tk{*}\to 0$ recovers the optimized solution ($\alpha_{\tk{*}}\to 1$, $\sigma_{\tk{*}}\to 0$), while in the high-noise limit $\tk{*}\to 1$ the prior prediction takes over. We mix the predicted noise
$\epsilon^{k} = (\bz^{k}_{\tk{*}} - \alpha_{\tk{*}} \hat{\bz}^{k}_0)/\sigma_{\tk{*}}$
with fresh Gaussian noise $\epsilon \sim \gauss(\zero, \Id)$ in the same proportion:
\begin{equation*}
\epsilon^{\text{mix}} := \alpha_{\tk{*}} \epsilon^{k} + \sigma_{\tk{*}} \epsilon.
\end{equation*}
Under VP and an ideal unit-covariance noise prediction, $\epsilon^{\text{mix}}$ has unit marginal covariance, while the conditional fresh-noise covariance is only $\sigma_{\tk{*}}^4\Id$ --- strictly smaller than the $\Id$ that a naive forward step would inject. Applying a forward-noising update to $\bz^{\text{mix}}_0$ yields:
\begin{align*}
\bz^{k+1}_{\tk{*}}
&= \alpha_{\tk{*}} \bz^{\text{mix}}_0 + \sigma_{\tk{*}} \epsilon^{\text{mix}} \\
&= \alpha_{\tk{*}} \left(\alpha_{\tk{*}} \bz^{k}_0 + \sigma_{\tk{*}} \hat{\bz}^{k}_0\right) + \sigma_{\tk{*}} \left(\alpha_{\tk{*}} \epsilon^{k} + \sigma_{\tk{*}} \epsilon\right) \\
&= \alpha_{\tk{*}} \underbrace{\left(\alpha_{\tk{*}} \bz^{k}_0 + \sigma_{\tk{*}} \hat{\bz}^{k}_0 + \sigma_{\tk{*}} \epsilon^{k}\right)}_{\bmu^{k}} + \sigma^2_{\tk{*}} \epsilon.
\end{align*}
This matches lines~7--9 of Algorithm~\ref{alg:spin-current}: setting $\sigma^k := \sigma^2_{\tk{*}}\Id$, the update is equivalent to sampling
\begin{equation*}
\bz^{k+1}_{\tk{*}} \sim \gauss\!\left(\alpha_{\tk{*}}\,\bmu^{k},\; \sigma^4_{\tk{*}}\Id\right),
\end{equation*}
which the algorithm writes compactly as $\gauss(\alpha_{\tk{*}}\bmu^k,\,\sigma^2_{\tk{*}}\sigma^k)$. The total marginal covariance is preserved at $\sigma^2_{\tk{*}}\Id$; only the conservative residual $\sigma^4_{\tk{*}}\Id$ is freshly injected per iteration.

\subsection{Conservative re-noising: Phase~2}
\label{sec:phase2-renoise}

In Phase~2 the data-consistency step (Algorithm~\ref{alg:spin-current}, line~18) is already anchored to the prior-predicted clean state $\hat{\bx}_0$ via the regularizer $\lambda\|\bx-\hat{\bx}_0\|^2$, so the optimized output $\bx^*_0$ already encodes the trade-off between observation and prior. \emph{No second clean-state mixing is therefore needed}: unlike Phase~1, only the noise component requires the conservative treatment of Lemma~\ref{lem:optimal_coupling} (Appendix~\ref{sec:renoising_app}).

Starting from $\bz^*_0=\encoder(\bx^*_0)$ and the noise extracted from the probe state,
$\epsilon_{\tk{k}} = (\tilde{\bz}_{\tk{k}} - \alpha_{\tk{k}} \hat{\bz}_0)/\sigma_{\tk{k}}$,
we mix $\epsilon_{\tk{k}}$ with fresh Gaussian noise $\epsilon\sim\gauss(\zero,\Id)$ at the next-timestep coefficients:
\begin{equation*}
\epsilon^{\text{mix}} := \alpha_{\tk{k+1}} \epsilon_{\tk{k}} + \sigma_{\tk{k+1}} \epsilon.
\end{equation*}
A forward-noising update with $\bz^*_0$ as the (already anchored) clean state then gives:
\begin{align*}
\tilde{\bz}_{\tk{k+1}}
&= \alpha_{\tk{k+1}} \bz^*_0 + \sigma_{\tk{k+1}} \epsilon^{\text{mix}} \\
&= \alpha_{\tk{k+1}} \bz^*_0 + \sigma_{\tk{k+1}} \alpha_{\tk{k+1}} \epsilon_{\tk{k}} + \sigma^2_{\tk{k+1}} \epsilon \\
&= \underbrace{\alpha_{\tk{k+1}} \bz^*_0 + \sigma_{\tk{k+1}} \alpha_{\tk{k+1}} \epsilon_{\tk{k}}}_{\bmu_{\tk{k+1}}} + \sigma^2_{\tk{k+1}} \epsilon.
\end{align*}
This matches lines~19--21 of Algorithm~\ref{alg:spin-current}: with $\sigma'_{\tk{k+1}} := \sigma^2_{\tk{k+1}}\Id$, we sample
\begin{equation*}
\tilde{\bz}_{\tk{k+1}} \sim \gauss\!\left(\bmu_{\tk{k+1}},\; \sigma^4_{\tk{k+1}}\Id\right),
\end{equation*}
which the algorithm writes compactly as $\gauss(\bmu_{\tk{k+1}},\,\sigma^2_{\tk{k+1}}\sigma'_{\tk{k+1}})$. The accepted state $\bz^\obs_{\tk{k}}$ is then obtained by $n$ DDIM denoising steps from $\tk{k+1}$ to $\tk{k}$ (Algorithm~\ref{alg:spin-current}, line~22). The placement of the guidance timesteps $\{\tk{1},\dots,\tk{M}\}$ is described in Appendix~\ref{sec:sampling-schedule}.

\clearpage
\section{Theoretical Interpretation}
\label{sec:theory}

This appendix provides a concise interpretation of \textsc{Spin} as an \emph{approximate split inference} procedure that alternates between (i) a prior-driven diffusion update and (ii) a data-consistency update implemented via pixel-space optimization. The exposition proceeds as follows:
\begin{itemize}
    \item Section~\ref{sec:setup_app} fixes notation and recalls the posterior time-marginals targeted by the method;
    \item Section~\ref{sec:renoising_app} formalizes the conservative re-noising step as an optimal coupling of Gaussians;
    \item Section~\ref{sec:gibbs_app} interprets the resulting prior--likelihood--re-noise loop as an inexact Gibbs sampler;
    \item Section~\ref{sec:phase1_contraction} establishes a $W_1$-contraction for the Phase~1 inner loop and bounds $\bm{\epsilon_*}$ geometrically in the iteration count $N$;
    \item Section~\ref{sec:warm_start_error} bounds the warm-start truncation error and proves Lemma~\ref{lem:warm_start_truncation}; combined with Section~\ref{sec:phase1_contraction} this yields the end-to-end posterior-error bound (Corollary~\ref{cor:end_to_end}).
\end{itemize}

\subsection{Setup}
\label{sec:setup_app}

\paragraph{Notation.}
We follow the paper's convention that $Z_t$ (resp.\ $X_t$) denotes a latent-space (resp.\ pixel-space) variable, while bold symbols such as $\bz_t$ and $\bx$ denote generic realizations/iterates.
Algorithm~\ref{alg:spin-current} follows this convention and writes iterates in bold (e.g., $\bz$, $\bx$), while the background uses uppercase notation ($Z_t$, $X_t$) when referring to the underlying diffusion random variables and their marginals.
In Phase~2, the algorithm uses a constant weight $\lambda$; throughout this appendix we allow a time-dependent $\lambda_t$, and the algorithm is recovered by taking $\lambda_t\equiv \lambda$ (or $\lambda_{t_k}$ on the chosen schedule). The case $\lambda_t=0$ used in some experiments corresponds to the limiting unregularized data-consistency update, with regularization supplied implicitly by the surrounding denoising dynamics and re-noising steps.

\paragraph{Target posterior and time-marginals.} Recall the latent-space posterior at $t=0$:
\begin{equation}
\label{eq:app_posterior_target}
\pi_0(\bz_0 \mid \obs) \;\propto\; \ell_0(\obs \mid \bz_0)\, p_0(\bz_0),
\qquad
\ell_0(\obs \mid \bz_0) = \gauss\!\left(\obs;\, \opA(\decoder(\bz_0)), \sigma_y^2 \Id\right).
\end{equation}
For any noise level $t \in [0,1]$, define the posterior time-marginal $\pi_t(\bz_t\mid \obs)$ by pushing $\pi_0(\cdot\mid\obs)$ through the forward noising kernel $p_t(\bz_t \mid \bz_0)$, (i.e., $\bz_t=\alpha_t\bz_0+\sigma_t\bz_1$). This yields the classical form
\begin{equation}
\label{eq:app_posterior_marginal}
\pi_t(\bz_t \mid \obs) \;\propto\; \ell_t(\obs \mid \bz_t)\, p_t(\bz_t),
\qquad
\ell_t(\obs\mid\bz) = \pE[\ell_0(\obs\mid \BZ_0)\mid \BZ_t=\bz],
\end{equation}
and exact posterior sampling can be implemented by reverse-time transitions that depend on the score of $\ell_t$, cf.\ \eqref{eq:posterior-denoiser}. Our method avoids evaluating $\nabla_{\bz_t}\log \ell_t(\obs\mid \bz_t)$ via VJPs, and instead enforces data consistency through optimization steps at selected timesteps.

\subsection{Re-noising as an optimal coupling}
\label{sec:renoising_app}

After updating the clean estimate via optimization, the algorithm re-noises to return to a noisy manifold (Phase~1 and Phase~2). A useful perspective is that re-noising should change the mean induced by the updated clean state \emph{while minimally perturbing the current sample}, which is naturally formalized via optimal couplings of Gaussians.

\begin{lemma}[Optimal coupling of equal-covariance Gaussians]
\label{lem:optimal_coupling}
Let $P=\gauss(\mathbf{m}, \cov)$ and $P'=\gauss(\mathbf{m}', \cov)$ with the same covariance $\cov \succ 0$. The coupling
\(
X=\mathbf{m}+\xi,\; X'=\mathbf{m}'+\xi
\)
with $\xi \sim \gauss(0,\cov)$ minimizes $\pE\|X-X'\|^2$ among all couplings of $(P,P')$, and achieves
\(
\inf \pE\|X-X'\|^2=\|\mathbf{m}-\mathbf{m}'\|^2.
\)

\end{lemma}

\begin{proof}
Let $(X,X')$ be an arbitrary coupling of $(P,P')$. Write
$X=\mathbf{m}+\xi$ and $X'=\mathbf{m}'+\xi'$, where both
$\xi$ and $\xi'$ have law $\gauss(0,\cov)$. Then
\begin{equation*}
    \pE\|X-X'\|^2
    =
    \|\mathbf{m}-\mathbf{m}'\|^2
    +
    \pE\|\xi-\xi'\|^2.
\end{equation*}
It follows that, for every coupling,
\begin{equation*}
    \pE\|X-X'\|^2 \ge \|\mathbf{m}-\mathbf{m}'\|^2.
\end{equation*}
The lower bound is attained by the synchronous coupling
$\xi'=\xi$, which is feasible since the centered marginals coincide.
This proves both optimality and the value of the infimum.
\end{proof}

Lemma~\ref{lem:optimal_coupling} motivates \emph{noise preservation}: when the clean estimate changes (e.g., from $\hat{\bx}_0$ to $\bx^*_0$), the least-disruptive way to return to the noisy manifold is to keep the same underlying noise realization and only update the mean term. In practice the true noise is not observed; we therefore approximate it using the predicted noise extracted from the denoiser output (see Appendices~\ref{sec:phase1-renoise} and~\ref{sec:phase2-renoise} for the explicit construction in each phase). We additionally mix this predicted noise with fresh Gaussian noise. Under a variance-preserving schedule, $\alpha_t^2+\sigma_t^2=1$, and an ideal unit-covariance noise prediction, the blended noise has unit marginal covariance, while the conditional fresh-noise covariance is only $\sigma_t^4\Id$. Thus the update is deliberately more conservative than drawing a fully fresh sample from the forward kernel: it preserves accumulated data-consistent information while allowing controlled stochastic refresh.

\subsection{Inexact Gibbs / alternating update viewpoint}
\label{sec:gibbs_app}

This viewpoint explains \textsc{Spin} as a simple ``alternate between prior and likelihood'' routine at a fixed noise level.

At a fixed timestep $t$, consider the augmented latent posterior over $(\bz_0,\bz_t)$:
\begin{equation}
\label{eq:augmented_posterior}
\pi(\bz_0,\bz_t \mid \obs)\;\propto\; \ell_0(\obs \mid \bz_0)\, p_0(\bz_0)\, p(\bz_t\mid \bz_0),
\end{equation}
where $p(\bz_t\mid \bz_0)=\gauss(\alpha_t \bz_0, \sigma_t^2 \Id)$ is the forward noising kernel.

A (hypothetical) exact Gibbs sampler for \eqref{eq:augmented_posterior} would alternate between sampling
\(
\bz_0 \sim \pi(\bz_0\mid \bz_t,\obs)
\)
and
\(
\bz_t \sim p(\bz_t\mid \bz_0).
\)
The first conditional is intractable in general. \textsc{Spin} can be seen as an \emph{inexact} replacement of this step:
\begin{enumerate}
    \item \textbf{Prior prediction:} use the pretrained denoiser to produce $\hat{\bz}_0=\denoiser{t}{}{\bz_t}[\param]$ (a point estimate for $\bz_0\mid \bz_t$), and decode to $\hat{\bx}_0=\decoder(\hat{\bz}_0)$.
    \item \textbf{Data consistency:} incorporate $\obs$ by solving the MAP/prox problem, which outputs $\bx_t^*$.
    \item \textbf{Re-noise:} map $\bx_t^*$ back to the noisy manifold by a conservative forward-style update that preserves the current noise realization as much as possible (Lemma~\ref{lem:optimal_coupling}).
\end{enumerate}
This yields an intuitive ``prior $\rightarrow$ data-consistency $\rightarrow$ re-noise'' loop.

\begin{remark}[Phase 1 as repeated approximate conditioning at $t_*$]
\label{rem:phase1}
Phase~1 fixes a single noise level $t_*$ and repeats the inexact update above for $N$ iterations. This can be viewed as iteratively refining an approximate sample from the posterior time-marginal $\pi_{t_*}(Z_{t_*}\mid \obs)$ using repeated ``(prior prediction $\rightarrow$ prox data-consistency $\rightarrow$ re-noise)'' cycles, so that Phase~2 can start closer to the conditional manifold while skipping expensive steps from $t=1$.
\end{remark}

\begin{remark}[Guidance schedule]
\label{rem:intermediate_noise}
Interpreting $\lambda_t$ as a proxy for SNR, guidance is typically most effective at intermediate noise levels.
When $t\to 0$ (high SNR), the diffusion prior/denoiser is already confident and can be ``too rigid'': heavy optimization may introduce artifacts that the remaining denoising steps cannot easily correct.
When $t\to 1$ (low SNR), the denoiser prediction is highly uncertain, so the quadratic anchoring term in the MAP/prox objective becomes less informative.
This motivates concentrating optimization at intermediate timesteps (Appendix~\ref{sec:sampling-schedule}).
\end{remark}

\subsection{Phase~1 contraction and bound on \texorpdfstring{$\bm{\epsilon_*}$}{epsilon\_*}}
\label{sec:phase1_contraction}

Lemma~\ref{lem:warm_start_truncation} treats the warm-start error
\begin{equation*}
    \bm{\epsilon_*}
    =
    W_1\!\left(\mu_{\tk{*}},\pi_{\tk{*}}(\cdot\mid\obs)\right)
\end{equation*}
as a free quantity. This subsection establishes an explicit bound on $\bm{\epsilon_*}$ as a function of the number of Phase~1 iterations $N$. Under Lipschitz assumptions on the denoiser and the data-consistency map (Assumption~\ref{ass:phase1_lipschitz}), the Phase~1 update is a $W_1$-contraction at the fixed timestep $\tk{*}$ (Proposition~\ref{prop:phase1_contraction}). The iterates $\mu^N$ therefore converge geometrically to a unique fixed-point law $\mu^\infty$, and the residual mismatch $W_1(\mu^\infty,\pi_{\tk{*}}(\cdot\mid\obs))$ is a standing bias attributable to approximation error in the denoiser and the data-consistency step (Remark~\ref{rem:phase1_bias}).

For analytical clarity we work with the simplified pixel-space Phase~1 loop with fresh-noise re-noising. The conservative noise-preserving variant used in the implementation and the latent-space extension are addressed in Remarks~\ref{rem:phase1_conservative} and~\ref{rem:phase1_latent}.

\paragraph{Phase~1 update operator.}
At fixed timestep $\tk{*}$, one Phase~1 iteration takes the current iterate $\bx^k_{\tk{*}}$ through the same three steps used in Section~\ref{sec:method}. (i) Apply the pretrained denoiser to get a clean estimate,
\begin{equation}
\label{eq:phase1_denoise}
    \hat{\bx}^k_0 \;=\; \denoiser{\tk{*}}{}{\bx^k_{\tk{*}}}[\param].
\end{equation}
(ii) Apply the unregularized Phase~1 data-consistency map of Section~\ref{sec:method},
\begin{equation}
\label{eq:phase1_dc}
    \bx^k_0 \;=\; \DC(\hat{\bx}^k_0),
\end{equation}
where $\DC$ denotes the operator that sends $\hat{\bx}$ to the solution of the unregularized data-consistency problem of Section~\ref{sec:method}. Assumption~\ref{ass:phase1_lipschitz}(L1) below specifies $\DC$ concretely in the linear case.
(iii) Re-noise back to noise level $\tk{*}$,
\begin{equation}
\label{eq:phase1_steps}
    \bx^{k+1}_{\tk{*}}\;=\;\alpha_{\tk{*}}\bx^k_0+\sigma_{\tk{*}}\xi^k, \qquad \xi^k\sim\gauss(\zero,\Id)\;\text{i.i.d.}
\end{equation}
The denoiser $\denoiser{\tk{*}}{}{\cdot}[\param]$ is the same object used throughout Sections~\ref{sec:background}--\ref{sec:method}. The data-consistency step~\eqref{eq:phase1_dc} is the operator counterpart of the Phase~1 problem solved in Section~\ref{sec:method}; the symbol $\DC$ for that operator is the only new notation in this subsection. Let $T$ denote the induced update on probability laws:
\begin{equation*}
    T\mu^k=\mu^{k+1}.
\end{equation*}
After $N$ warm-start iterations, $\mu^N$ is the law denoted by $\mu_{\tk{*}}$ in Lemma~\ref{lem:warm_start_truncation}.

\begin{assumption}[Component Lipschitz constants]
\label{ass:phase1_lipschitz}
\hfill
\begin{enumerate}
    \item[(L1)] The forward operator $\opA$ is linear, and $\DC$ returns the minimum-displacement least-squares solution
    \begin{equation*}
    \begin{aligned}
        \DC(\hat{\bx})
        \;=\;
        \argmin_{\bx\,\in\,\argmin_{\bz}\|\obs-\opA\bz\|^2}\;\|\bx-\hat{\bx}\|^2
        \;=\;
        \opA^{\dagger}\obs + P_{\ker\opA}\,\hat{\bx},
    \end{aligned}
    \end{equation*}
    where $\opA^{\dagger}$ is the Moore--Penrose pseudoinverse and $P_{\ker\opA}$ the orthogonal projection onto $\ker\opA$; equivalently, $\DC(\hat{\bx})$ is the limit of gradient descent on $\|\obs-\opA\bx\|^2$ initialized at $\hat{\bx}$ with step size $\eta\in(0,1/\|\opA\|_{\mathrm{op}}^2)$. In the noiseless consistent case $\obs\in\operatorname{range}(\opA)$ the constraint set $\{\bx:\opA\bx=\obs\}$ is nonempty and this reduces to the Euclidean projection onto it. Since the data-dependent term $\opA^{\dagger}\obs$ cancels in differences, for any two inputs $\hat{\bx},\hat{\bx}'$,
    \begin{equation*}
        \DC(\hat{\bx})-\DC(\hat{\bx}')=P_{\ker\opA}(\hat{\bx}-\hat{\bx}'),
    \end{equation*}
    so $\DC$ is non-expansive with $L_{DC}\coloneqq\|P_{\ker\opA}\|_{\mathrm{op}}\in\{0,1\}$.
    \item[(L2)] The denoiser $\denoiser{\tk{*}}{}{\cdot}[\param]$ is $L_D$-Lipschitz in its input. For the exact MMSE denoiser, the Tweedie--Stein identity gives
    \begin{equation*}
        \nabla_{\bx}\,\denoiser{\tk{*}}{}{\bx}[\param]
        =
        \frac{\alpha_{\tk{*}}}{\sigma_{\tk{*}}^2}\,
        \Sigma_{0\mid \tk{*}}(\bx),
    \end{equation*}
    where $\Sigma_{0\mid \tk{*}}(\bx)\coloneqq\pV[X_0\mid X_{\tk{*}}=\bx]$ is the posterior covariance under the prior. We assume the conditional covariance is uniformly bounded in operator norm,
    \begin{equation*}
        M_{\tk{*}} \;\coloneqq\; \sup_{\bx}\,\bigl\|\Sigma_{0\mid \tk{*}}(\bx)\bigr\|_{\mathrm{op}} \;<\; \infty,
    \end{equation*}
    which holds in particular for any compactly supported prior (e.g.\ bounded pixel intensities). One may then take
    \begin{equation*}
        L_D \;\le\; \frac{\alpha_{\tk{*}}}{\sigma_{\tk{*}}^2}\,M_{\tk{*}}.
    \end{equation*}
    (A finite \emph{unconditional} second moment $\pE\|X_0\|^2$ does \emph{not} suffice: by the law of total variance it bounds only the \emph{average} conditional covariance $\pE[\Sigma_{0\mid\tk{*}}(X_{\tk{*}})]$, not its supremum over $\bx$, which can be inflated on low-probability conditioning events.)
\end{enumerate}
\end{assumption}

\begin{proposition}[Phase~1 contraction]
\label{prop:phase1_contraction}
Under Assumption~\ref{ass:phase1_lipschitz}, the Phase~1 law update $T$ satisfies
\begin{equation}
\label{eq:phase1_rate}
    W_1(T\mu,T\mu')
    \;\le\;
    \rho\,W_1(\mu,\mu'),
    \qquad
    \rho \coloneqq \alpha_{\tk{*}} L_{DC}L_D.
\end{equation}
In particular, if $\rho<1$, then $T$ has a unique fixed-point law $\mu^\infty$, and for any initialization $\mu^0$,
\begin{equation}
\label{eq:phase1_geometric}
    W_1(\mu^N,\mu^\infty)\;\le\;\rho^N\,W_1(\mu^0,\mu^\infty).
\end{equation}
\end{proposition}

\begin{proof}
Write $D_*(\bx)=\denoiser{\tk{*}}{}{\bx}[\param]$. Let $(X,X')$
be an optimal $W_1$-coupling of $\mu$ and $\mu'$, so that
\begin{equation*}
    \pE\|X-X'\|=W_1(\mu,\mu').
\end{equation*}
Given $(X,X')$, draw a common $\xi\sim\gauss(\zero,\Id)$ and define
\begin{equation*}
    \begin{aligned}
        Y
        &\coloneqq
        \alpha_{\tk{*}}\DC(D_*(X))+\sigma_{\tk{*}}\xi,\\
        Y'
        &\coloneqq
        \alpha_{\tk{*}}\DC(D_*(X'))+\sigma_{\tk{*}}\xi.
    \end{aligned}
\end{equation*}
By construction, $(Y,Y')$ is a coupling of $T\mu$ and $T\mu'$. Therefore, using Assumption~\ref{ass:phase1_lipschitz},
\begin{equation*}
    \begin{aligned}
        W_1(T\mu,T\mu')
        &\le
        \pE\|Y-Y'\| \\
        &=
        \alpha_{\tk{*}}\,
        \pE\bigl\|\DC(D_*(X))-\DC(D_*(X'))\bigr\| \\
        &\le
        \alpha_{\tk{*}}L_{DC}\,
        \pE\|D_*(X)-D_*(X')\| \\
        &\le
        \alpha_{\tk{*}}L_{DC}L_D\,\pE\|X-X'\| \\
        &=
        \rho\,W_1(\mu,\mu').
    \end{aligned}
\end{equation*}
This proves~\eqref{eq:phase1_rate}. If $\rho<1$, the Banach fixed-point theorem yields a unique fixed point $\mu^\infty$ of $T$, and iterating~\eqref{eq:phase1_rate} gives~\eqref{eq:phase1_geometric}.
\end{proof}

\begin{remark}[Contraction regime]
\label{rem:contraction_regime}
The rate~\eqref{eq:phase1_rate} is $\rho=\alpha_{\tk{*}}L_{DC}L_D\le \alpha_{\tk{*}}^2 M_{\tk{*}}/\sigma_{\tk{*}}^2$, using $L_{DC}\le1$ and Assumption~\ref{ass:phase1_lipschitz}(L2). Whether $\rho<1$ thus depends on the size of the conditional covariance $M_{\tk{*}}$ relative to $\sigma_{\tk{*}}^2$. For a Gaussian prior $X_0\sim\gauss(\zero,\Id)$ one has $\Sigma_{0\mid\tk{*}}(\bx)=\sigma_{\tk{*}}^2\Id$ (VP), so $M_{\tk{*}}=\sigma_{\tk{*}}^2$ and
\begin{equation*}
    \rho \;\le\; \alpha_{\tk{*}}^2 \;<\; 1 \qquad\text{for every } \tk{*}\in(0,1),
\end{equation*}
i.e.\ the Phase~1 loop contracts at all noise levels. We caution that the contraction rate alone does \emph{not} single out the empirically preferred window $\tk{*}\in[0.4,0.6]$ (Appendix~\ref{sec:ablation}): that preference reflects the trade-off between warm-start quality and the length of the truncated reverse segment left for Phase~2 (Lemma~\ref{lem:warm_start_truncation}), together with the Phase~1 bias $\bm{B}$ (Remark~\ref{rem:phase1_bias}), rather than any failure of Phase~1 contraction. Replacing $M_{\tk{*}}$ by a crude constant bound $M_{\tk{*}}\le M$ would instead give $\rho\le\alpha_{\tk{*}}^2 M/\sigma_{\tk{*}}^2$, whose apparent blow-up as $\tk{*}\to0$ is an artifact of discarding the $\sigma_{\tk{*}}^2$ decay of $M_{\tk{*}}$, not a genuine loss of contraction.
\end{remark}

\begin{remark}[Bias of the Phase~1 fixed point]
\label{rem:phase1_bias}
Convergence of the Phase~1 loop does not mean convergence to the exact posterior marginal. The fixed point $\mu^\infty$ is the invariant law of the approximate update $T$, whereas the desired target is $\pi_{\tk{*}}(\cdot\mid\obs)$. We define the \textbf{Phase~1 bias}
\begin{equation*}
    \bm{B} \;\coloneqq\; W_1\!\left(\mu^\infty,\pi_{\tk{*}}(\cdot\mid\obs)\right).
\end{equation*}
This bias has two main sources. First, the learned denoiser $\denoiser{\tk{*}}{}{\cdot}[\param]$ is only an approximation to the exact MMSE denoiser. Second, the unregularized Phase~1 data-consistency step enforces the hard constraint $\opA\bx=\obs$, while the noisy posterior for $\obs=\opA\bx_0+\sigma_y\nu$ typically concentrates near this set rather than exactly on it. In the noiseless limit $\sigma_y\to0$, the second contribution disappears because the posterior itself concentrates on $\{\bx:\opA\bx=\obs\}$. For nonzero measurement noise, this contribution can be reduced by replacing the hard projection with a regularized data-consistency step matched to the likelihood noise level, as in the role of $\lambda$ in Phase~2.
\end{remark}

\paragraph{End-to-end posterior-error bound.}
Combining Proposition~\ref{prop:phase1_contraction} with Lemma~\ref{lem:warm_start_truncation} via the triangle inequality
\begin{equation*}
    \bm{\epsilon_*}
    \le
    W_1(\mu^N,\mu^\infty)+\bm{B}
\end{equation*}
yields the following corollary.

\begin{corollary}[End-to-end warm-started error]
\label{cor:end_to_end}
Under Assumptions~\ref{ass:warm_start_stability} and~\ref{ass:phase1_lipschitz}, and when $\rho=\alpha_{\tk{*}}L_{DC}L_D<1$, the Phase~2 output law $\mu_0$ of \textsc{Spin} satisfies
\begin{equation}
\label{eq:end_to_end_bound}
    W_1\!\left(\mu_0,\pi_0(\cdot\mid\obs)\right)
    \;\le\;
    L\cdot\rho^N\cdot W_1(\mu^0,\mu^\infty)
    \;+\;
    L\cdot\bm{B}
    \;+\;
    \bm{e_{tr}},
\end{equation}
where $\mu^0$ is the Phase~1 initialization law, and $L,\bm{e_{tr}}$ are as defined in Lemma~\ref{lem:warm_start_truncation}.
\end{corollary}

The bound separates three effects. The first term decays geometrically with the number of Phase~1 iterations $N$. The second term is the irreducible Phase~1 bias from the approximate denoiser and approximate data-consistency map. The third term, $\bm{e_{tr}}$, is the same truncated reverse-sampling error incurred by any sampler that starts at $\tk{*}$ and runs to $0$. Thus, compared with Lemma~\ref{lem:warm_start_truncation}, Corollary~\ref{cor:end_to_end} replaces the abstract condition $L\bm{\epsilon_*}<\bm{e_{sk}}$ with an explicit description of how Phase~1 reduces $\bm{\epsilon_*}$ as $N$ grows.

\begin{remark}[Conservative re-noising]
\label{rem:phase1_conservative}
The implementation uses conservative re-noising rather than the fresh-noise proxy analyzed above. It reuses the predicted noise
\begin{equation*}
    \epsilon^k=(\bx^k_{\tk{*}}-\alpha_{\tk{*}}\hat{\bx}^k_0)/\sigma_{\tk{*}},
\end{equation*}
which introduces an explicit feedthrough from the current noisy iterate $\bx^k_{\tk{*}}$. A direct Lipschitz bound for this variant is therefore looser than~\eqref{eq:phase1_rate}. The qualitative fixed-point picture is unchanged, however: the step still updates the clean component and returns to the same noise level. Empirically, this conservative variant converges faster because it reduces the variance injected per iteration from $\Id$ (the full fresh-noise replacement of~\eqref{eq:phase1_steps}) to $\sigma_{\tk{*}}^4\Id$ (Appendix~\ref{sec:renoising_app}).
\end{remark}

\begin{remark}[Latent space]
\label{rem:phase1_latent}
For latent diffusion, data consistency is applied after decoding, so $\opA$ is replaced by $\opA\circ\decoder$, and the optimized image is encoded again before re-noising. Under Lipschitz assumptions on $\decoder$ and $\encoder$, the same argument applies with $L_{DC}$ in~\eqref{eq:phase1_rate} replaced by $L_{DC}L_{\decoder}L_{\encoder}$. The conclusion is unchanged as long as the resulting $\rho$ is below one.
\end{remark}

\begin{remark}[Nonlinear $\opA$]
\label{rem:nonlinear_A}
For nonlinear $\opA$, $\DC$ is no longer an exact affine projection; it is an approximate least-squares solve. The global projection argument in Assumption~\ref{ass:phase1_lipschitz}(L1) should then be read locally. If $\opA$ has a well-conditioned full-rank Jacobian on the region visited by the iterates, the data-consistency map is locally Lipschitz, giving an analogous local contraction statement. A uniform global rate is not guaranteed in this case, but empirically the Phase~1 loop remains stable on the nonlinear tasks considered in Section~\ref{sec:experiments}.
\end{remark}

\subsection{Warm-start truncation error}
\label{sec:warm_start_error}

We prove Lemma~\ref{lem:warm_start_truncation} by tracking how local transition errors propagate, in $W_1$, along (i) the warm-started reverse trajectory $\tk{*}\to0$ used by \textsc{Spin} and (ii) the full reverse trajectory $1\to0$ used by a baseline sampler initialized from the endpoint Gaussian. The two bounds share the same one-step recursion; the difference is only the initialization point and how far the recursion must be unrolled.

\paragraph{Setup.}
Let
\begin{equation}
    0=\tk{0}<\cdots<\tk{J}=\tk{*}<\cdots<\tk{K}=1
\end{equation}
be the reverse-time discretization, and let $\mathsf{P}_i$ denote the approximate reverse Markov kernel used by the sampler from $\tk{i+1}$ to $\tk{i}$, i.e., the Markov operator on probability measures induced by the conditional density $\pdata{\tk{i}\tbar\tk{i+1}}{\bx_{\tk{i+1}}}{}[\param]$ defined in Section~\ref{sec:background}, extended at scheduled guidance times to include the \textsc{Spin} data-consistency and re-noising correction. The notation $\mathsf{P}_i$ is used throughout this subsection to keep the measure-action notation $\mu\mathsf{P}_i$ readable; it refers to the same Markov kernel as the conditional density above.

\begin{assumption}[Stable reverse error propagation]
\label{ass:warm_start_stability}
For each $i$, the approximate reverse kernel $\mathsf{P}_i$ is $L_i$-Lipschitz in $W_1$:
\begin{equation*}
    W_1(\rho\mathsf{P}_i,\eta\mathsf{P}_i)
    \;\le\;
    L_i\, W_1(\rho,\eta)
    \qquad\text{for all probability measures }\rho,\eta.
\end{equation*}
\end{assumption}

\paragraph{Key quantities.}
Define the \textbf{local posterior-transition error} and the \textbf{cumulative propagation factor}
\begin{equation*}
    \delta_i
    \;\coloneqq\;
    W_1\!\left(
        \pi_{\tk{i+1}}(\cdot\mid\obs)\mathsf{P}_i,\;
        \pi_{\tk{i}}(\cdot\mid\obs)
    \right),
    \qquad
    G_i \;\coloneqq\; \prod_{r=0}^{i-1}L_r,
\end{equation*}
with the convention $G_0=1$. The scalar $\delta_i$ measures how far one application of $\mathsf{P}_i$ takes the true posterior marginal at $\tk{i+1}$ from the true posterior marginal at $\tk{i}$, while $G_i$ is the Lipschitz amplification of any error sitting at time $\tk{i}$ as it is pushed forward by $\mathsf{P}_{i-1},\ldots,\mathsf{P}_0$ down to $\tk{0}=0$.

The three named scalars used in Lemma~\ref{lem:warm_start_truncation} are then
\begin{equation*}
    \underbrace{L}_{\text{\bfseries stability}}
    \;\coloneqq\; G_J,
    \qquad
    \underbrace{\bm{e_{tr}}}_{\text{\bfseries truncated error}}
    \;\coloneqq\; \sum_{i=0}^{J-1} G_i\,\delta_i,
    \qquad
    \underbrace{\bm{e_{sk}}}_{\text{\bfseries skipped error}}
    \;\coloneqq\; \sum_{i=J}^{K-1} G_i\,\delta_i.
\end{equation*}
Intuitively, $\bm{e_{tr}}$ accumulates the per-step posterior-transition errors over the truncated segment $\tk{*}\to0$ that \textsc{Spin} traverses, while $\bm{e_{sk}}$ accumulates them over the skipped segment $1\to\tk{*}$ that a full sampler must additionally traverse. Each contribution is propagated to $\tk{0}=0$ through the corresponding Lipschitz factor $G_i$.

\begin{proof}[Proof of Lemma~\ref{lem:warm_start_truncation}]
Consider first any sequence of laws $\{\eta_{\tk{i}}\}_{i=0}^{K}$
satisfying $\eta_{\tk{i}}=\eta_{\tk{i+1}}\mathsf{P}_i$, and define
\begin{equation*}
    e_i \;\coloneqq\; W_1\!\left(\eta_{\tk{i}},\pi_{\tk{i}}(\cdot\mid\obs)\right).
\end{equation*}
By the triangle inequality and Assumption~\ref{ass:warm_start_stability},
\begin{equation}
\label{eq:one_step_recursion}
\begin{aligned}
    e_i
    &\le
    W_1\!\left(
        \eta_{\tk{i+1}}\mathsf{P}_i,\;
        \pi_{\tk{i+1}}(\cdot\mid\obs)\mathsf{P}_i
    \right)
    +
    W_1\!\left(
        \pi_{\tk{i+1}}(\cdot\mid\obs)\mathsf{P}_i,\;
        \pi_{\tk{i}}(\cdot\mid\obs)
    \right) \\
    &\le L_i\, e_{i+1} + \delta_i.
\end{aligned}
\end{equation}
Thus any trajectory driven by the kernels $\mathsf{P}_i$ satisfies the
same recursion.

For the warm-started sampler, set $\eta_{\tk{J}}=\mu_{\tk{*}}$ and
propagate through $\eta_{\tk{i}}=\eta_{\tk{i+1}}\mathsf{P}_i$ for
$i=J-1,\ldots,0$. Unrolling~\eqref{eq:one_step_recursion} gives
\begin{equation*}
    e_0
    \;\le\;
    G_J\, e_J \;+\; \sum_{i=0}^{J-1} G_i\,\delta_i
    \;=\; L\, e_J + \bm{e_{tr}}.
\end{equation*}
Since $e_J=\bm{\epsilon_*}$, we obtain
\begin{equation*}
    W_1\!\left(\mu_0,\pi_0(\cdot\mid\obs)\right) \;\le\; L\cdot\bm{\epsilon_*} + \bm{e_{tr}}.
\end{equation*}

For the full sampler, write the laws as $\nu_{\tk{i}}$ and initialize at
the exact endpoint marginal
\begin{equation*}
    \nu_{\tk{K}}
    \;=\; \pi_{\tk{K}}(\cdot\mid\obs)
    \;=\; \pi_1(\cdot\mid\obs)
    \;=\; \gauss(\zero,\Id),
\end{equation*}
where the last equality follows from $\alpha_1=0$, $\sigma_1=1$, and the
independence of $X_1$ from $X_0$ and $\obs$. Hence
$W_1(\nu_{\tk{K}},\pi_{\tk{K}}(\cdot\mid\obs))=0$, and unrolling
\eqref{eq:one_step_recursion} from $K-1$ to $0$ yields
\begin{equation*}
    W_1\!\left(\nu_0,\pi_0(\cdot\mid\obs)\right)
    \;\le\; \sum_{i=0}^{K-1} G_i\,\delta_i
    \;=\; \bm{e_{tr}} + \bm{e_{sk}}.
\end{equation*}
The two upper bounds differ only by $L\bm{\epsilon_*}$ versus
$\bm{e_{sk}}$, which gives the stated comparison. The case
$\bm{\epsilon_*}=0$ follows immediately.
\end{proof}

\begin{remark}[Comparison of upper bounds]
Lemma~\ref{lem:warm_start_truncation} compares the canonical $W_1$ upper bounds for the two samplers, not their actual errors. The bounds are tight under standard assumptions, but the actual error ordering can in principle differ; our experiments confirm that the warm-started ordering predicted here matches the empirical ordering.
\end{remark}

\begin{remark}[Interpretation of $\delta_i$ at guidance steps]
At scheduled guidance times, $\mathsf{P}_i$ is the composite ``denoise + pixel-MAP + re-noise + DDIM'' kernel, and $\delta_i$ measures how much this composite kernel deviates from the true posterior transition---generally larger than at unmodified DDIM steps. The lemma is agnostic to this distinction; its content is that warm-starting at $\tk{*}$ replaces $\bm{e_{sk}}$ by $L\cdot\bm{\epsilon_*}$, regardless of how $\delta_i$ is distributed across the timeline.
\end{remark}

\clearpage
\section{Method Comparison}
\label{sec:method_comparison}

\begin{table*}[!t]
    \centering
    \caption{Comparison of training-free diffusion/consistency-model inverse-problem solvers. ``VJP'' indicates backprop through the denoiser or encoder/decoder (latent diffusion models only).}
    \label{tab:method_comparison_supp}
    \fontsize{7.3}{8.5}\selectfont
    \setlength{\tabcolsep}{2pt}
    \begin{tabular}{@{}
        >{\raggedright\arraybackslash}p{2.5cm}
        >{\raggedright\arraybackslash}p{3.3cm}
        >{\raggedright\arraybackslash}p{4.5cm}
        >{\centering\arraybackslash}p{1.2cm}
        >{\centering\arraybackslash}p{1.4cm}
        >{\raggedright\arraybackslash}p{3.5cm}
        @{}}
    \toprule
    \textbf{Method} & \textbf{Target} & \textbf{Kernel / DC step} & \textbf{\shortstack{VJP\\den.?}} & \textbf{\shortstack{VJP\\enc/dec?}} & \textbf{Notes} \\
    \midrule
    \textsc{DPS}, \textsc{PGDM}\\[-0.25em]{\scriptsize\citep{chung2022diffusion,song2023pseudoinverse}} &
    Posterior via DPS-style score approximation &
    Reverse diffusion + per-step likelihood guidance &
    Yes & Yes$^\dagger$ &
    General $\mathcal{A}$; PGDM uses known/generalized $h^\dagger$; high memory/NFE \\
    \addlinespace
    \textsc{DAPS}, \textsc{RED-Diff}\\[-0.25em]{\scriptsize\citep{zhang2025improving,mardani2023variational}} &
    Decoupled / variational posterior &
    Annealed MCMC / variational update without denoiser VJPs &
    No & Often$^\dagger$ &
    Inner sampling or optimization; DAPS latent variants may use encoder/decoder gradients \\
    \addlinespace
    \textsc{DDNM}, \textsc{DDRM}\\[-0.25em]{\scriptsize\citep{wang2022zero,kawar2022denoising}} &
    Linear posterior / projection &
    Denoise + pseudo-inverse / SVD projection (DC) &
    No & No &
    Needs linear $\mathcal{A}$; DDNM uses pseudo-inverse, DDRM uses SVD \\
    \addlinespace
    \textsc{DiffPIR}\\[-0.25em]{\scriptsize\citep{diffpir}} &
    PnP / proximal-MAP &
    Denoise $\leftrightarrow$ pixel-space prox/grad step (DC) &
    No & No &
    Often applies DC at every step \\
    \addlinespace
    \textsc{PnP-DM}\\[-0.25em]{\scriptsize\citep{wu2024principled}} &
    Split sampling / PnP posterior &
    Prior step + likelihood step (often Langevin) &
    No & No &
    Inner MCMC/optimization loops \\
    \addlinespace
    \textsc{ReSample}\\[-0.25em]{\scriptsize\citep{song2023resample}} &
    Hard DC for latent diffusion &
    Skipped DC subproblem + resample to noisy manifold &
    No & Often$^\dagger$ &
    Inner optimization + resampling, typically not every reverse step \\
    \addlinespace
    \textsc{Spin} (ours) &
    Approx.\ split inference (posterior time-marginals) &
    Phase 1: fixed $t_*$ loop; Phase 2: DDIM jumps + pixel opt + noise-preserving re-noise &
    No & No &
    Gradients only through $\mathcal{A}$; sparse schedule + warm start \\
    \bottomrule
    \end{tabular}

    \vspace{1mm}
    {\footnotesize $^\dagger$Only relevant for latent diffusion models (pixel-space models have identity encoder/decoder).}
\end{table*}

\clearpage
\section{Extended Related Work}
\label{sec:related_work_ext}

This appendix expands on families of training-free inverse-problem solvers that are adjacent to but less directly comparable with \textsc{Spin}. We discuss them here to keep the main related-work section focused on methods that share \textsc{Spin}'s core ingredients (avoiding denoiser/encoder VJPs, sparse data consistency, and warm starts).

\paragraph{Latent inverse solvers.}
Several recent methods adapt diffusion priors to latent or text-to-image models, including PSLD~\citep{rout2023solving}, ReSample~\citep{song2023resample}, P2L~\citep{chung2024prompt}, TReg~\citep{kim2025regularization}, and LATINO-PRO~\citep{spagnoletti2025latino-pro}. These works address latent-specific issues such as linear inverse problems, hard data consistency, prompt tuning, text regularization, or few-step consistency-model inference, and they differ in whether data consistency is applied densely or through skipped/inner optimization steps. In contrast, \textsc{Spin} uses the latent model only as a prior: data consistency is solved in pixel space after decoding and then re-encoded, avoiding denoiser and encoder/decoder VJPs while preserving the same warm-start and sparse-guidance structure.

\paragraph{Other warm-start strategies.}
Several earlier works explore alternatives to pure-noise initialization that are not aimed at general inverse problems with explicit measurement operators. SDEdit~\citep{meng2021sdedit} starts the reverse process from an intermediate noise level using a noisy guide image, primarily for stroke-based image synthesis and editing. ILVR~\citep{choi2021ilvr} iteratively injects low-frequency reference information into the reverse trajectory for conditional generation. While these methods share the high-level idea of skipping early reverse steps, they do not enforce data consistency with respect to a forward operator $\mathcal{A}$ and do not provide a guidance schedule, so we treat them as conceptually adjacent rather than direct baselines.

\paragraph{Consistency and flow-based guidance.}
Consistency models~\citep{song2023consistency} have also been adapted to inverse problems, e.g., CM4IR~\citep{garber2025zeroshot}, and related flow or rectified-flow solvers include FlowDPS~\citep{kim2025flowdps}, PnP-Flow~\citep{martin2024pnp}, and FLAIR~\citep{erbach2025flair}; see also~\citep{lipman2024flowmatchingguidecode}. These works are adjacent to our goal of decoupling prior evolution from data consistency, but \textsc{Spin} is developed for DDPM/DDIM priors, where re-noising is a noise-preserving coupling (Appendix~\ref{sec:algorithm}); extending the two-phase design to flow matching is left for future work.

\clearpage
\section{Sampling Schedule}
\label{sec:sampling-schedule}

\paragraph{Motivation:}
In Phase 2 of our method, we perform optimization at a selected subset of $M$ timesteps from the diffusion trajectory. A critical design choice is \textbf{where} along this trajectory to concentrate our computational effort. The spacing of these timesteps---controlled by our scheduling strategy---has a significant impact on reconstruction quality.
\textbf{Key insight}: Optimization applied too late in the denoising process (near $t=0$) has limited opportunity for the denoiser to refine artifacts, while optimization applied too early (near $\tk{*}$) may be wasted on high-noise states. Our scheduling strategies allow flexible allocation of computational budget across the denoising trajectory.

\subsection{Timestep Allocation Framework}

Given $M$ guidance steps and starting timestep $\tk{*}$, we construct a grid of discrete sampler indices $\{\tk{1}, \tk{2}, \ldots, \tk{M}\}$ with $0 < \tk{1} < \cdots < \tk{M} = \tk{*}$. The normalized times reported elsewhere correspond to these indices after rescaling to $[0,1]$. Thus the index $k$ increases with noise level, while Phase~2 traverses the grid in reverse order from $\tk{M}$ to $\tk{1}$.

\paragraph{Weight-based allocation.} Rather than spacing timesteps uniformly, we assign a weight $w_i$ to each position $i \in \{1, \ldots, M\}$ and allocate spacing proportionally. Higher $w_i$ produces a larger gap between adjacent scheduled points in that region, while smaller $w_i$ produces denser guidance. We first define cumulative weights and their integer offsets
\begin{equation*}
S_i \eqdef \sum_{j=1}^{i} w_j,
\qquad
C_i \eqdef \mathrm{floor}\!\left((\tk{*} - 1) \cdot \frac{S_i}{S_M}\right),
\qquad
C_0 \eqdef 0,
\end{equation*}
Equivalently, the adjacent gaps are $C_i-C_{i-1}$. Since the cumulative spacing sum telescopes and $C_M=\tk{*}-1$, the scheduled timesteps are simply
\begin{equation*}
\tk{k} = C_k + 1, \qquad k = 1, \ldots, M,
\end{equation*}
so that $\tk{M} = \tk{*}$ by construction.

\subsection{Scheduling Strategies}

We compare a uniform baseline with five non-uniform scheduling strategies that concentrate optimization steps at different noise levels. The weight functions for each strategy are summarized in Table~\ref{tab:schedule_weights}.

\begin{table}[h]
\centering
\caption{Weight functions for different scheduling strategies. Weights are computed for $i \in \{1, \ldots, M\}$ and used to derive the timestep grid via the construction above.}
\label{tab:schedule_weights}
\begin{tabular}{lcc}
\toprule
\textbf{Strategy} & \textbf{Weight Function} $w_i$ & \textbf{Parameters} \\
\midrule
Uniform & $1$ & --- \\
Linear & $i + 1$ & --- \\
Polynomial & $(i+1)^p$ & $p > 1$ \\
Exponential & $\rho^{i}$ & $\rho > 1$ \\
Gaussian & $\exp\!\left(\dfrac{(i - \mu\,M)^2}{2\sigma^2}\right)$ & $\mu \in [0,1],\, \sigma > 0$ \\
Beta & $\mathrm{Beta}(x_i;\, a, b)^{-1}$, $x_i=\dfrac{i-\frac{1}{2}}{M}$ & $a, b > 0$ \\
\bottomrule
\end{tabular}
\end{table}

\paragraph{Uniform Schedule:} $w_i = 1$ for all $i$. Equal spacing between all guidance steps. This serves as our baseline, distributing computational effort uniformly across the denoising trajectory.

\paragraph{Linear Schedule:} $w_i = i + 1$.
Because $i$ increases with noise level, the gaps grow toward the high-noise end near $\tk{*}$ and shrink toward the low-noise end. This makes guidance points denser later in the reverse trajectory.

\paragraph{Polynomial Schedule:} $w_i = (i+1)^p$ with power $p > 1$.
Similar to linear but with adjustable aggressiveness controlled by power $p$. Higher values of $p$ make the high-noise gaps larger and concentrate guidance more strongly toward lower-noise timesteps. Common choices include $p = 2$ (quadratic growth, moderate concentration) or $p = 3$ (cubic growth, aggressive concentration). We find that this schedule works best for \texttt{FFHQ LDM}.

\paragraph{Exponential Schedule:} $w_i = \rho^{i}$ with growth rate $\rho > 1$.
Rapidly increases the gaps toward the high-noise end using exponential growth. The rate $\rho$ controls how aggressively guidance is concentrated toward low-noise regions. Typical values are $\rho \in [1.5, 2.0]$.

\paragraph{Gaussian Schedule (Default):}
\begin{equation*}
w_i = \exp\!\left(\frac{(i - \mu\,M)^2}{2\sigma^2}\right)
\end{equation*}
with fractional center $\mu \in [0,1]$ and width $\sigma > 0$ (in units of the guidance-step index, so $\sigma$ here corresponds to $\sigma_G$ in the main text). The weight is proportional to the reciprocal of a Gaussian density centered at $\mu\,M$, so $w_i$ is small near the center and large at the endpoints, producing close timestep spacing---and thus dense guidance---around the chosen position $\mu$. For example, $\mu = 0.5$ centers around mid-trajectory, while smaller $\sigma$ creates tighter concentration.
This schedule works best for most tasks. Our experiments (Table~\ref{tab:schedule_comparison}) show this provides the best balance, typically with $\mu \in [0.3, 0.5]$ and $\sigma \in [10, 15]$.
By concentrating optimization in the intermediate noise regime (approximately $t \in [0.4, 0.6]$), the denoiser has sufficient structure to work with while retaining enough trajectory to refine artifacts before reaching the final output.

\paragraph{Beta Schedule:} Let $x_i=(i-\frac{1}{2})/M$ and set $w_i = \mathrm{Beta}(x_i;\, a, b)^{-1}$, where
\begin{equation*}
\mathrm{Beta}(x;\, a, b) = \frac{x^{a-1}(1-x)^{b-1}}{\mathrm{B}(a, b)}
\end{equation*}
is the Beta probability density with shape parameters $a, b > 0$ and $\mathrm{B}(a, b)$ the Beta function.
The reciprocal flips the role of high- and low-density regions: small $w_i$ produces small spacing, so guidance concentrates where the Beta density is \emph{high}. Symmetric concentration occurs when $a = b$; $a < b$ skews concentration toward later iterations (lower noise), while $a > b$ skews it toward earlier iterations (higher noise).

\clearpage
\section{Competitors}
\label{sec:competitors}

In this section, we detail the implementation of baseline methods. We use the hyperparameters suggested by the original authors and perform additional tuning for each dataset when specific values are not provided.

\paragraph{\textsc{DPS.}}
We implemented the method from \citet{chung2022diffusion}, adopting the hyperparameters for each task as specified in \citet{chung2022diffusion} (App. D). For tasks not covered in the original work, we performed our own tuning: specifically, we set $\psi = 0.2$ for JPEG 2\%, $\psi = 0.07$ for High Dynamic Range.

\paragraph{\textsc{DiffPIR.}}
We implemented \citet{diffpir} to ensure compatibility with our codebase, using the hyperparameters from the official released version. We attempted to extend the method to nonlinear problems following the guidelines in \citet{diffpir} (Eqn. (13)); however, the algorithm diverged in these cases. We were unable to resolve this issue as neither the paper nor the released code provide examples for nonlinear problems. For motion blur, \citet{diffpir} offers an FFT-based solution that is only applicable to circular convolution. Since we adopt the experimental setup of \citet{chung2022diffusion}, which employs convolution with reflection padding, we exclude \textsc{DiffPIR} from the motion blur evaluation.

\paragraph{\textsc{DDNM}~\citep{wang2022zero}.}
We adapted the implementation from the released code. The original code provides separate classes for each degradation operator in the module \texttt{functions/svd\_operators.py}. We refactored these into a single unified class to support all SVD-decomposable linear degradation operators. We observe that \textsc{DDNM} exhibits instability for operators whose \texttt{SVD} decomposition is susceptible to numerical errors, such as Gaussian blur with wide convolution kernels. This instability arises from the algorithm's reliance on the pseudo-inverse of the operator.

\paragraph{\textsc{Red-Diff}~\citep{mardani2023variational}.}
We employed the implementation of \textsc{Red-Diff} from the released code. For linear problems, we initialize the variational optimization using the pseudo-inverse of the observation. For nonlinear problems, where the pseudo-inverse is unavailable, we initialize the optimization with a sample drawn from a standard Gaussian distribution.

\paragraph{\textsc{PGDM}.}
We use the implementation provided in the \textsc{Red-Diff} repository, as several authors of \textsc{Red-Diff} are also co-authors of \textsc{PGDM}. We note a minor deviation from the algorithm presented in \citet{song2023pseudoinverse} (Algorithm 1): in the final step, the guidance term $g$ is scaled by $\alpha_{t-1}\alpha_t$ in the implementation, whereas the original formulation scales it by $\alpha_t\sqrt{\alpha_t}$. We find that this modification improves performance across most tasks, with the exception of JPEG dequantization, for which the original $\alpha_t$ scaling yields better results.

\paragraph{\textsc{PSLD}.}
We implemented the PSLD algorithm from \citet{kadkhodaie2020solving} and configured the hyperparameters for each task based on the publicly available implementation.

\paragraph{\textsc{ReSample}.}
We modified the original code from the authors to enable direct adjustment of key hyperparameters: the tolerance $\varepsilon$ and maximum iteration count $N$ for the optimization problems enforcing hard data consistency, as well as the variance scaling factor $\gamma$ for the stochastic resampling distribution. We observe that the algorithm is sensitive to $\varepsilon$, with optimal reconstructions achieved by setting it equal to the noise level of the inverse problem across all tasks and noise levels. In contrast, we find that $\gamma$ has minimal impact on reconstruction quality. To reduce computational cost, we set a maximum threshold of $N = 200$ gradient iterations.

\paragraph{\textsc{DAPS}.}
We use the official codebase and configure the hyperparameters according to \citet{zhang2025improving} (Table 7). For audio-source separation, we set $\sigma_{\max}$ and $\sigma_{\min}$ to match the values used in the sound model, and adapt the Langevin step size \texttt{lr} and standard deviation \texttt{tau} accordingly.

\paragraph{\textsc{PnP-DM}.}
We adapted the implementation from the released code, making the coupling parameter $\rho$ (including its initial value, minimum value, and decay rate) and the number of Langevin steps with their step size directly adjustable. We configure the hyperparameters following \citet{wu2024principled} (Tables 3 and 4). For inpainting tasks, although exact likelihood steps are theoretically possible via Gaussian conjugacy~\citep[Sec.~3.1]{wu2024principled}, we find that Langevin dynamics yield superior results in practice. For instance, the reconstructions in Figure 6 (left) are obtained by exact posterior sampling, whereas the right-hand side uses Langevin dynamics.

\clearpage
\section{Ablation Studies}
\label{sec:ablation}

\begin{table}[h]
    \centering
    \caption{SSIM scores for different noise schedules across tasks}
    \label{tab:schedule_comparison_ssim}
    \fontsize{7.6}{7}\selectfont
    \begin{tabular}{lccc}
    \toprule
    Schedule & \textsc{Deblur} & \textsc{Motion Deblur} & \textsc{SR4} \\
    \midrule
    Uniform         & \tabwideval{0.772}{0.057} & \tabwideval{0.768}{0.057} & \tabwideval{0.772}{0.056} \\
    Linear          & \tabwideval{0.692}{0.068} & \tabwideval{0.673}{0.076} & \tabwideval{0.691}{0.073} \\
    Polynomial      & \tabwideval{0.818}{0.042} & \tabwideval{0.754}{0.041} & \tabwideval{0.785}{0.036} \\
    Exponential     & \tabhi{BurntOrange!40}{\tabwideval{0.823}{0.045}} & \tabhi{BurntOrange!40}{\tabwideval{0.808}{0.040}} & \tabhi{BurntOrange!40}{\tabwideval{0.805}{0.036}} \\
    Beta            & \tabwideval{0.759}{0.034} & \tabwideval{0.523}{0.046} & \tabwideval{0.717}{0.032} \\
    Gaussian        & \tabhi{BurntOrange}{\tabwidebest{0.825}{0.045}} & \tabhi{BurntOrange}{\tabwidebest{0.824}{0.041}} & \tabhi{BurntOrange}{\tabwidebest{0.827}{0.038}} \\
    \bottomrule
    \end{tabular}
\end{table}

\begin{table}[h]
    \centering
    \caption{PSNR scores for different noise schedules across tasks}
    \label{tab:schedule_comparison_psnr}
    \fontsize{8.0}{7}\selectfont
    \begin{tabular}{lccc}
    \toprule
    Schedule & \textsc{Deblur} & \textsc{Motion Deblur} & \textsc{SR4} \\
    \midrule
    Uniform         & \tabwideval{26.42}{1.71} & \tabwideval{26.24}{1.72} & \tabwideval{26.38}{1.69} \\
    Linear          & \tabwideval{23.63}{1.58} & \tabwideval{23.04}{1.82} & \tabwideval{23.48}{1.68} \\
    Polynomial      & \tabhi{BurntOrange}{\tabwidebest{28.78}{1.91}} & \tabwideval{27.93}{1.64} & \tabwideval{28.18}{1.60} \\
    Exponential     & \tabhi{BurntOrange!40}{\tabwideval{28.71}{2.04}} & \tabhi{BurntOrange!40}{\tabwideval{28.52}{1.85}} & \tabhi{BurntOrange!40}{\tabwideval{28.40}{1.71}} \\
    Beta            & \tabwideval{27.84}{1.58} & \tabwideval{24.70}{0.92} & \tabwideval{27.19}{1.31} \\
    Gaussian        & \tabhi{BurntOrange!40}{\tabwideval{28.72}{2.00}} & \tabhi{BurntOrange}{\tabwidebest{28.57}{2.00}} & \tabhi{BurntOrange}{\tabwidebest{28.66}{1.92}} \\
    \bottomrule
    \end{tabular}
\end{table}

\onecolumn
\section{Additional Experiments}
\label{sec:additional-experiments}

\subsection{Computational Resources}
All experiments are conducted on a single node with 8 NVIDIA A6000 GPUs. Performance metrics, including runtime and memory usage, are measured on GPUs without competing processes or memory allocation. CPU load is monitored throughout to ensure no performance degradation from CPU-GPU synchronization bottlenecks.

\subsection{Extended discussion of the 2D trajectory visualization (Figure~\ref{fig:real_sampler_trajectory_landscape})}
\label{sec:fig2-extended}

\begin{figure}[h]
    \centering
    \includegraphics[width=0.95\linewidth]{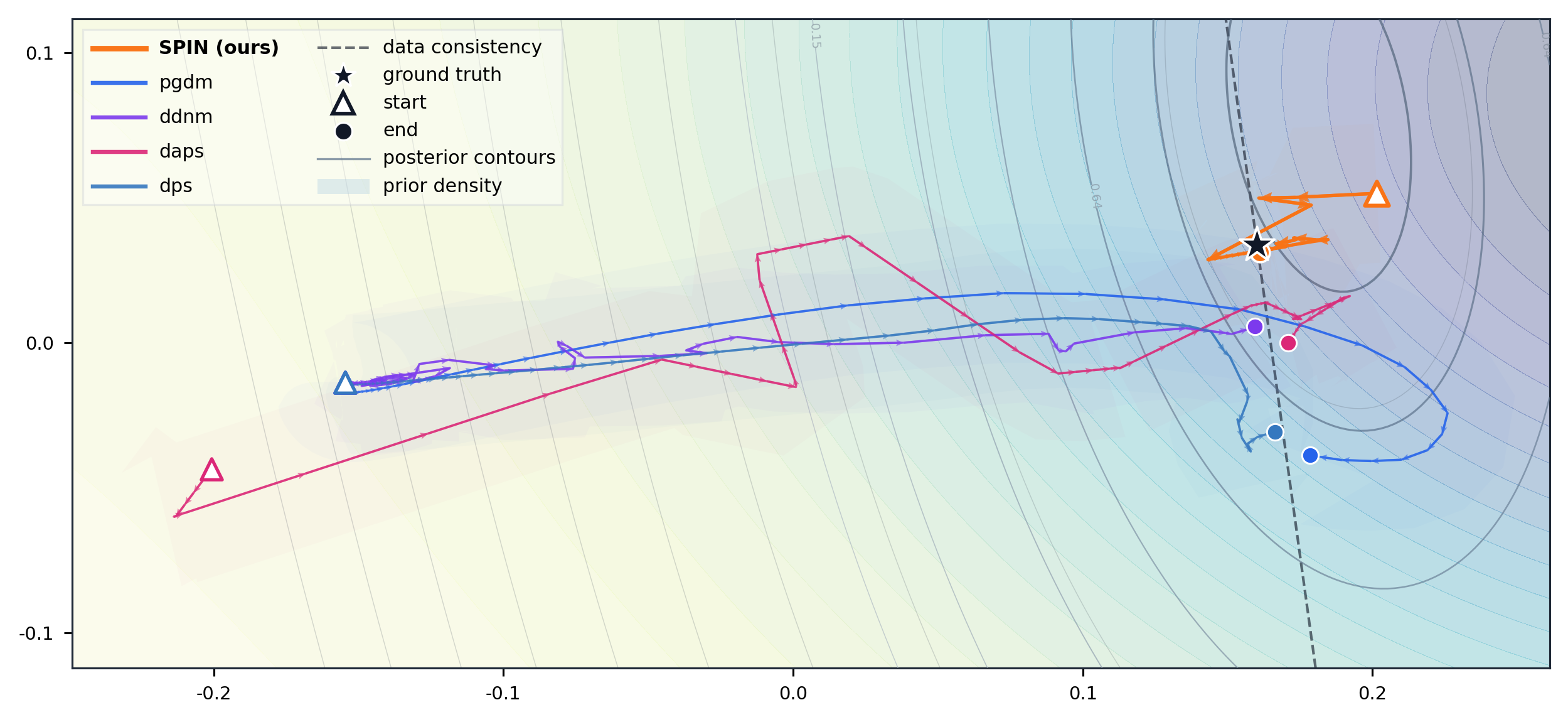}
    \caption{Reproduction of Figure~\ref{fig:real_sampler_trajectory_landscape} for ease of reference. Mean trajectories of posterior samplers on the 2D inverse problem described in this section. Filled background: exact GMM prior density. Gray contours: relative posterior density. Dashed line: data-consistency subspace $\bm{A}\bx=\obs$. Star: ground truth $\bx^\star$. Empty triangles ($\triangle$) and filled circles ($\bullet$) mark each sampler's trajectory start and endpoint, respectively.}
    \label{fig:real_sampler_trajectory_landscape_app}
\end{figure}

\paragraph{Why a 2D problem.}
Image-domain posterior sampling is hard to inspect because the state is high-dimensional and the prior is implicit in a trained network. To isolate sampler behavior from model-approximation noise, we set up a controlled 2D inverse problem in which both the prior and the score are available in closed form. This reduces the experiment to a microscope for the central conflict any posterior sampler must resolve --- prior plausibility against measurement consistency --- and lets us draw the prior density and the posterior contours on the same plane, which is impossible at image resolution.

\paragraph{Problem setup.}
The unknown is a point $\bx=[x_1,x_2]\in\mathbb{R}^2$ drawn from a four-component Gaussian mixture prior $p(\bx)=\sum_{k=1}^{4}w_k\,\gauss(\bx;\bm{\mu}_k,\Sigma_k)$ with weights $w=[0.38,\,0.26,\,0.22,\,0.14]$, means
\[
\bm{\mu}_1=[0.45,\,0.55],\;\bm{\mu}_2=[-0.55,\,0.45],\;\bm{\mu}_3=[0.02,-0.62],\;\bm{\mu}_4=[0.82,-0.25],
\]
and identity component covariances. The forward operator is a single linear measurement $\bm{A}=[0.55,\;0.65]$ with Gaussian noise of standard deviation $\sigma_\obs=0.08$. The ground-truth point is $\bx^\star=[0.35,\,0.40]$, producing $\obs=\bm{A}\bx^\star=0.4525$. Because $\bm{A}$ is $1\times 2$, the measurement defines a one-dimensional affine subspace $\{\bx:0.55\,x_1+0.65\,x_2=0.4525\}$ in the plane, drawn as the dashed \emph{data-consistency} line in Figure~\ref{fig:real_sampler_trajectory_landscape}. The data alone cannot identify a unique solution; identification only emerges where this line intersects regions of high prior density.

\paragraph{Diffusion model and exact score.}
We use the standard variance-preserving schedule with $T=1000$ linear $\beta_t$'s ranging from $10^{-4}$ to $0.02$, and the corresponding cumulative coefficients $\bar\alpha_t$. A Gaussian mixture remains a Gaussian mixture under the forward noising $\bx_t=\sqrt{\bar\alpha_t}\bx_0+\sqrt{1-\bar\alpha_t}\,\epsilon$, with time-dependent component statistics
\[
\bm{\mu}_k(t)=\sqrt{\bar\alpha_t}\,\bm{\mu}_k,\qquad
\Sigma_k(t)=\bar\alpha_t\,\Sigma_k+(1-\bar\alpha_t)\Id.
\]
We therefore evaluate the noisy-time score $\nabla_{\bx_t}\log p_t(\bx_t)$ analytically (via the closed-form responsibilities of the noised mixture) and convert it to an $\epsilon$-predictor through $\epsilon_\theta(\bx_t,t)=-\sqrt{1-\bar\alpha_t}\,\nabla_{\bx_t}\log p_t(\bx_t)$. This means the diffusion model is an \emph{oracle}: there is no training error and no neural-network approximation, so any difference between samplers reflects the sampler itself, not modeling artifacts.

\paragraph{Construction of the plotted prior and posterior.}
The filled background in Figure~\ref{fig:real_sampler_trajectory_landscape} is the exact normalized prior density $p(\bx)=\sum_k w_k\,\gauss(\bx;\bm{\mu}_k,\Sigma_k)$, evaluated on a $220\times 220$ grid via the closed-form mixture log-density. The thin gray curves are level sets of the posterior, obtained up to a constant from the unnormalized log posterior
\[
\log p(\bx\mid\obs)\;=\;\log p(\bx)\;-\;\tfrac{1}{2\sigma_\obs^2}\|\bm{A}\bx-\obs\|^2\;+\;\text{const}.
\]
Because the normalizer is irrelevant for visualization, the plot shows a \emph{relative} posterior density rescaled so that the maximum is~$1$: contour shapes are exact, but absolute density values are not. The dashed line marks $\bm{A}\bx=\obs$, the black star marks $\bx^\star$, empty triangles mark each trajectory's start, and filled circles mark its endpoint at $t=0$.

\paragraph{Samplers compared and trajectory averaging.}
We run \textsc{Spin} alongside four representative baselines --- \textsc{dps}~\citep{chung2022diffusion}, \textsc{ddnm}~\citep{wang2022zero}, \textsc{pgdm}~\citep{song2023pseudoinverse}, and \textsc{daps}~\citep{zhang2025improving} --- using the same sampler implementations as in the image experiments. Only the problem adapter (forward operator, observation, score model) is replaced by its analytic 2D version; sampler logic is unchanged. Each sampler is executed for $N=300$ independent seeds, with seed $s$ shared across samplers within a run so that trajectories are paired and directly comparable. The curves shown in the figure are the pointwise means of these $300$ runs.

\paragraph{Quantitative summary across 300 runs.}
For each sampler we report the mean Euclidean distance from the endpoint to the ground truth and the mean residual data loss $\tfrac{1}{2\sigma_\obs^2}\|\bm{A}\bx-\obs\|^2$. \textsc{Spin} attains both the smallest distance to the ground truth ($0.143$) and a near-zero data residual ($2.8\!\times\!10^{-5}$); the baselines achieve markedly worse endpoint accuracy (\textsc{pgdm} $0.304$, \textsc{daps} $0.311$, \textsc{dps} $0.350$, \textsc{ddnm} $0.382$) and either trade prior plausibility for data fit (\textsc{daps} $0.139$, \textsc{dps} $0.126$) or fit the measurement only loosely (\textsc{ddnm} $0.087$, \textsc{pgdm} $0.018$).

\begin{table}[h]
    \centering
    \small
    \begin{tabular}{lcc}
        \toprule
        Sampler & Mean dist.\ to truth $\downarrow$ & Mean data loss $\downarrow$ \\
        \midrule
        \textsc{Spin} (ours) & \textbf{0.1431} & \textbf{2.79e-05} \\
        \textsc{pgdm}        & 0.3041 & 0.0184 \\
        \textsc{daps}        & 0.3108 & 0.1391 \\
        \textsc{dps}         & 0.3498 & 0.1260 \\
        \textsc{ddnm}        & 0.3819 & 0.0870 \\
        \bottomrule
    \end{tabular}
    \caption{Mean endpoint metrics across $300$ paired seeds for the 2D landscape problem of Figure~\ref{fig:real_sampler_trajectory_landscape}.}
    \label{tab:fig2-summary}
\end{table}

\paragraph{Reading the figure.}
Three behaviors are worth highlighting. First, the baseline trajectories begin near pure noise and traverse a long, often curved path through low-density regions before reaching the posterior support, while \textsc{Spin}'s warm-start at $\tk{*}$ already lies close to a posterior mode --- so its triangle starts visibly closer to the star than any baseline triangle. Second, baseline endpoints tend to land somewhere along the data-consistency line but away from the dominant posterior mode, indicating samples that satisfy the measurement without being prior-consistent; \textsc{Spin}'s endpoint sits on a high-density posterior contour near the ground truth. Third, because trajectories are averages over $300$ seeds, the smoother and shorter \textsc{Spin} curve also reflects lower across-seed variability: the truncated reverse segment $\tk{*}\to 0$ leaves less room for stochastic drift than the full $1\to 0$ path traversed by the baselines.

\paragraph{Connection to the warm-start truncation bound.}
This visualization is the geometric counterpart of Lemma~\ref{lem:warm_start_truncation}. The baselines accumulate the skipped-segment error $\bm{e_{sk}}$ over the full reverse path, while \textsc{Spin} substitutes the smaller warm-start error $L\cdot\bm{\epsilon_*}$ at the cost of starting from $\tk{*}$ rather than $1$. The figure, together with the endpoint metrics reported above, shows that on this problem $L\cdot\bm{\epsilon_*}\ll\bm{e_{sk}}$, comfortably inside the regime in which the warm-start bound is the tighter of the two (Appendix~\ref{sec:warm_start_error}).

\paragraph{Scope.}
Because the prior, score, and posterior are all known analytically, this experiment isolates sampler behavior but does not test the model-approximation regime that dominates real image problems. It is intended as an interpretive companion to the image-domain results in Section~\ref{sec:experiments}, not as evidence at scale.

\begin{table*}[!htbp]
    \centering
    \caption[FFHQ SSIM comparison]{Quantitative comparison of different methods across various inverse problems. Results are reported as SSIM $\pm$ standard deviation. Higher is better. \textcolor{RoyalBlue}{Blue} factors are relative to \textsc{Spin}; larger values mean greater time or memory savings.}
    \label{tab:ffhq_ssim}
    \fontsize{7.3}{8.5}\selectfont
    \setlength{\tabcolsep}{2.0pt}
    \begin{tabular}{lrrrrrrrr}
    \toprule
    \multirow{2}{*}{} & \multicolumn{8}{c}{\textbf{SSIM FFHQ}} \\
    \cmidrule(lr){2-9}
    \textbf{Task}                     & \textsc{Spin}        & \textsc{DAPS}      & \textsc{Red-Diff}   & \textsc{PNP-DM}    & \textsc{DPS}       & \textsc{DDNM}      & \textsc{PGDM}      & \textsc{Diffpir}     \\
    \midrule
    Gaussian Deblur                    & \tabhi{BurntOrange}{\tabbest{0.83}{0.04}}   & \tabhi{BurntOrange!40}{\tabval{0.81}{0.05}}  & \tabhi{BurntOrange!40}{\tabval{0.81}{0.04}}  & \tabval{0.77}{0.06}  & \tabval{0.72}{0.07}  & \tabval{0.03}{0.01}  & \tabval{0.14}{0.09}  & --              \\
    Motion Deblur                      & \tabhi{BurntOrange}{\tabbest{0.82}{0.04}}   & \tabhi{BurntOrange!40}{\tabval{0.79}{0.05}}  & \tabval{0.74}{0.03}  & \tabval{0.75}{0.06}  & \tabval{0.65}{0.07}  & --               & --               & --              \\
    SR ($\times$4)                     & \tabhi{BurntOrange!40}{\tabval{0.79}{0.04}}   & \tabhi{BurntOrange}{\tabbest{0.80}{0.05}}  & \tabval{0.63}{0.03}  & \tabval{0.75}{0.05}  & \tabval{0.67}{0.08}  & \tabval{0.69}{0.03}  & \tabval{0.56}{0.03}  & --              \\
    SR ($\times$16)                    & \tabhi{BurntOrange}{\tabbest{0.63}{0.07}}   & \tabval{0.55}{0.08}  & \tabval{0.59}{0.06}  & \tabval{0.50}{0.07}  & \tabval{0.50}{0.08}  & \tabhi{BurntOrange!40}{\tabval{0.62}{0.07}}  & \tabval{0.42}{0.06}  & --              \\
    Box Inpainting                     & \tabval{0.76}{0.04}   & \tabhi{BurntOrange!40}{\tabval{0.80}{0.03}}  & \tabval{0.70}{0.03}  & \tabval{0.55}{0.03}  & \tabval{0.72}{0.06}  & \tabval{0.73}{0.03}  & \tabval{0.70}{0.03}  & \tabhi{BurntOrange}{\tabbest{0.82}{0.03}} \\
    Half Inpainting                    & \tabval{0.68}{0.05}   & \tabhi{BurntOrange!40}{\tabval{0.71}{0.04}}  & \tabval{0.63}{0.04}  & \tabval{0.43}{0.02}  & \tabval{0.66}{0.06}  & \tabval{0.65}{0.05}  & \tabval{0.59}{0.04}  & \tabhi{BurntOrange}{\tabbest{0.73}{0.05}} \\
    JPEG (QF=2)                        & \tabhi{BurntOrange}{\tabbest{0.77}{0.05}}   & \tabhi{BurntOrange!40}{\tabval{0.76}{0.05}}  & \tabval{0.72}{0.05}  & \tabval{0.71}{0.06}  & \tabval{0.60}{0.10}  & --               & --               & --              \\
    Phase Retrieval                    & \tabhi{BurntOrange!40}{\tabval{0.58}{0.21}}   & \tabhi{BurntOrange}{\tabbest{0.65}{0.23}}  & \tabval{0.53}{0.22}  & \tabval{0.54}{0.21}  & \tabval{0.39}{0.16}  & --               & --               & --              \\
    HDR                                & \tabhi{BurntOrange!40}{\tabval{0.78}{0.09}}   & \tabhi{BurntOrange}{\tabbest{0.85}{0.07}}  & \tabval{0.72}{0.09}  & \tabval{0.67}{0.13}  & \tabval{0.34}{0.34}  & --               & --               & --              \\
    \midrule
    \textbf{Memory (MB)}               & \leftcell{\tabhi{BurntOrange}{\textbf{1983}}}              & \factorcell{2095}{1.1}             & \colorfactorcell{BurntOrange!40}{1985}{1.0}             & \colorfactorcell{BurntOrange!40}{1985}{1.0}             & \factorcell{3309}{1.7}             & \factorcell{2019}{1.0}             & \factorcell{3409}{1.7}             & \colorfactorcell{BurntOrange!40}{1985}{1.0}            \\
    \textbf{Run time (sec)}            & \leftcell{\tabhi{BurntOrange}{\textbf{25}}}                & \factorcell{75}{3.0}               & \factorcell{50}{2.0}               & \factorcell{194}{7.8}              & \factorcell{105}{4.2}              & \colorfactorcell{BurntOrange!40}{47}{1.9}               & \factorcell{101}{4.0}              & \factorcell{50}{2.0}              \\
    \bottomrule
    \end{tabular}
\end{table*}

\begin{table*}[!htbp]
    \centering
    \caption[FFHQ PSNR comparison]{Quantitative comparison of different methods across various inverse problems. Results are reported as PSNR $\pm$ standard deviation. Higher is better. \textcolor{RoyalBlue}{Blue} factors are relative to \textsc{Spin}; larger values mean greater time or memory savings.}
    \label{tab:ffhq_psnr}
    \fontsize{7.3}{8.5}\selectfont
    \setlength{\tabcolsep}{2.0pt}
    \begin{tabular}{lrrrrrrrr}
    \toprule
    \multirow{2}{*}{} & \multicolumn{8}{c}{\textbf{PSNR FFHQ}} \\
    \cmidrule(lr){2-9}
    \textbf{Task}                     & \textsc{Spin}        & \textsc{DAPS}      & \textsc{Red-Diff}   & \textsc{PNP-DM}    & \textsc{DPS}       & \textsc{DDNM}      & \textsc{PGDM}      & \textsc{Diffpir}     \\
    \midrule
    Gaussian Deblur                    & \tabhi{BurntOrange}{\tabwidebest{29.0}{1.80}}  & \tabwideval{28.2}{2.08} & \tabhi{BurntOrange!40}{\tabwideval{28.7}{1.82}} & \tabwideval{25.2}{2.76} & \tabwideval{24.7}{2.07} & \tabwideval{7.8}{0.10}  & \tabwideval{13.2}{0.70} & --              \\
    Motion Deblur                      & \tabhi{BurntOrange}{\tabwidebest{28.1}{1.76}}  & \tabwideval{27.0}{1.86} & \tabhi{BurntOrange!40}{\tabwideval{27.8}{1.22}} & \tabwideval{24.7}{2.33} & \tabwideval{22.2}{1.83} & --               & --              & --              \\
    SR ($\times$4)                     & \tabhi{BurntOrange}{\tabwidebest{28.2}{1.62}}  & \tabhi{BurntOrange!40}{\tabwideval{27.5}{1.76}} & \tabwideval{26.1}{0.92} & \tabwideval{24.7}{2.21} & \tabwideval{22.8}{2.05} & \tabwideval{26.9}{1.17}  & \tabwideval{24.6}{1.22} & --              \\
    SR ($\times$16)                    & \tabhi{BurntOrange}{\tabwidebest{21.6}{1.69}}  & \tabwideval{17.8}{1.48} & \tabhi{BurntOrange!40}{\tabwideval{21.2}{1.44}} & \tabwideval{16.3}{1.10} & \tabwideval{17.9}{1.65} & \tabhi{BurntOrange}{\tabwidebest{21.6}{1.67}}  & \tabwideval{18.4}{1.23} & --              \\
    Box Inpainting                     & \tabwideval{21.7}{2.65}  & \tabwideval{22.4}{2.90} & \tabwideval{21.6}{2.60} & \tabwideval{12.5}{0.74} & \tabwideval{20.9}{2.41} & \tabhi{BurntOrange!40}{\tabwideval{22.5}{2.78}}  & \tabwideval{21.1}{2.37} & \tabhi{BurntOrange}{\tabwidebest{22.6}{3.42}} \\
    Half Inpainting                    & \tabwideval{16.0}{2.57}  & \tabwideval{15.4}{2.51} & \tabwideval{15.5}{2.51} & \tabwideval{10.8}{0.94} & \tabwideval{15.9}{2.44} & \tabhi{BurntOrange!40}{\tabwideval{16.1}{2.89}}  & \tabwideval{15.0}{2.32} & \tabhi{BurntOrange}{\tabwidebest{16.2}{2.82}} \\
    JPEG (QF=2)                        & \tabhi{BurntOrange}{\tabwidebest{26.2}{1.56}}  & \tabhi{BurntOrange!40}{\tabwideval{25.4}{1.73}} & \tabwideval{24.5}{1.19} & \tabwideval{22.4}{1.49} & \tabwideval{20.4}{1.71} & --               & --               & --              \\
    Phase Retrieval                    & \tabwideval{19.7}{6.82}  & \tabhi{BurntOrange}{\tabwidebest{21.6}{8.92}} & \tabhi{BurntOrange!40}{\tabwideval{21.5}{7.81}} & \tabwideval{18.1}{6.59} & \tabwideval{14.1}{4.32} & --               & --               & --              \\
    HDR                                & \tabhi{BurntOrange!40}{\tabwideval{22.9}{2.87}}  & \tabhi{BurntOrange}{\tabwidebest{27.1}{2.96}} & \tabwideval{21.7}{2.86} & \tabwideval{21.4}{2.12} & \tabwideval{12.9}{7.61} & --               & --               & --              \\
    \midrule
    \textbf{Memory (MB)}               & \leftcell{\tabhi{BurntOrange}{\textbf{1983}}}               & \factorcell{2095}{1.1}             & \colorfactorcell{BurntOrange!40}{1985}{1.0}             & \colorfactorcell{BurntOrange!40}{1985}{1.0}             & \factorcell{3309}{1.7}             & \factorcell{2019}{1.0}             & \factorcell{3409}{1.7}             & \colorfactorcell{BurntOrange!40}{1985}{1.0}            \\
    \textbf{Run time (sec)}            & \leftcell{\tabhi{BurntOrange}{\textbf{25}}}                 & \factorcell{75}{3.0}               & \factorcell{50}{2.0}               & \factorcell{194}{7.8}              & \factorcell{105}{4.2}              & \colorfactorcell{BurntOrange!40}{47}{1.9}               & \factorcell{101}{4.0}              & \factorcell{50}{2.0}              \\
    \bottomrule
    \end{tabular}
\end{table*}

\begin{table*}[!htbp]
    \centering
    \caption[ImageNet SSIM comparison]{Quantitative comparison of different methods across various inverse problems. Results are reported as SSIM $\pm$ standard deviation. Higher is better. \textcolor{RoyalBlue}{Blue} factors are relative to \textsc{Spin}; larger values mean greater time or memory savings.}
    \label{tab:imagenet_ssim}
    \fontsize{7.3}{8.5}\selectfont
    \setlength{\tabcolsep}{1.9pt}
    \begin{tabular}{lrrrrrrrr}
    \toprule
    \multirow{2}{*}{} & \multicolumn{8}{c}{\textbf{SSIM ImageNet}} \\
    \cmidrule(lr){2-9}
    \textbf{Task}                     & \textsc{Spin}      & \textsc{DAPS}      & \textsc{Red-Diff}   & \textsc{PNP-DM}    & \textsc{DPS}       & \textsc{DDNM}      & \textsc{PGDM}      & \textsc{Diffpir}     \\ 
    \midrule
    Gaussian Deblur                    & \tabhi{BurntOrange}{\tabbest{0.77}{0.07}}   & \tabval{0.68}{0.13}  & \tabhi{BurntOrange!40}{\tabval{0.69}{0.10}}  & \tabval{0.63}{0.14}  & \tabval{0.57}{0.17}  & \tabval{0.60}{0.15}  & \tabval{0.07}{0.02}  & --              \\
    Motion Deblur                      & \tabhi{BurntOrange}{\tabbest{0.72}{0.05}}   & \tabhi{BurntOrange!40}{\tabval{0.66}{0.13}}  & \tabval{0.65}{0.05}  & \tabval{0.60}{0.14}  & \tabval{0.49}{0.17}  & --               & --               & --              \\
    SR ($\times$4)                     & \tabhi{BurntOrange!40}{\tabval{0.71}{0.05}}   & \tabval{0.67}{0.13}  & \tabval{0.59}{0.05}  & \tabval{0.60}{0.15}  & \tabval{0.51}{0.17}  & \tabhi{BurntOrange}{\tabbest{0.74}{0.10}}  & \tabval{0.27}{0.05}  & --              \\
    SR ($\times$16)                    & \tabhi{BurntOrange}{\tabbest{0.58}{0.14}}   & \tabval{0.43}{0.15}  & \tabval{0.45}{0.13}  & \tabval{0.40}{0.14}  & \tabval{0.34}{0.16}  & \tabhi{BurntOrange!40}{\tabval{0.50}{0.16}}  & \tabval{0.21}{0.09}  & --              \\
    Box Inpainting                     & \tabval{0.72}{0.07}   & \tabval{0.73}{0.05}  & \tabval{0.36}{0.06}  & \tabval{0.51}{0.04}  & \tabval{0.58}{0.15}  & \tabhi{BurntOrange!40}{\tabval{0.76}{0.05}}  & \tabval{0.61}{0.02}  & \tabhi{BurntOrange}{\tabbest{0.78}{0.04}} \\
    Half Inpainting                    & \tabhi{BurntOrange!40}{\tabval{0.67}{0.10}}   & \tabval{0.65}{0.07}  & \tabval{0.58}{0.05}  & \tabval{0.38}{0.03}  & \tabval{0.52}{0.14}  & \tabval{0.66}{0.08}  & \tabval{0.52}{0.04}  & \tabhi{BurntOrange}{\tabbest{0.72}{0.08}} \\
    JPEG (QF=2)                        & \tabhi{BurntOrange}{\tabbest{0.75}{0.09}}   & \tabhi{BurntOrange!40}{\tabval{0.64}{0.14}}  & \tabval{0.61}{0.11}  & \tabval{0.59}{0.14}  & \tabval{0.46}{0.16}  & --               & --               & --              \\
    Phase Retrieval                    & \tabval{0.25}{0.14}   & \tabhi{BurntOrange}{\tabbest{0.36}{0.24}}  & \tabval{0.23}{0.12}  & \tabhi{BurntOrange!40}{\tabval{0.33}{0.15}}  & \tabval{0.20}{0.11}  & --               & --               & --              \\
    HDR                                & \tabhi{BurntOrange!40}{\tabval{0.77}{0.14}}   & \tabhi{BurntOrange}{\tabbest{0.82}{0.11}}  & \tabval{0.72}{0.12}  & \tabval{0.64}{0.21}  & \tabval{0.16}{0.21}  & --               & --               & --              \\
    \midrule
    \textbf{Memory (MB)}               & \leftcell{\tabhi{BurntOrange}{\textbf{4991}}}              & \colorfactorcell{BurntOrange!40}{4993}{1.0}             & \factorcell{4995}{1.0}             & \colorfactorcell{BurntOrange!40}{4993}{1.0}             & \factorcell{8701}{1.7}             & \factorcell{5031}{1.0}             & \factorcell{8741}{1.8}             & \factorcell{5007}{1.0}            \\
    \textbf{Run time (sec)}            & \leftcell{\tabhi{BurntOrange}{\textbf{83}}}                & \factorcell{219}{2.6}              & \colorfactorcell{BurntOrange!40}{165}{2.0}              & \factorcell{636}{7.7}              & \factorcell{360}{4.3}              & \factorcell{296}{3.6}              & \factorcell{371}{4.5}              & \factorcell{169}{2.0}             \\
    \bottomrule
    \end{tabular}
\end{table*}

\begin{table*}[!htbp]
    \centering
    \caption[ImageNet PSNR comparison]{Quantitative comparison of different methods across various inverse problems. Results are reported as PSNR $\pm$ standard deviation. Higher is better. \textcolor{RoyalBlue}{Blue} factors are relative to \textsc{Spin}; larger values mean greater time or memory savings.}
    \label{tab:imagenet_psnr}
    \fontsize{7.3}{8.5}\selectfont
    \setlength{\tabcolsep}{1.8pt}
    \begin{tabular}{lrrrrrrrr}
    \toprule
    \multirow{2}{*}{} & \multicolumn{8}{c}{\textbf{PSNR ImageNet}} \\
    \cmidrule(lr){2-9}
    \textbf{Task}                     & \textsc{Spin}    & \textsc{DAPS}      & \textsc{Red-Diff}   & \textsc{PNP-DM}    & \textsc{DPS}       & \textsc{DDNM}      & \textsc{PGDM}      & \textsc{Diffpir}     \\
    \midrule
    Gaussian Deblur                    & \tabhi{BurntOrange}{\tabwidebest{27.3}{2.57}}  & \tabwideval{25.4}{3.15} & \tabhi{BurntOrange!40}{\tabwideval{25.8}{2.99}} & \tabwideval{24.5}{2.83} & \tabwideval{22.6}{3.06} & \tabwideval{23.7}{3.26} & \tabwideval{9.7}{0.64}  & --              \\
    Motion Deblur                      & \tabhi{BurntOrange}{\tabwidebest{26.9}{2.20}}  & \tabwideval{24.7}{2.97} & \tabhi{BurntOrange!40}{\tabwideval{25.5}{1.97}} & \tabwideval{23.4}{2.68} & \tabwideval{20.4}{2.65} & --              & --               & --              \\
    SR ($\times$4)                     & \tabhi{BurntOrange}{\tabwidebest{26.6}{1.95}}  & \tabwideval{25.1}{2.91} & \tabwideval{24.5}{1.88} & \tabwideval{23.8}{2.68} & \tabwideval{20.9}{3.00} & \tabhi{BurntOrange!40}{\tabwideval{26.4}{3.33}} & \tabwideval{18.2}{1.80} & --              \\
    SR ($\times$16)                    & \tabhi{BurntOrange}{\tabwidebest{20.8}{2.78}}  & \tabwideval{16.9}{1.90} & \tabwideval{19.7}{1.97} & \tabwideval{15.1}{1.37} & \tabwideval{16.4}{2.21} & \tabhi{BurntOrange!40}{\tabwideval{20.4}{2.38}} & \tabwideval{15.4}{1.89} & --              \\
    Box Inpainting                     & \tabwideval{18.3}{1.99}  & \tabwideval{18.2}{2.39} & \tabwideval{18.0}{2.82} & \tabwideval{12.6}{1.00} & \tabwideval{17.1}{2.18} & \tabhi{BurntOrange}{\tabwidebest{19.2}{2.51}} & \tabwideval{16.4}{1.85} & \tabhi{BurntOrange!40}{\tabwideval{18.70}{3.23}}              \\
    Half Inpainting                    & \tabhi{BurntOrange!40}{\tabwideval{15.9}{2.75}}  & \tabwideval{15.2}{3.03} & \tabwideval{14.3}{3.00} & \tabwideval{10.8}{1.51} & \tabwideval{14.0}{2.83} & \tabwideval{15.5}{3.14} & \tabwideval{13.8}{2.09} & \tabhi{BurntOrange}{\tabwidebest{16.56}{4.09}}              \\
    JPEG (QF=2)                        & \tabhi{BurntOrange}{\tabwidebest{25.6}{2.26}}  & \tabhi{BurntOrange!40}{\tabwideval{23.7}{2.63}} & \tabwideval{22.9}{1.95} & \tabwideval{21.6}{2.00} & \tabwideval{18.6}{2.60} & --              & --               & --              \\
    Phase Retrieval                    & \tabwideval{12.6}{4.43}  & \tabhi{BurntOrange}{\tabwidebest{14.6}{6.76}} & \tabhi{BurntOrange!40}{\tabwideval{14.2}{4.65}} & \tabwideval{13.0}{3.24} & \tabwideval{11.9}{2.44} & --              & --               & --              \\
    HDR                                & \tabhi{BurntOrange!40}{\tabwideval{23.2}{4.04}}  & \tabhi{BurntOrange}{\tabwidebest{25.6}{3.65}} & \tabwideval{22.5}{3.31} & \tabwideval{22.1}{3.96} & \tabwideval{7.98}{3.95}  & --              & --               & --              \\
    \midrule
    \textbf{Memory (MB)}               & \leftcell{\tabhi{BurntOrange}{\textbf{4991}}}             & \colorfactorcell{BurntOrange!40}{4993}{1.0}             & \factorcell{4995}{1.0}             & \colorfactorcell{BurntOrange!40}{4993}{1.0}             & \factorcell{8701}{1.7}             & \factorcell{5031}{1.0}             & \factorcell{8741}{1.8}             & \factorcell{5007}{1.0}               \\
    \textbf{Run time (sec)}            & \leftcell{\tabhi{BurntOrange}{\textbf{83}}}               & \factorcell{219}{2.6}              & \colorfactorcell{BurntOrange!40}{165}{2.0}              & \factorcell{636}{7.7}              & \factorcell{360}{4.3}              & \factorcell{296}{3.6}              & \factorcell{371}{4.5}              & \factorcell{169}{2.0}              \\
    \bottomrule
    \end{tabular}

\end{table*}

\clearpage
\section{Hyperparameters}
\label{sec:hyperparameters}

\begin{table*}[!htbp]
    \centering
    \caption{Pixel-space hyperparameters for the \texttt{ImageNet} and \texttt{FFHQ} experiments, grouped by task family. Unless stated otherwise, we use $\eta_{\text{main}} = 10^{-3}$, $\lambda = 0$, $\eta_{\text{interleave}} = 1.0$, $N_{\texttt{OPT}}^{\texttt{int}} = 50$, $\eta_{\text{init}} = 10^{-4}$, $M=30$, and a Gaussian guidance schedule with $\mu_G=0.4$ and $\sigma_G=10.0$. Table ranges summarize task- and dataset-specific choices within each group.}
    \label{tab:hyperparams_imagenet}
    \label{tab:hyperparams_ffhq}
    \begin{tabular}{lccc}
    \toprule
    \textbf{Task group} & $t_*$ & $N$ & $G_{\texttt{OPT}}$ \\
    \midrule
    Deblurring (Gaussian, Motion) & $[0.44, 0.50]$ & 10 & $[25, 50]$ \\
    Super-resolution (SR4, SR16) & $[0.50, 0.60]$ & $[1, 30]$ & $[75, 200]$ \\
    Inpainting (Box, Half) & $[0.70, 0.80]$ & $[5, 30]$ & 50 \\
    JPEG / Phase Retrieval / HDR & $[0.44, 0.70]$ & $[5, 20]$ & $[25, 200]$ \\
    \bottomrule
    \end{tabular}
    \end{table*}

\begin{table*}[!htbp]
\centering
\caption{Latent-space hyperparameters for the \texttt{FFHQ} experiments, grouped by task family. Unless stated otherwise, we use $\lambda = 0$, $\eta_{\text{init}} = 10^{-4}$, $\eta_{\text{main}} = 10^{-4}$, $\eta_{\text{interleave}} = 1.0$, $N=30$, $M=10$, and a polynomial guidance schedule $\textsc{Poly}(p)$. Table ranges summarize task-specific choices within each group.}
\label{tab:hyperparams_ffhq_ldm}
\begin{tabular}{lcccc}
\toprule
\textbf{Task group} & $t_*$ & $N_{\texttt{OPT}}^{\texttt{int}}$ & $G_{\texttt{OPT}}$ & $p$ \\
\midrule
Deblurring (Gaussian, Motion) & $[0.50, 0.70]$ & $[50, 300]$ & 3000 & $[2.0, 2.5]$ \\
Super-resolution (SR4, SR16) & $[0.30, 0.40]$ & 300 & $[1000, 2000]$ & $[2.5, 4.0]$ \\
Inpainting (Box, Half) & 0.44 & $[800, 1000]$ & 800 & $[2.5, 3.0]$ \\
JPEG / Phase Retrieval / HDR & $[0.40, 0.70]$ & $[50, 500]$ & $[500, 3000]$ & 3.0 \\
\bottomrule
\end{tabular}
\end{table*}

\clearpage
\section{Reconstruction Samples on \texttt{FFHQ}}
\label{sec:images_ffhq}

\begin{figure}[!htbp]
    \centering
    \includegraphics[width=\linewidth]{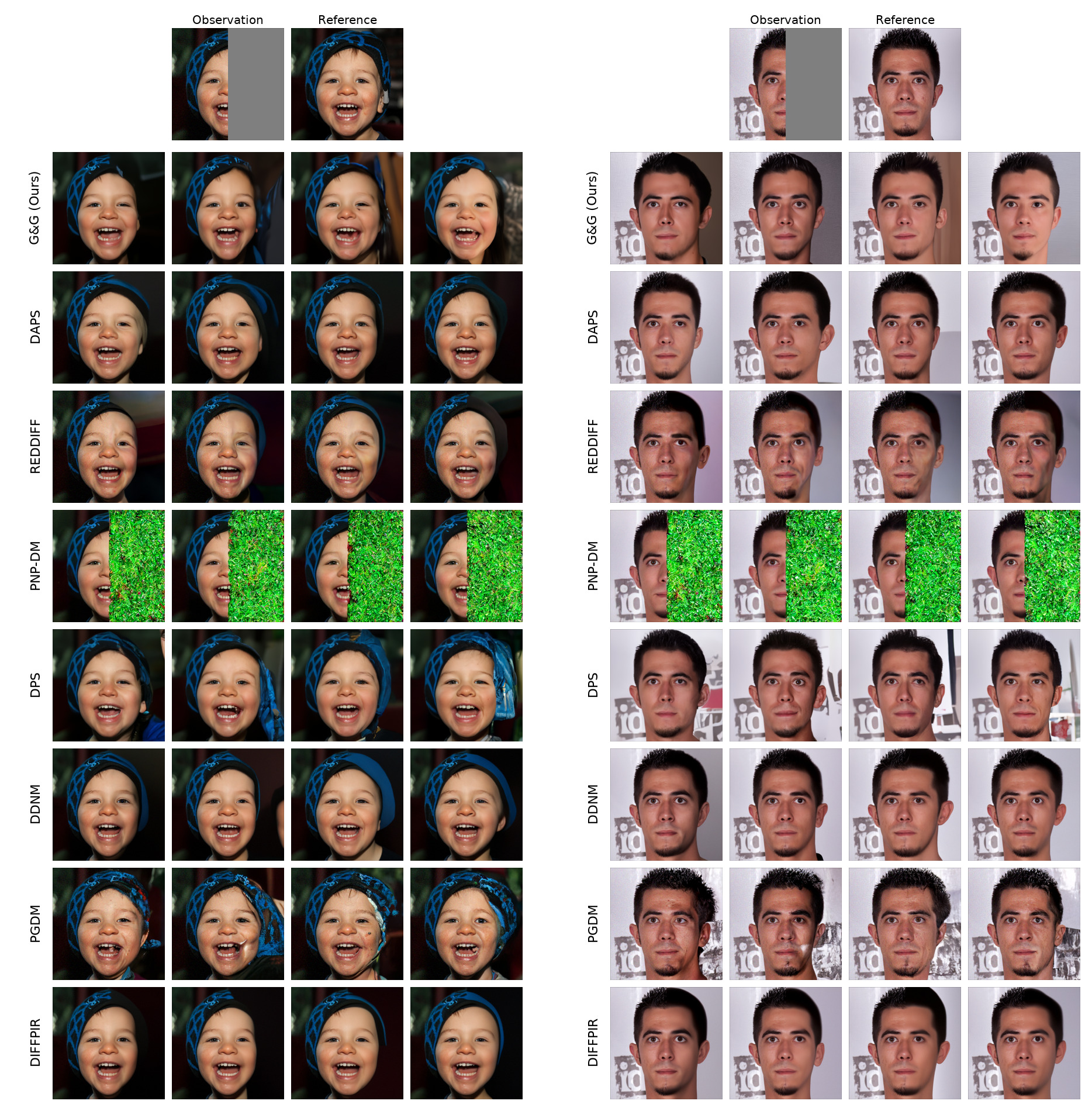}
    \vspace{-0.3cm}
    \caption{Reconstructions for half mask inpainting on \texttt{FFHQ} dataset.}
    \label{fig:all_samplers_ffhq_half_inp}
\end{figure}
\clearpage

\begin{figure}[!htbp]
    \centering
    \includegraphics[width=\linewidth]{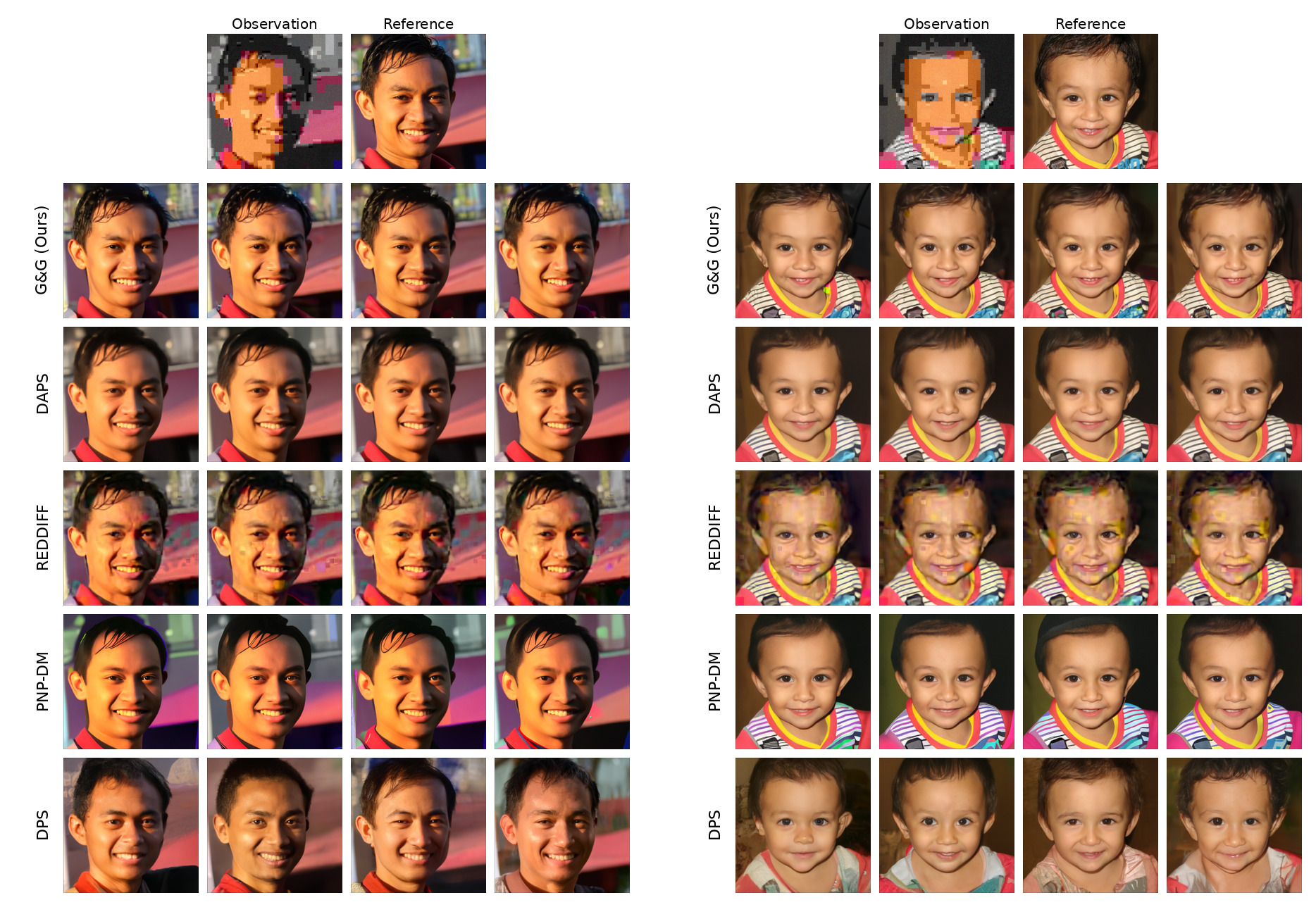}
    \vspace{-0.3cm}
    \caption{JPEG dequantization with QF = 2 on \texttt{FFHQ} dataset.}
    \label{fig:all_samplers_ffhq_jpeg2}
\end{figure}
\clearpage

\begin{figure}[!htbp]
    \centering
    \includegraphics[width=\linewidth]{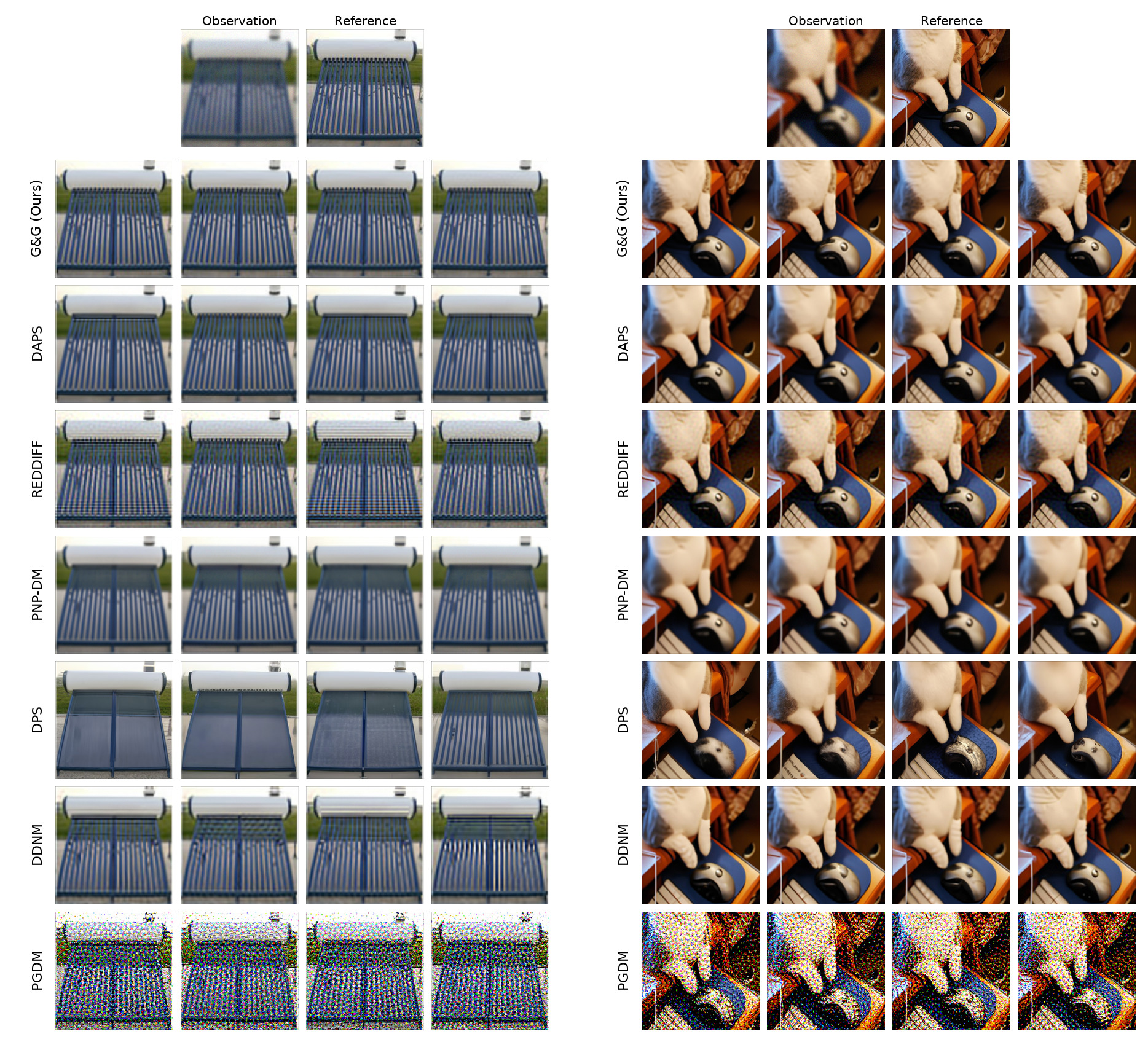}
    \vspace{-0.3cm}
    \caption{Reconstructions for Gaussian deblurring on \texttt{ImageNet} dataset.}
    \label{fig:all_samplers_imagenet_blur}
\end{figure}
\clearpage

\begin{figure}[!htbp]
    \centering
    \includegraphics[width=0.81\linewidth]{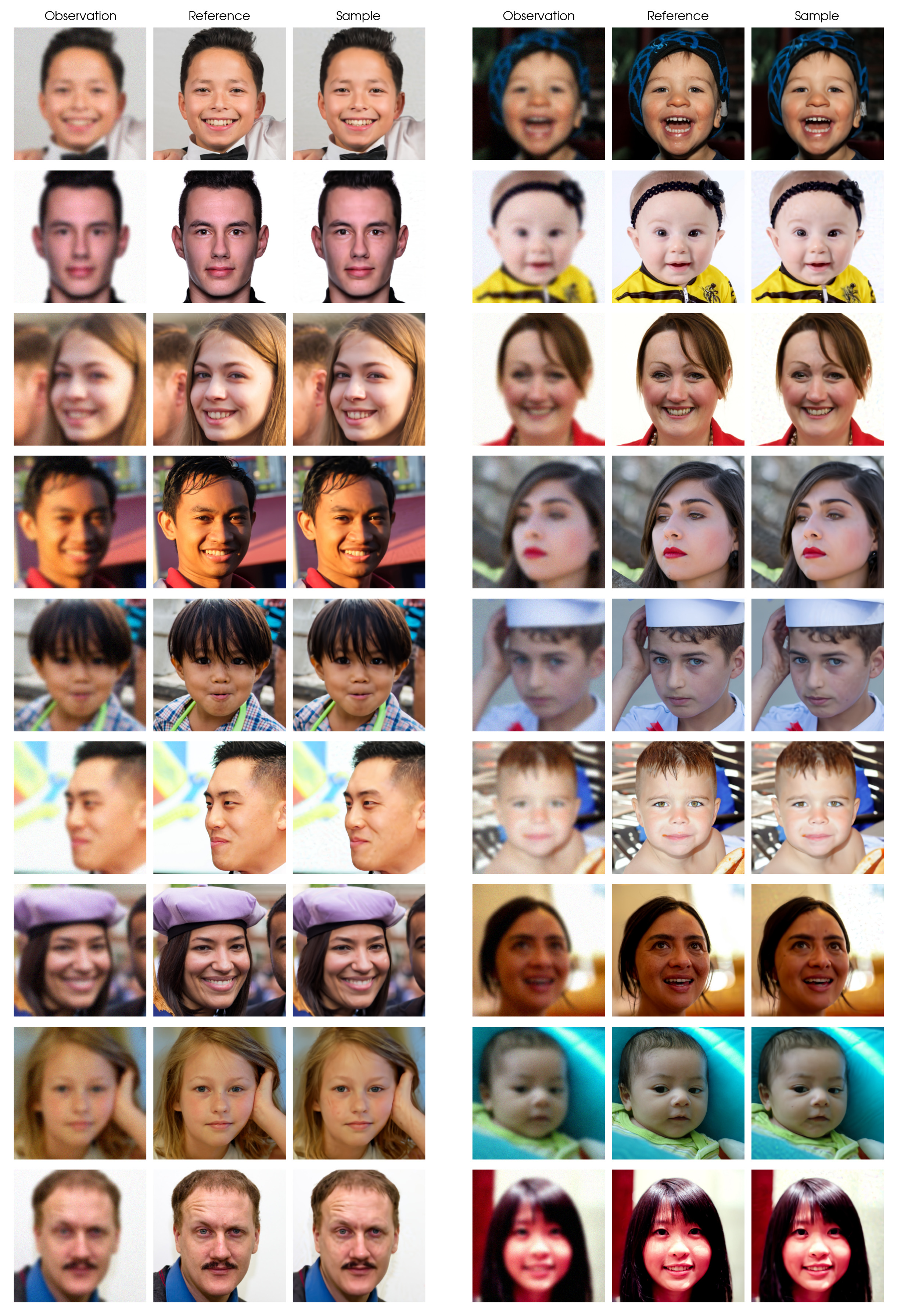}
    \vspace{-0.3cm}
    \caption{Gaussian Deblurring on \texttt{FFHQ} dataset.}
    \label{fig:images_ffhq_blur}
\end{figure}
\clearpage

\begin{figure}[!htbp]
    \centering
    \includegraphics[width=0.81\linewidth]{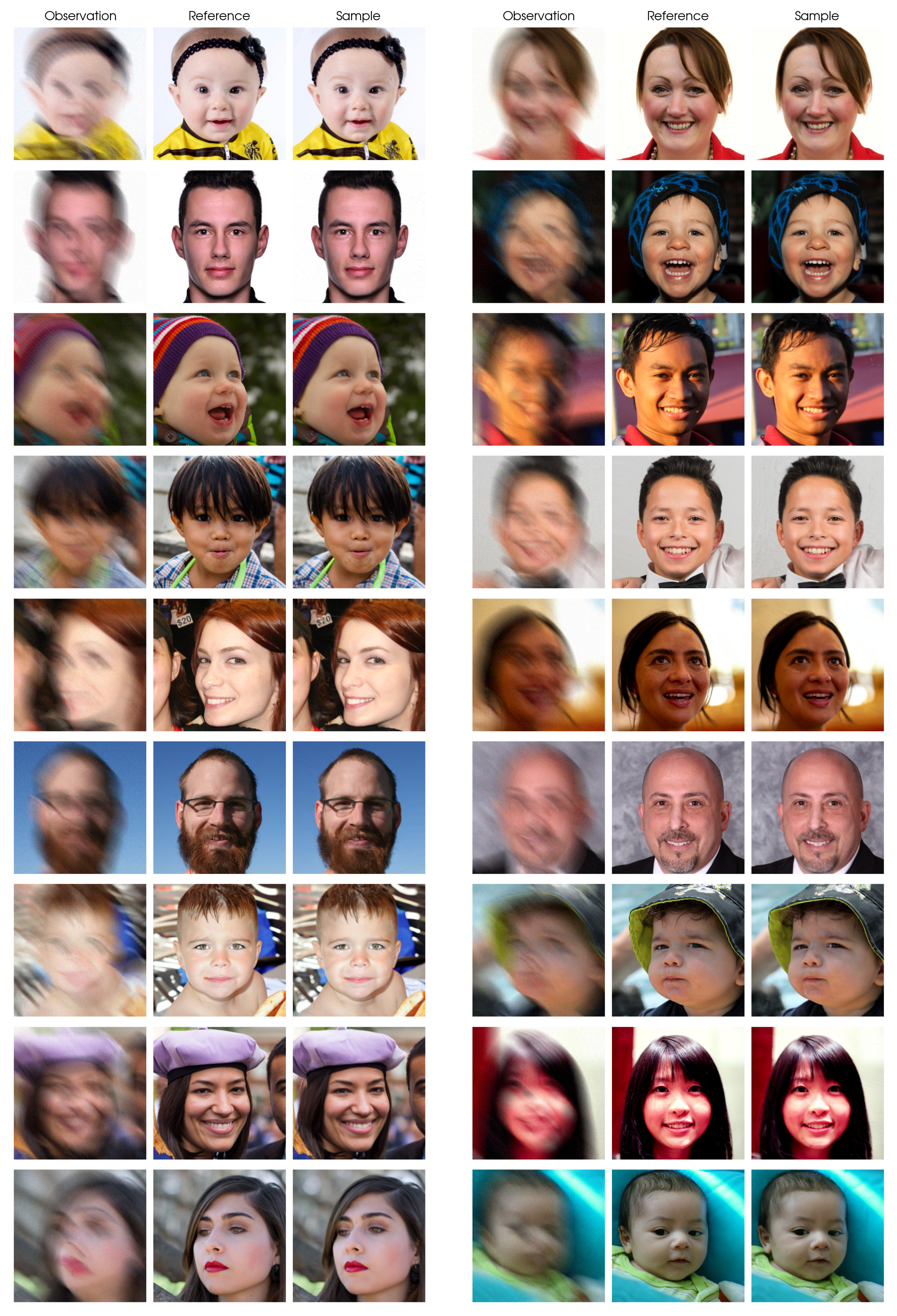}
    \vspace{-0.3cm}
    \caption{Motion Deblurring on \texttt{FFHQ} dataset.}
    \label{fig:images_ffhq_motion_blur}
\end{figure}
\clearpage

\begin{figure}[!htbp]
    \centering
    \includegraphics[width=0.81\linewidth]{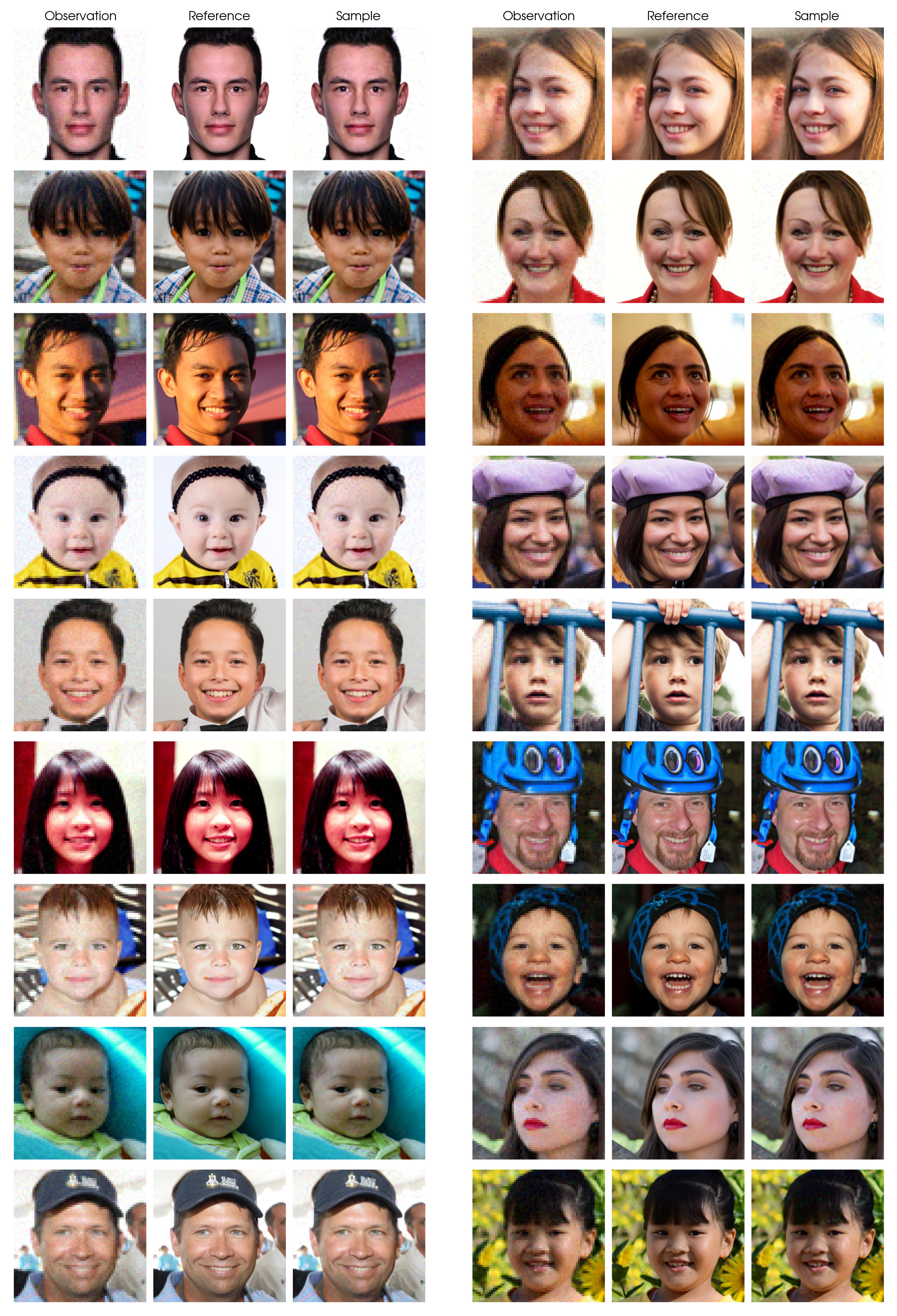}
    \vspace{-0.3cm}
    \caption{Super-resolution ($\times$4) on \texttt{FFHQ} dataset.}
    \label{fig:images_ffhq_sr4}
\end{figure}
\clearpage

\begin{figure}[!htbp]
    \centering
    \includegraphics[width=0.81\linewidth]{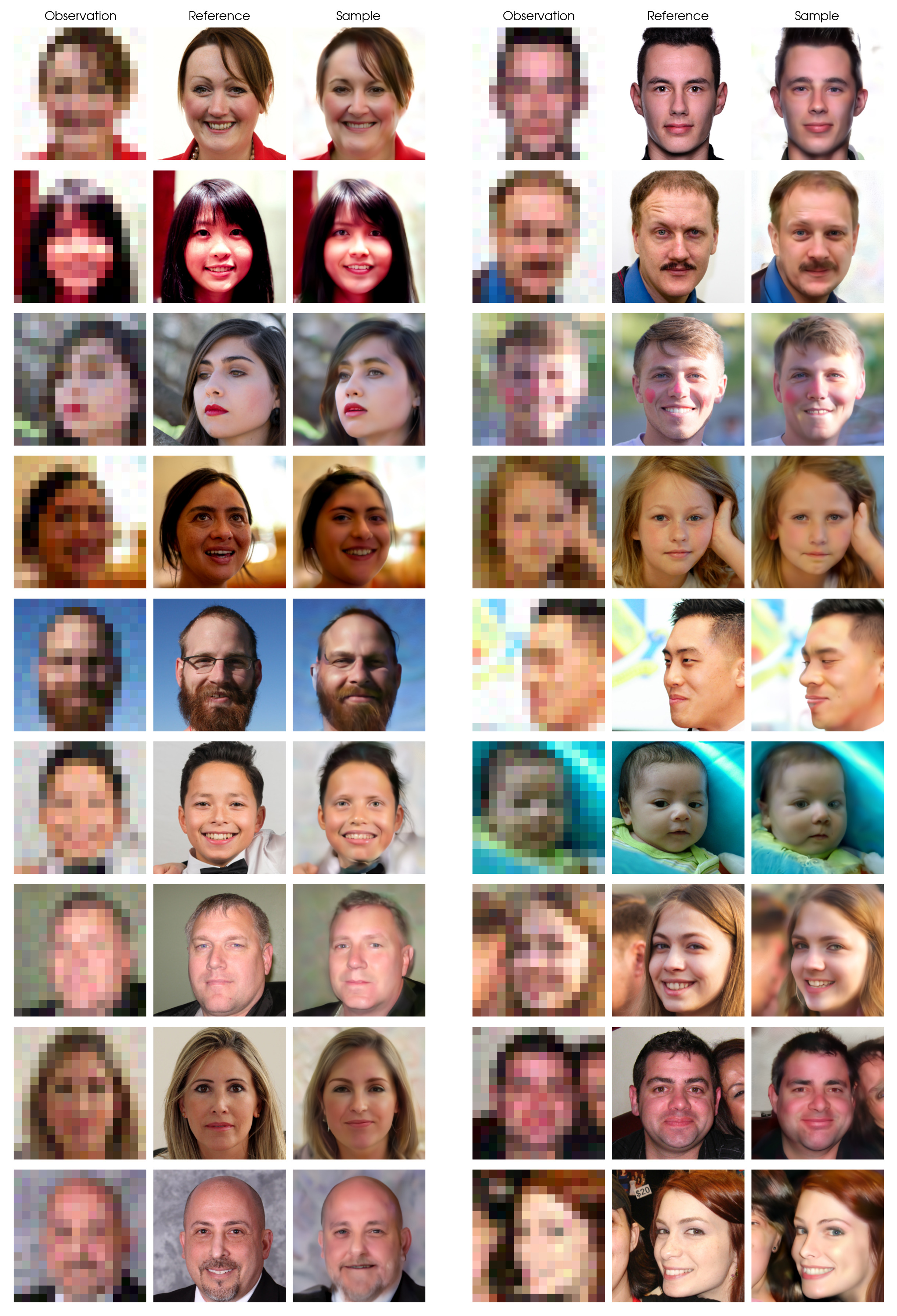}
    \vspace{-0.3cm}
    \caption{Super-resolution ($\times$16) on \texttt{FFHQ} dataset.}
    \label{fig:images_ffhq_sr16}
\end{figure}
\clearpage

\begin{figure}[!htbp]
    \centering
    \includegraphics[width=0.81\linewidth]{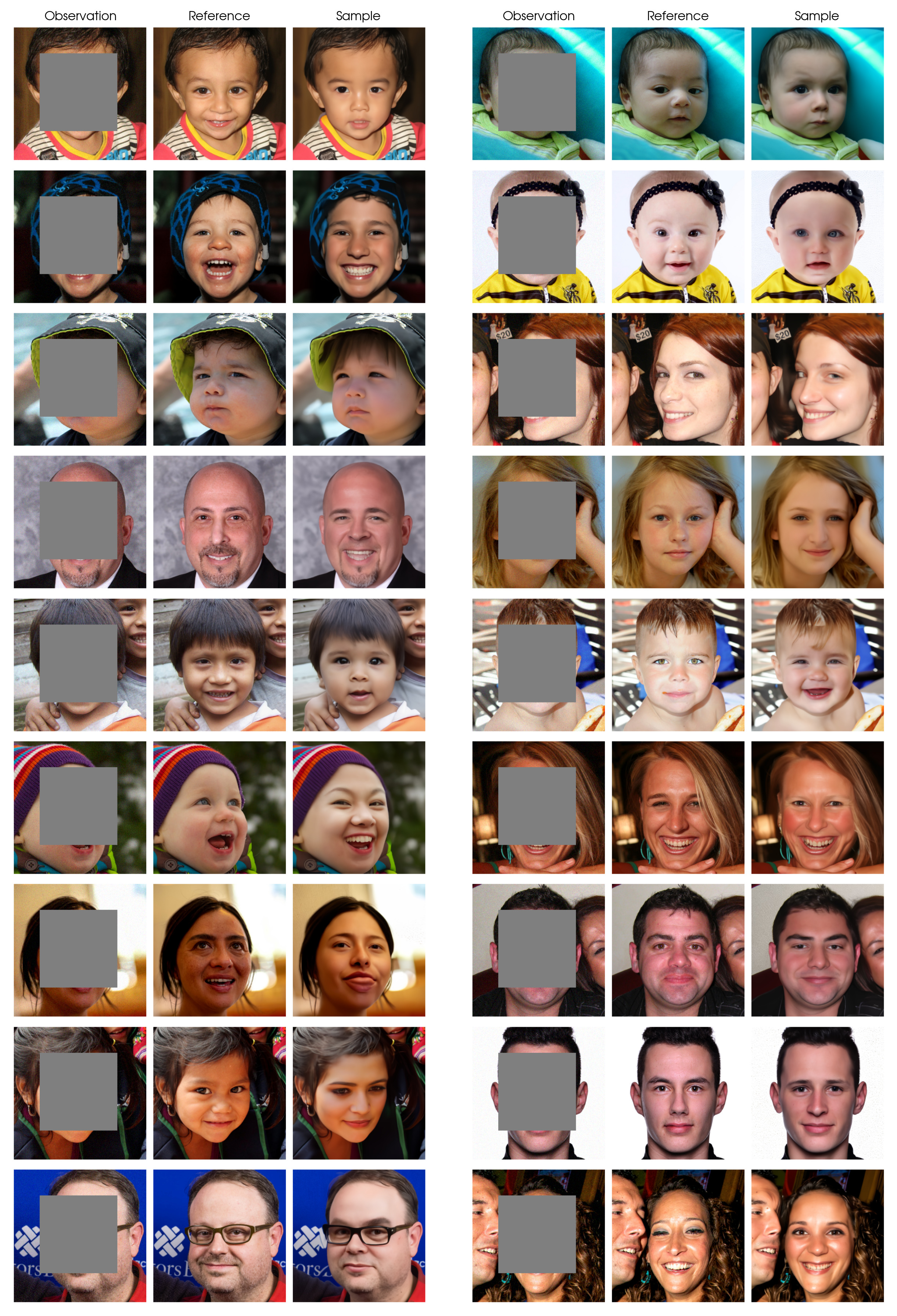}
    \vspace{-0.3cm}
    \caption{Box Inpainting on \texttt{FFHQ} dataset.}
    \label{fig:images_ffhq_box_inpainting}
\end{figure}
\clearpage

\begin{figure}[!htbp]
    \centering
    \includegraphics[width=0.81\linewidth]{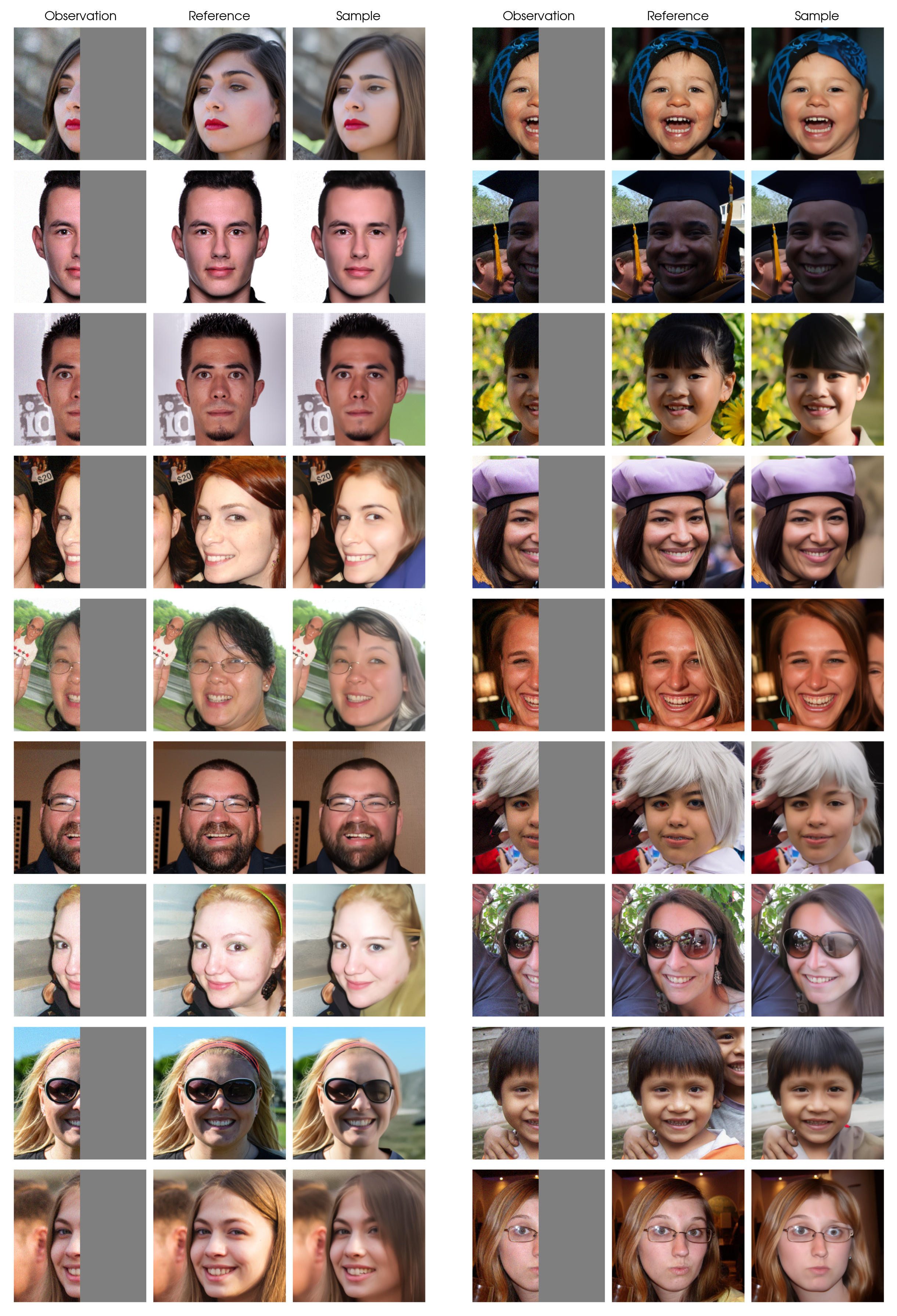}
    \vspace{-0.3cm}
    \caption{Half Inpainting on \texttt{FFHQ} dataset.}
    \label{fig:images_ffhq_half_inpainting}
\end{figure}
\clearpage

\begin{figure}[!htbp]
    \centering
    \includegraphics[width=0.81\linewidth]{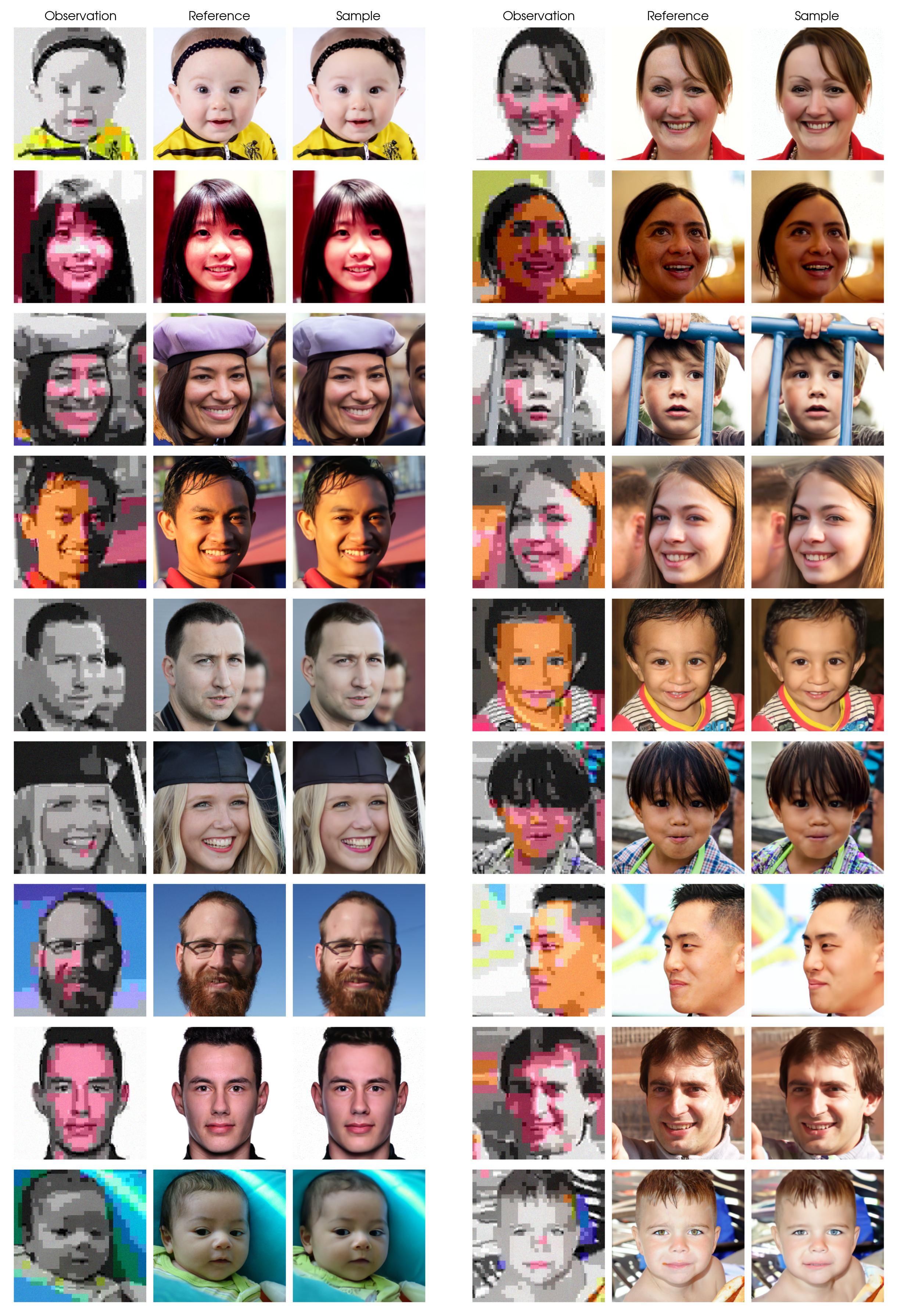}
    \vspace{-0.3cm}
    \caption{JPEG Compression (QF=2) on \texttt{FFHQ} dataset.}
    \label{fig:images_ffhq_jpeg2}
\end{figure}
\clearpage

\begin{figure}[!htbp]
    \centering
    \includegraphics[width=0.81\linewidth]{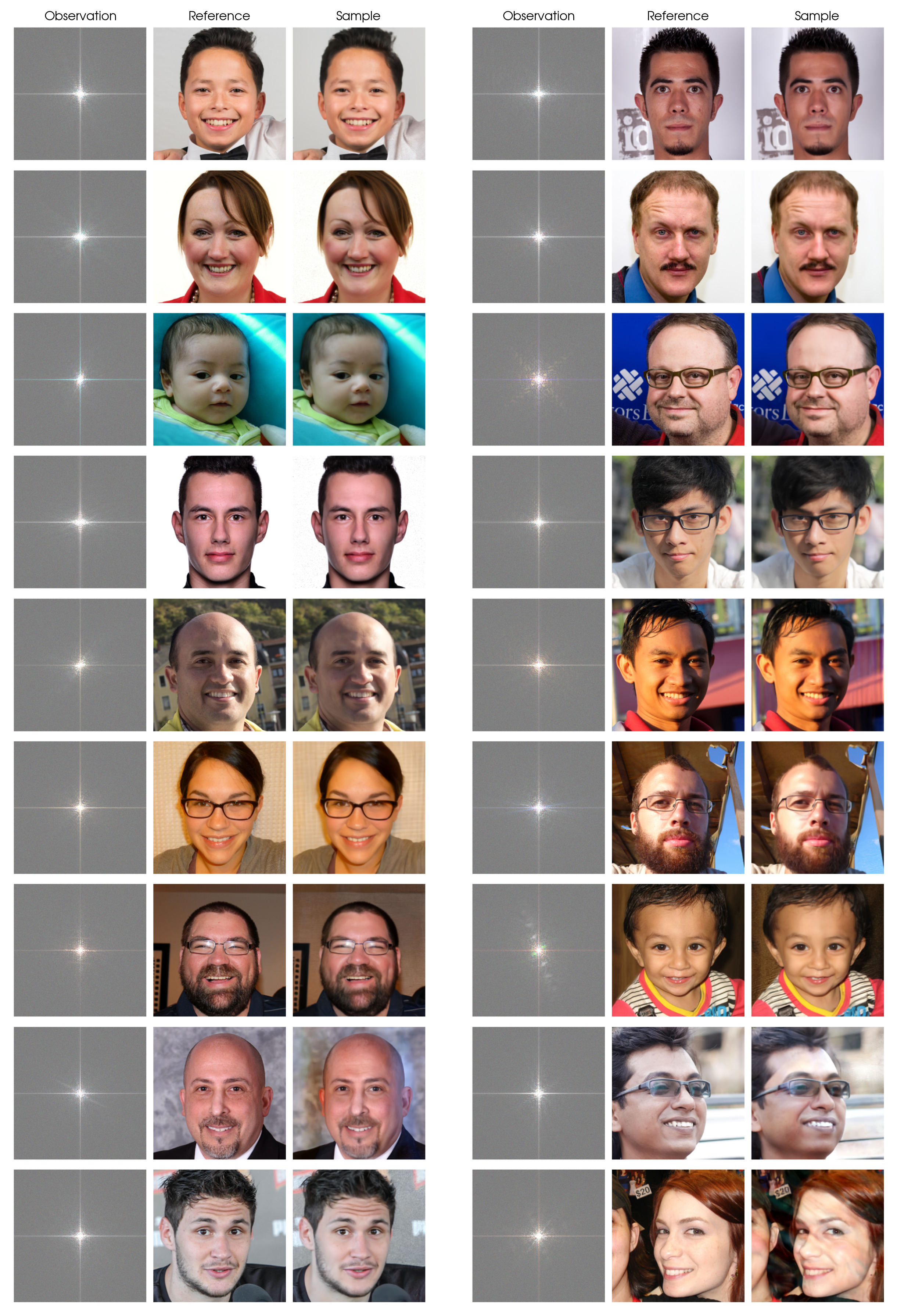}
    \vspace{-0.3cm}
    \caption{Phase Retrieval on \texttt{FFHQ} dataset.}
    \label{fig:images_ffhq_phase_retrieval}
\end{figure}

\clearpage
\section{Reconstruction Samples on \texttt{FFHQ} with \texttt{LDM}}
\label{sec:images_ffhq_ldm}

\begin{figure}[!htbp]
    \centering
    \includegraphics[width=0.81\linewidth]{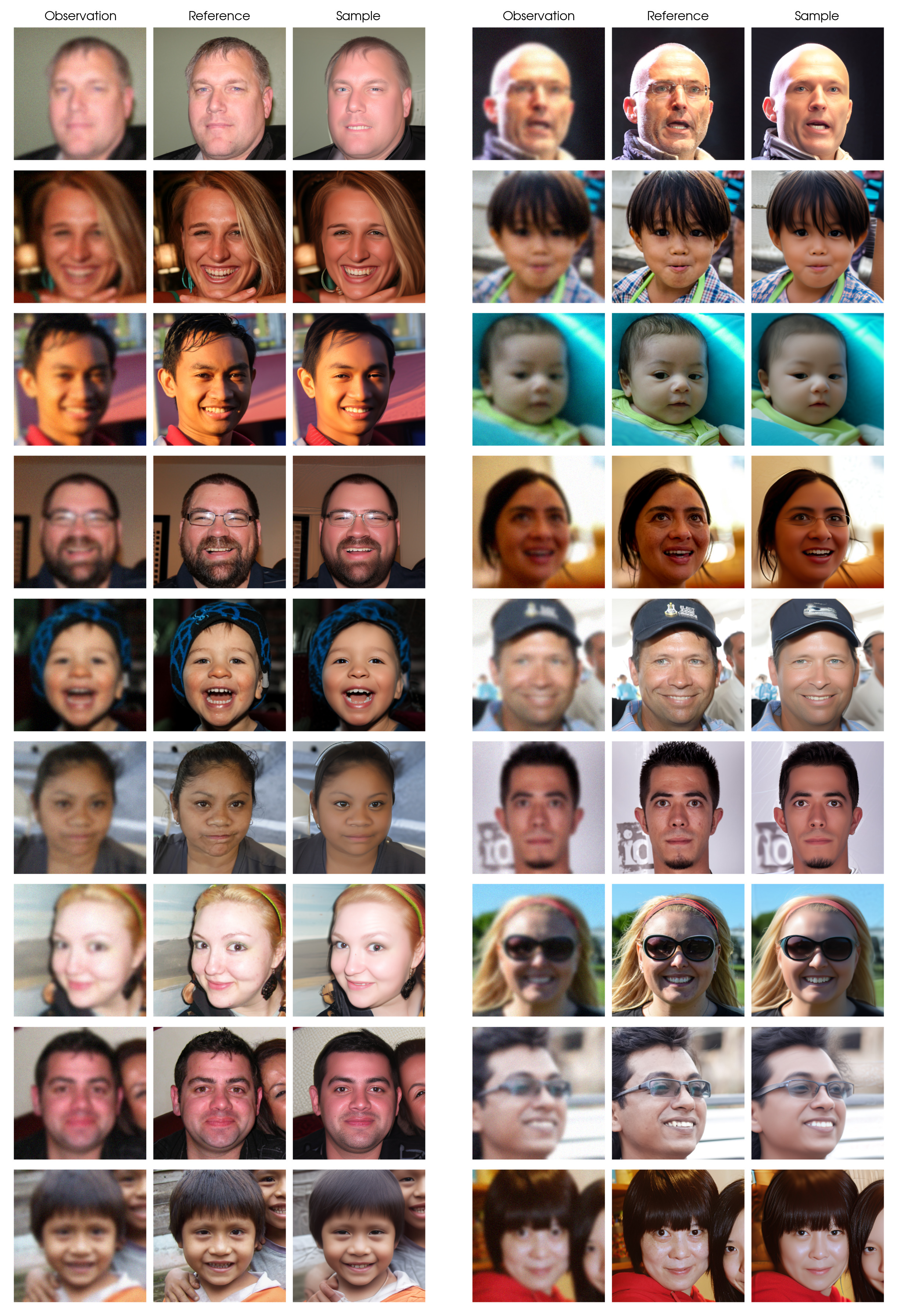}
    \vspace{-0.3cm}
    \caption{Gaussian Deblurring on \texttt{FFHQ} dataset with \texttt{LDM} prior.}
    \label{fig:images_ffhq_ldm_blur}
\end{figure}
\clearpage

\begin{figure}[!htbp]
    \centering
    \includegraphics[width=0.81\linewidth]{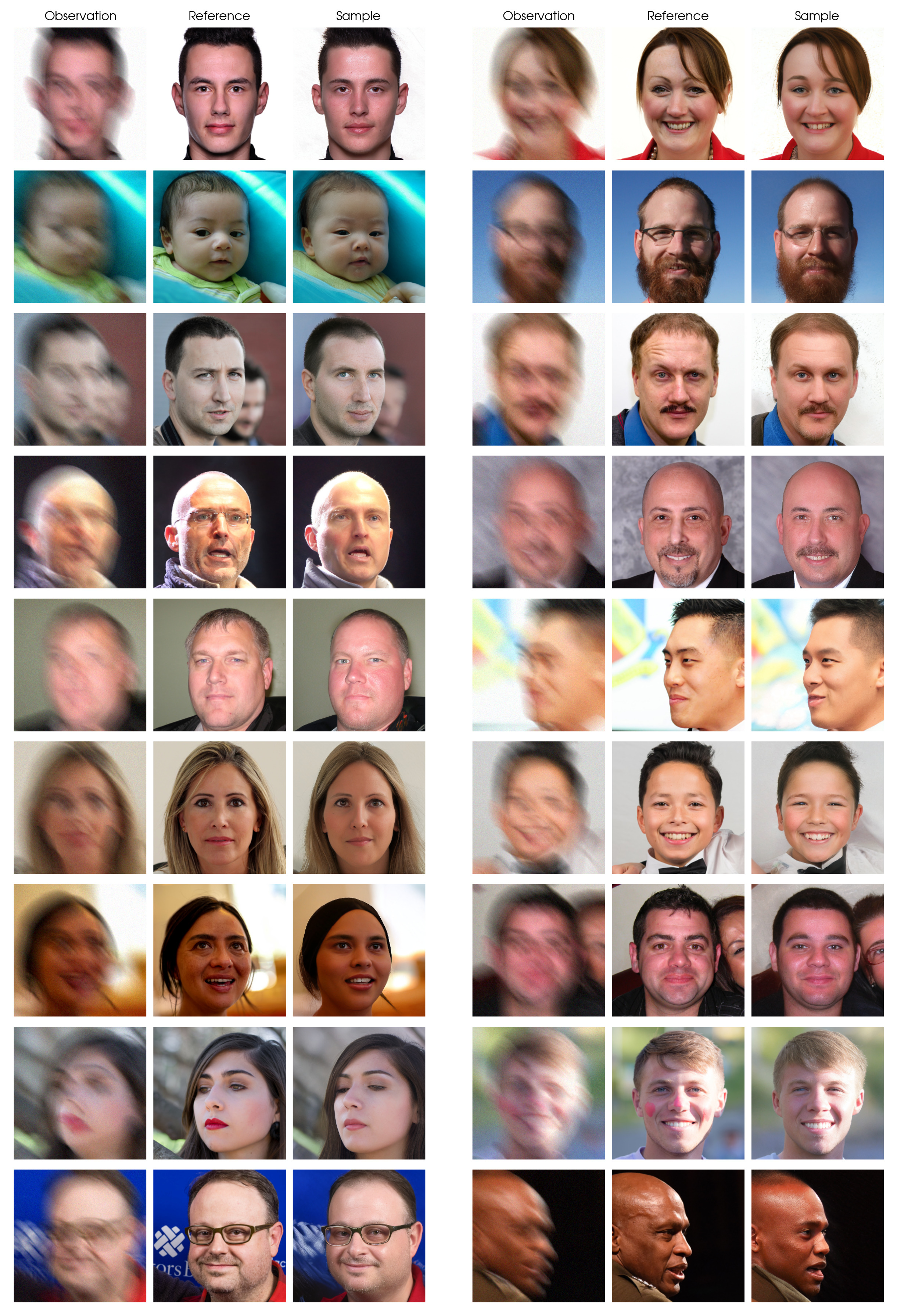}
    \vspace{-0.3cm}
    \caption{Motion Deblurring on \texttt{FFHQ} dataset with \texttt{LDM} prior.}
    \label{fig:images_ffhq_moblur}
\end{figure}
\clearpage

\begin{figure}[!htbp]
    \centering
    \includegraphics[width=0.81\linewidth]{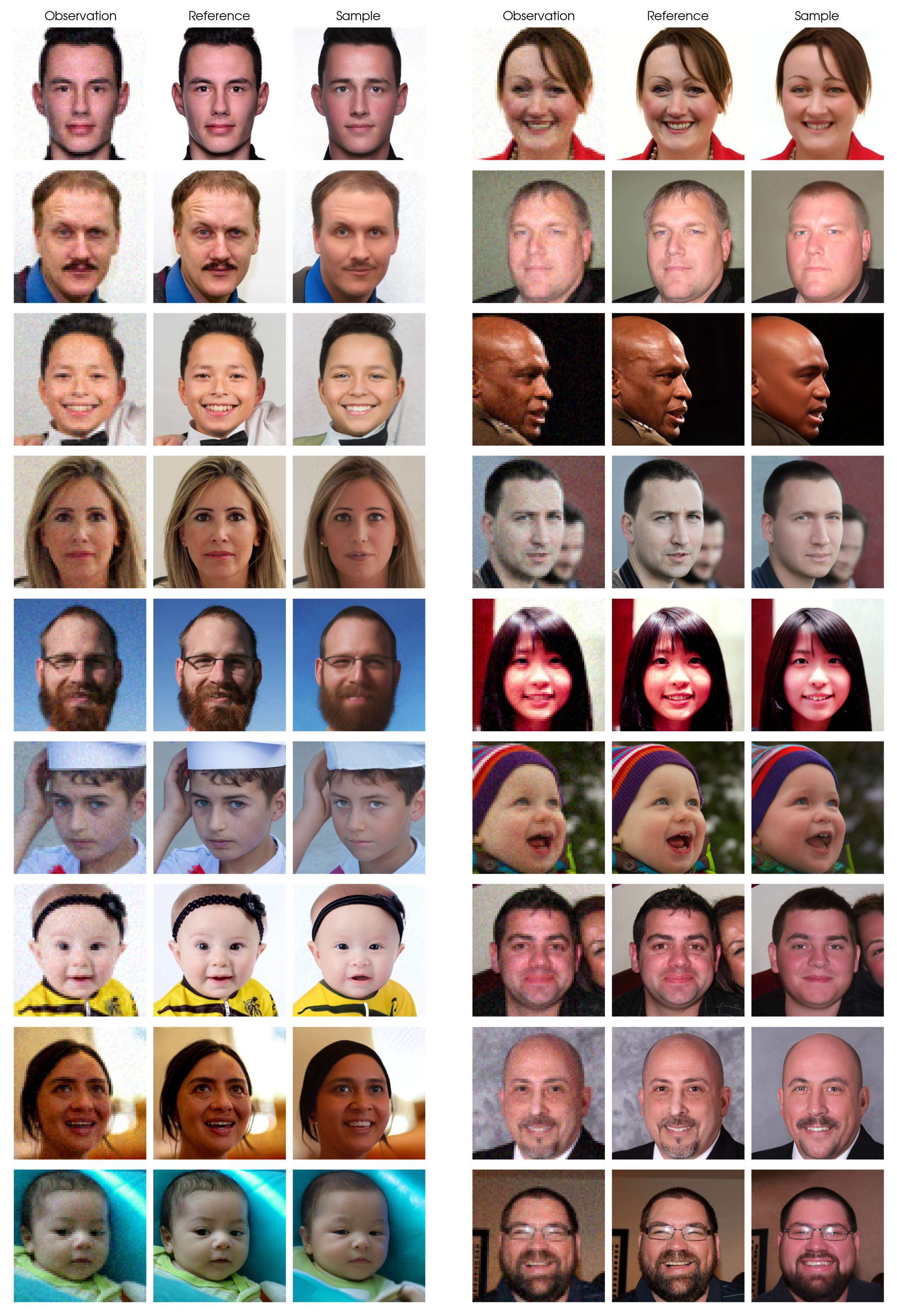}
    \vspace{-0.3cm}
    \caption{Super-Resolution ($\times$4) on \texttt{FFHQ} dataset with \texttt{LDM} prior.}
    \label{fig:images_ffhq_ldm_sr4}
\end{figure}
\clearpage

\begin{figure}[!htbp]
    \centering
    \includegraphics[width=0.81\linewidth]{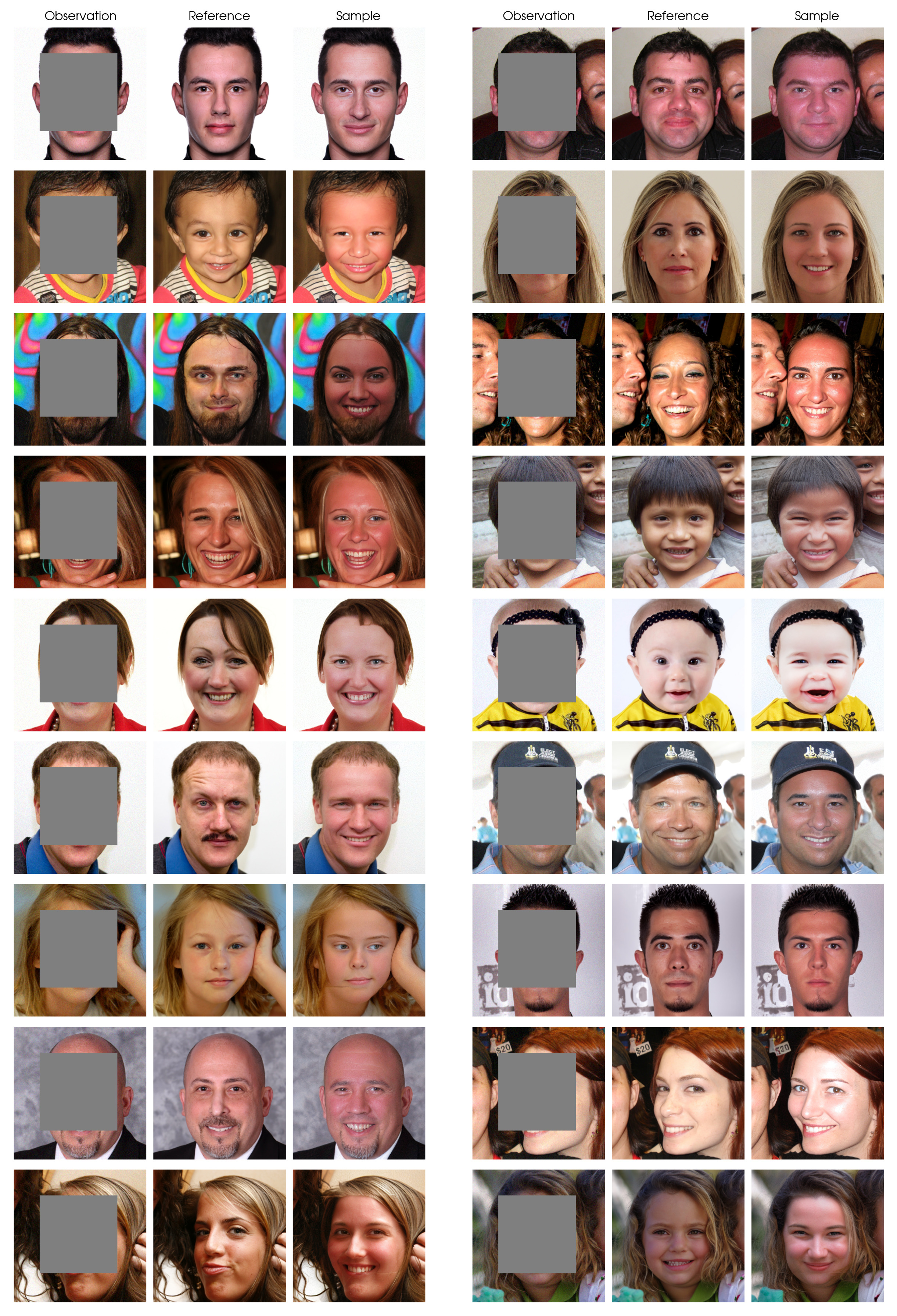}
    \vspace{-0.3cm}
    \caption{Box Inpainting on \texttt{FFHQ} dataset with \texttt{LDM} prior.}
    \label{fig:images_ffhq_inpainting}
\end{figure}
\clearpage

\begin{figure}[!htbp]
    \centering
    \includegraphics[width=0.81\linewidth]{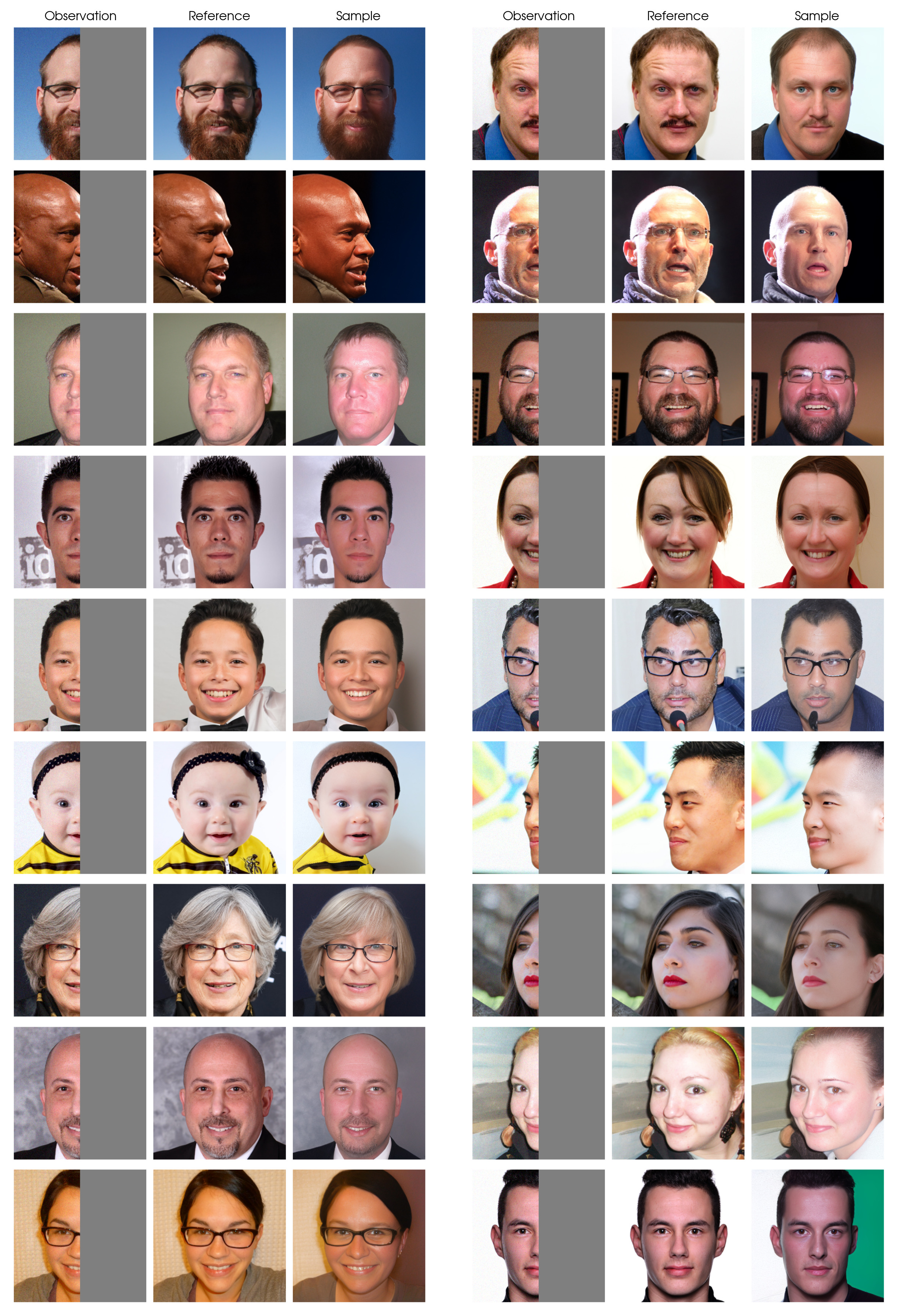}
    \vspace{-0.3cm}
    \caption{Half Inpainting on \texttt{FFHQ} dataset with \texttt{LDM} prior.}
    \label{fig:images_ffhq_outpainting}
\end{figure}
\clearpage

\clearpage
\section{Reconstruction Samples on \texttt{ImageNet}}
\label{sec:images_imagenet}

\begin{figure}[!htbp]
    \centering
    \includegraphics[width=0.935\linewidth]{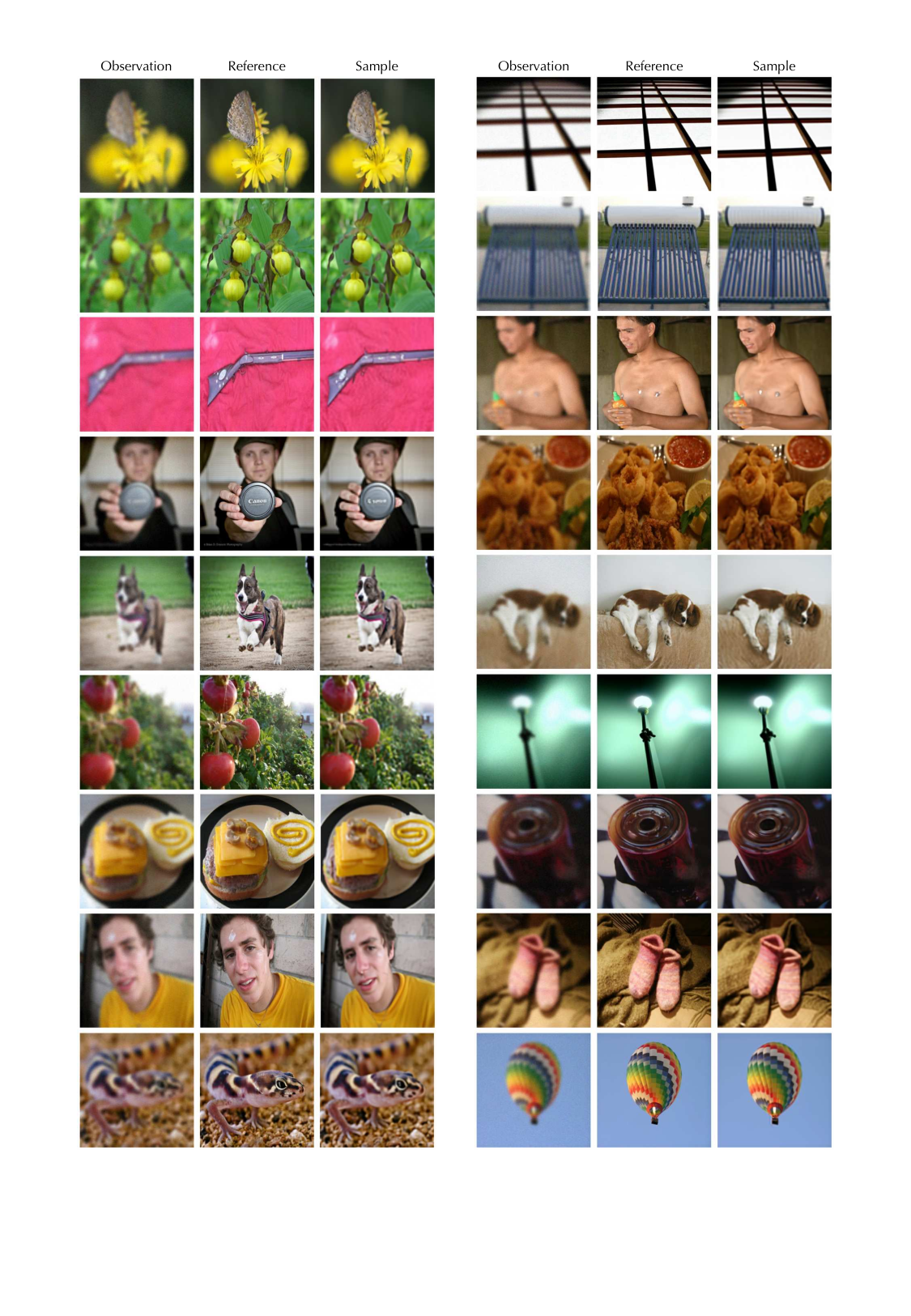}
    \vspace{-2.5cm}
    \caption{Deblurring on \texttt{ImageNet} dataset.}
    \label{fig:images_imagenet_blur}
\end{figure}
\clearpage

\begin{figure}[!htbp]
    \centering
    \includegraphics[width=0.935\linewidth]{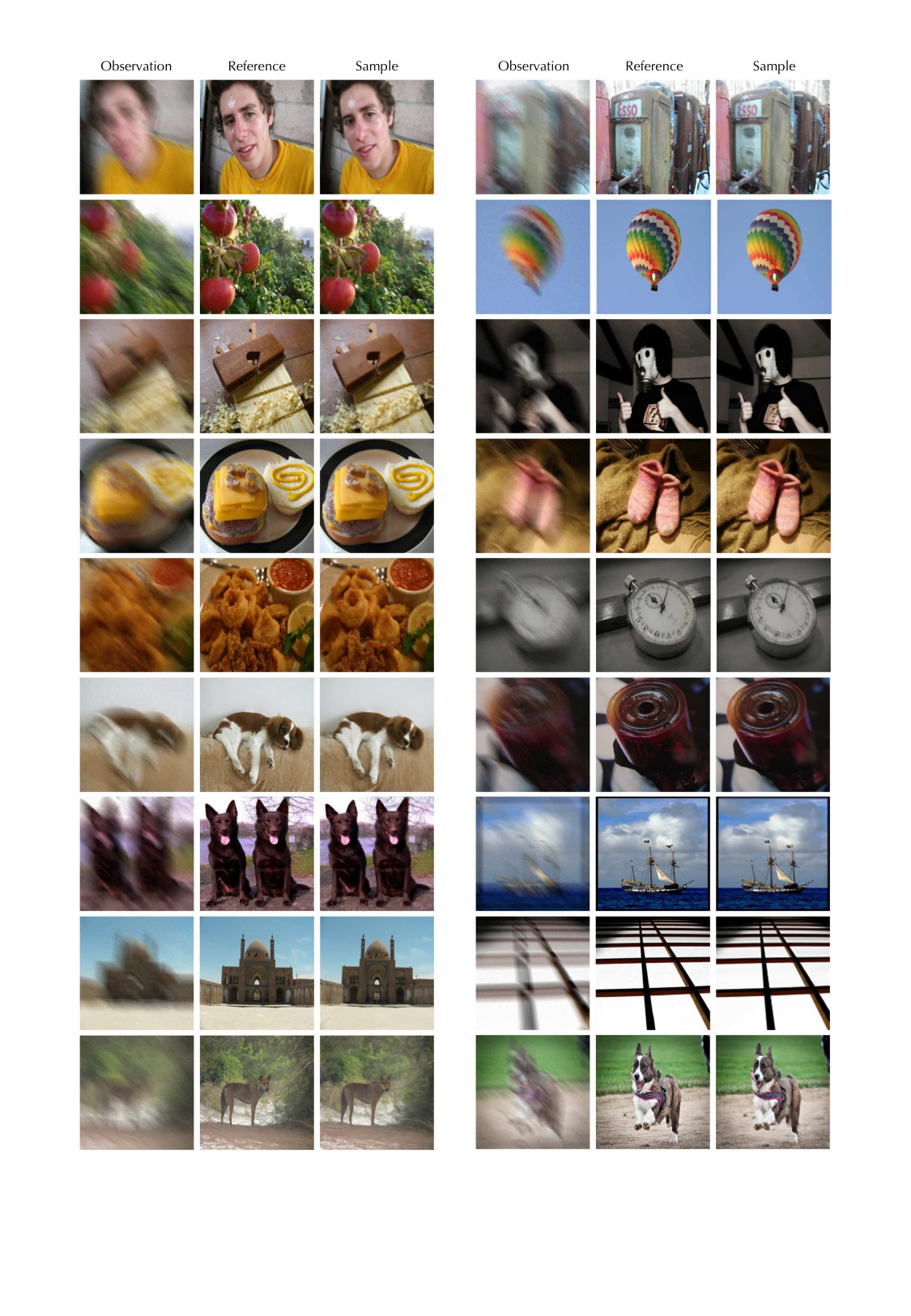}
    \vspace{-2.5cm}
    \caption{Deblurring on \texttt{ImageNet} dataset.}
    \label{fig:images_imagenet_motion_blur}
\end{figure}
\clearpage

\begin{figure}[!htbp]
    \centering
    \includegraphics[width=0.935\linewidth]{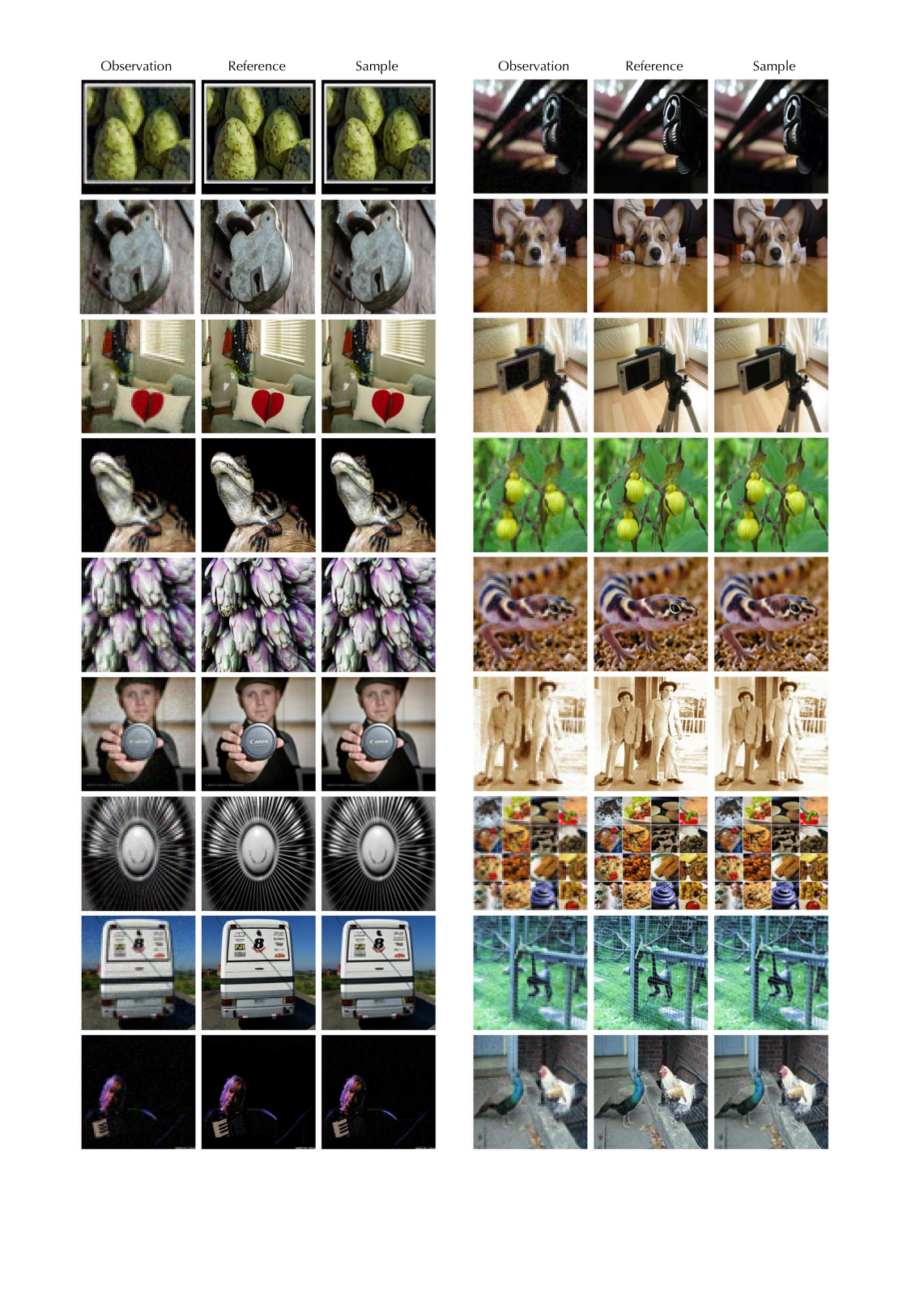}
    \vspace{-2.5cm}
    \caption{Super-resolution ($\times$4) on \texttt{ImageNet} dataset.}
    \label{fig:images_imagenet_sr4}
\end{figure}
\clearpage

\begin{figure}[!htbp]
    \centering
    \includegraphics[width=0.935\linewidth]{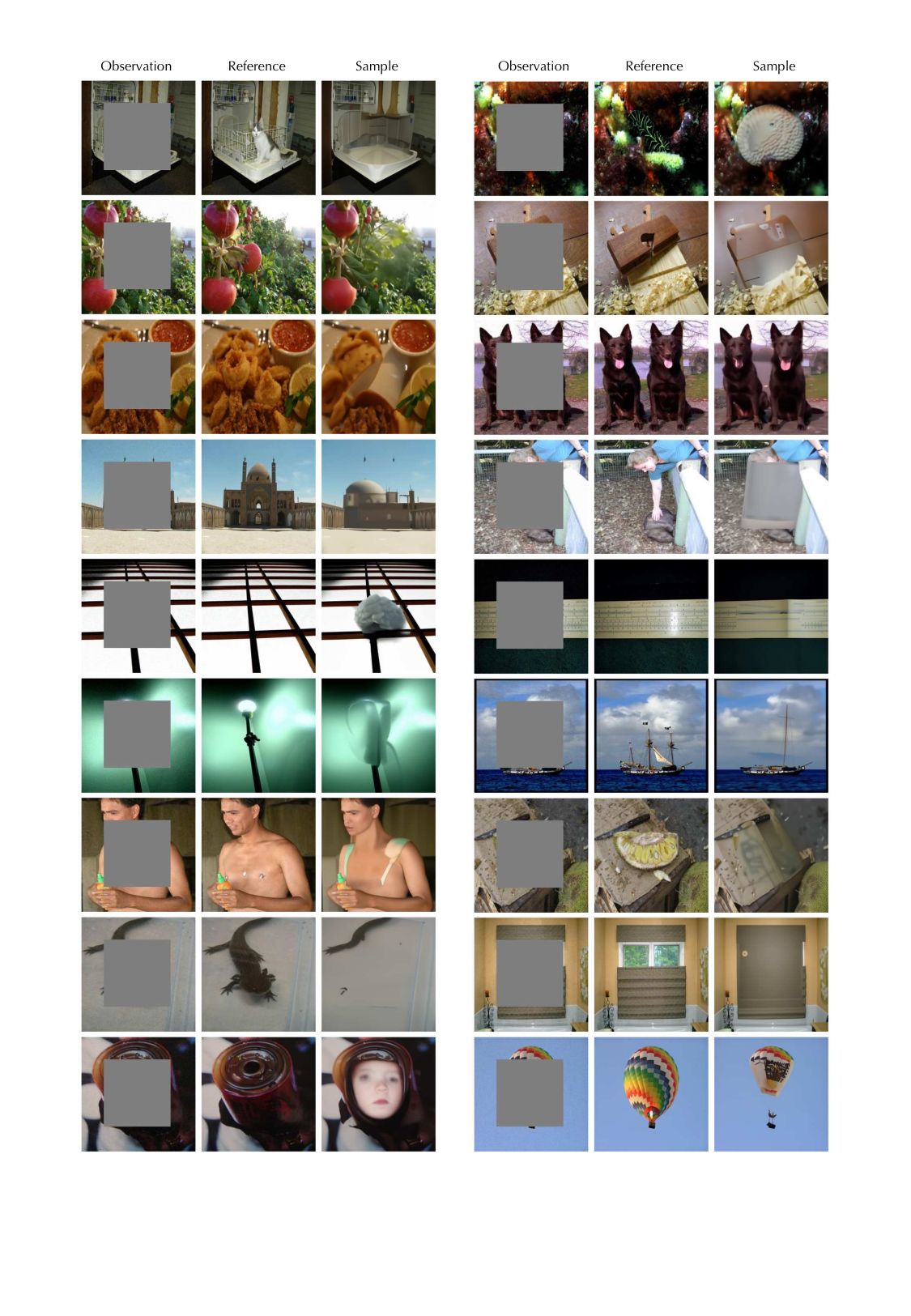}
    \vspace{-2.5cm}
    \caption{Center inpainting on \texttt{ImageNet} dataset.}
    \label{fig:images_imagenet_center_inpainting}
\end{figure}
\clearpage

\begin{figure}[!htbp]
    \centering
    \includegraphics[width=0.935\linewidth]{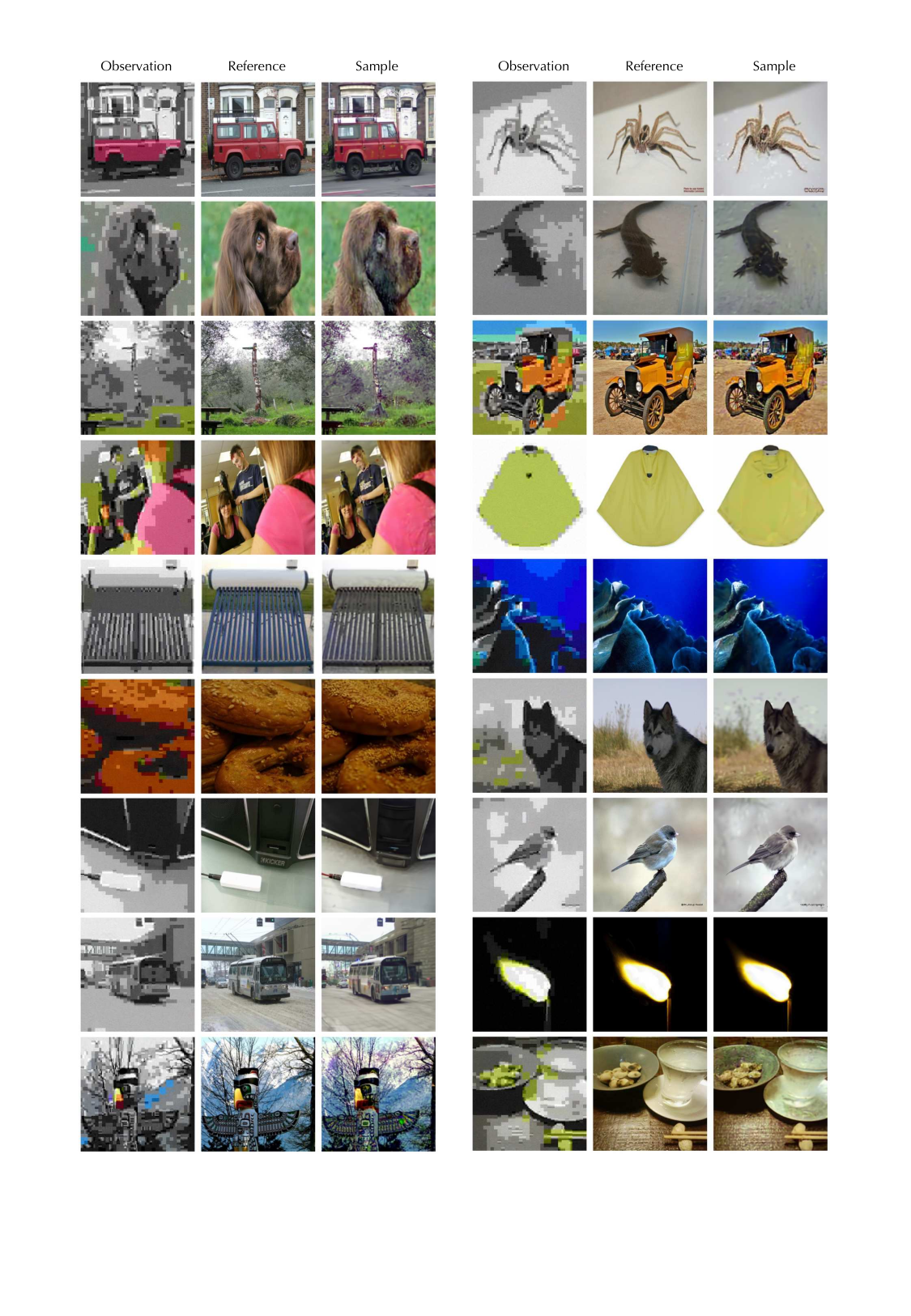}
    \vspace{-2.5cm}
    \caption{JPEG compression (QF=2) on \texttt{ImageNet} dataset.}
    \label{fig:images_imagenet_jpeg2}
\end{figure}
\clearpage

\begin{figure}[!htbp]
    \centering
    \includegraphics[width=0.935\linewidth]{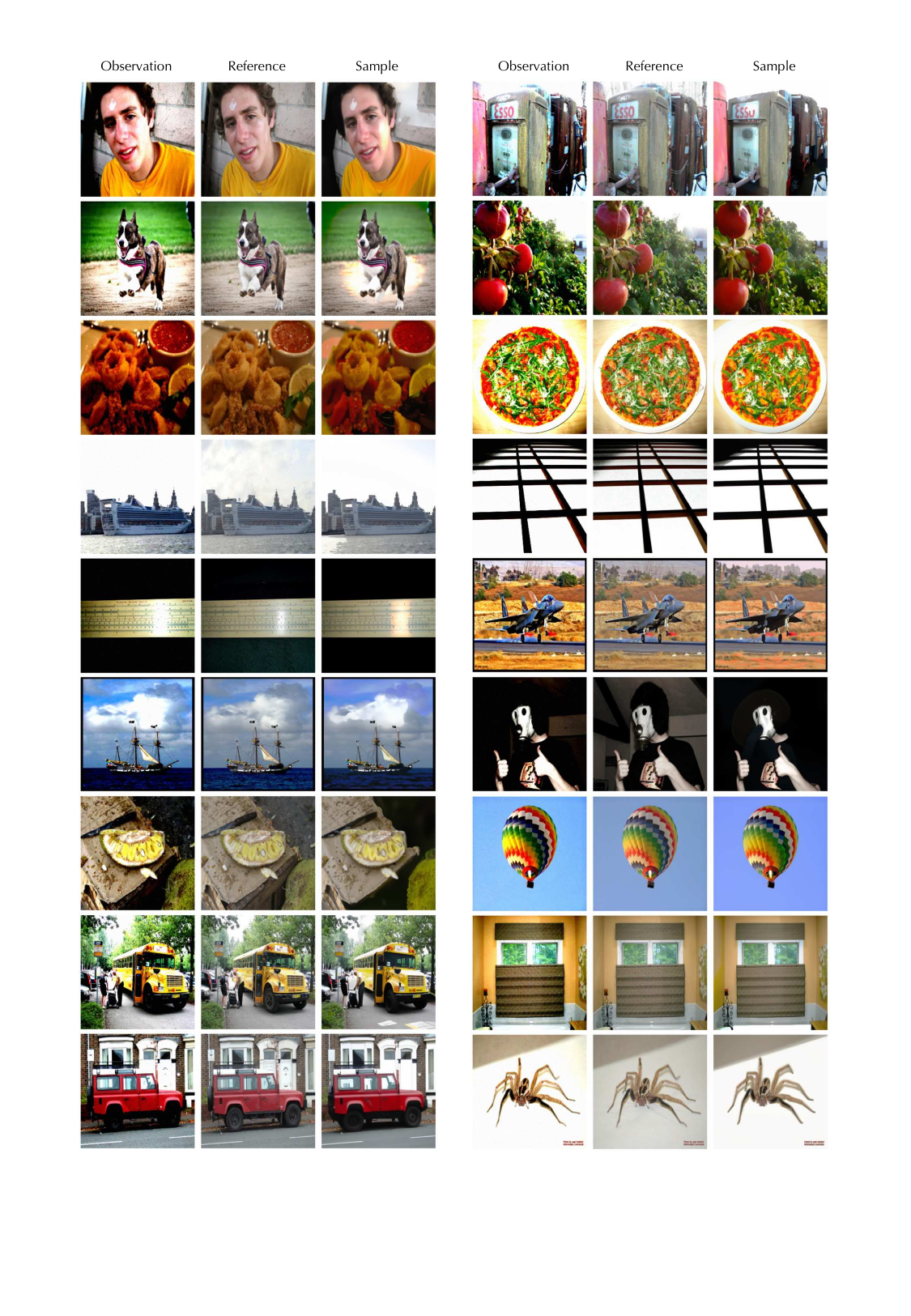}
    \vspace{-2.5cm}
    \caption{High Dynamic Range on \texttt{ImageNet} dataset.}
    \label{fig:images_imagenet_HDR}
\end{figure}
\clearpage

\newpage
\section*{NeurIPS Paper Checklist}

\begin{enumerate}

\item {\bf Claims}
    \item[] Question: Do the main claims made in the abstract and introduction accurately reflect the paper's contributions and scope?
    \item[] Answer: \answerYes{}
    \item[] Justification: The abstract and introduction state the three main contributions: an intermediate-time warm start, sparse scheduled guidance without denoiser/decoder VJPs, and improved quality--compute trade-offs for inverse problems. These claims are supported by the method section, Lemma~\ref{lem:warm_start_truncation}, and the FFHQ/ImageNet experiments, with scope restricted to the evaluated DDPM/DDIM-style image priors and inverse-problem settings.

\item {\bf Limitations}
    \item[] Question: Does the paper discuss the limitations of the work performed by the authors?
    \item[] Answer: \answerYes{}
    \item[] Justification: Section~\ref{sec:experiments} includes a limitations paragraph discussing sensitivity to the warm-start timestep, guidance schedule, and optimization budget, as well as the current restriction to pretrained DDPM/DDIM-style image priors on \texttt{FFHQ} and \texttt{ImageNet} at $256\times256$ resolution.

\item {\bf Theory assumptions and proofs}
    \item[] Question: For each theoretical result, does the paper provide the full set of assumptions and a complete (and correct) proof?
    \item[] Answer: \answerYes{}
    \item[] Justification: The main theoretical claim is Lemma~\ref{lem:warm_start_truncation}; its assumptions are stated in Assumption~\ref{ass:warm_start_stability}, and the proof is provided in Appendix~\ref{sec:warm_start_error}. Appendix~\ref{sec:renoising_app} also states and proves the optimal-coupling lemma used to motivate conservative re-noising.

    \item {\bf Experimental result reproducibility}
    \item[] Question: Does the paper fully disclose all the information needed to reproduce the main experimental results of the paper to the extent that it affects the main claims and/or conclusions of the paper (regardless of whether the code and data are provided or not)?
    \item[] Answer: \answerYes{}
    \item[] Justification: The paper gives pseudocode in Algorithm~\ref{alg:spin-highlevel}, implementation-level details in Appendix~\ref{sec:algorithm}, datasets, tasks, metrics, baselines, sample counts, and noise levels in Section~\ref{sec:experiments}, and method/baseline hyperparameters in Appendices~\ref{sec:hyperparameters} and~\ref{sec:competitors}.

\item {\bf Open access to data and code}
    \item[] Question: Does the paper provide open access to the data and code, with sufficient instructions to faithfully reproduce the main experimental results, as described in supplemental material?
    \item[] Answer: \answerYes{}
    \item[] Justification: The paper provides open access to the data and code, with sufficient instructions to faithfully reproduce the main experimental results, as described in supplemental material.

\item {\bf Experimental setting/details}
    \item[] Question: Does the paper specify all the training and test details (e.g., data splits, hyperparameters, how they were chosen, type of optimizer) necessary to understand the results?
    \item[] Answer: \answerYes{}
    \item[] Justification: Section~\ref{sec:experiments} specifies the datasets, pretrained priors, validation-sample protocol, image resolution, inverse tasks, observation noise, baselines, metrics, and task families. Appendix~\ref{sec:hyperparameters} reports the warm-start, guidance, optimization, and schedule hyperparameters used for the reported experiments.

\item {\bf Experiment statistical significance}
    \item[] Question: Does the paper report error bars suitably and correctly defined or other appropriate information about the statistical significance of the experiments?
    \item[] Answer: \answerYes{}
    \item[] Justification: The quantitative tables report mean $\pm$ standard deviation over 100 validation samples for LPIPS in the main paper and PSNR/SSIM in the supplementary material. No formal hypothesis tests are claimed.

\item {\bf Experiments compute resources}
    \item[] Question: For each experiment, does the paper provide sufficient information on the computer resources (type of compute workers, memory, time of execution) needed to reproduce the experiments?
    \item[] Answer: \answerYes{}
    \item[] Justification: The main tables report per-method GPU memory and runtime measurements, Table~\ref{tab:ffhq_nfe} reports NFE, and Appendix~\ref{sec:additional-experiments} states that experiments were run on a single node with 8 NVIDIA A6000 GPUs with runtime and memory measured on GPUs without competing processes.
    
\item {\bf Code of ethics}
    \item[] Question: Does the research conducted in the paper conform, in every respect, with the NeurIPS Code of Ethics \url{https://neurips.cc/public/EthicsGuidelines}?
    \item[] Answer: \answerYes{}
    \item[] Justification: The research uses existing public datasets, pretrained image priors, and synthetic inverse-problem degradations; it does not involve human-subject experiments, private data collection, or deployment decisions. The authors have reviewed the NeurIPS Code of Ethics and are not aware of any deviation.

\item {\bf Broader impacts}
    \item[] Question: Does the paper discuss both potential positive societal impacts and negative societal impacts of the work performed?
    \item[] Answer: \answerNo{}
    \item[] Justification: The paper does not discuss the broader impacts of the work.
    
\item {\bf Safeguards}
    \item[] Question: Does the paper describe safeguards that have been put in place for responsible release of data or models that have a high risk for misuse (e.g., pre-trained language models, image generators, or scraped datasets)?
    \item[] Answer: \answerNA{}
    \item[] Justification: The paper does not introduce or release a new pretrained model, scraped dataset, or high-risk data asset. It uses existing pretrained diffusion priors and datasets under their original access conditions.

\item {\bf Licenses for existing assets}
    \item[] Question: Are the creators or original owners of assets (e.g., code, data, models), used in the paper, properly credited and are the license and terms of use explicitly mentioned and properly respected?
    \item[] Answer: \answerNo{}
    \item[] Justification: The paper credits the datasets, pretrained models, metrics, and baseline methods through citations, but the current draft does not explicitly list the licenses or terms of use for all existing assets.

\item {\bf New assets}
    \item[] Question: Are new assets introduced in the paper well documented and is the documentation provided alongside the assets?
    \item[] Answer: \answerNA{}
    \item[] Justification: The paper introduces a new algorithmic method, but it does not introduce or release a new dataset, benchmark, pretrained model, or other standalone asset requiring separate asset documentation.

\item {\bf Crowdsourcing and research with human subjects}
    \item[] Question: For crowdsourcing experiments and research with human subjects, does the paper include the full text of instructions given to participants and screenshots, if applicable, as well as details about compensation (if any)? 
    \item[] Answer: \answerNA{}
    \item[] Justification: The work does not involve crowdsourcing, new data collection from participants, or experiments with human subjects.

\item {\bf Institutional review board (IRB) approvals or equivalent for research with human subjects}
    \item[] Question: Does the paper describe potential risks incurred by study participants, whether such risks were disclosed to the subjects, and whether Institutional Review Board (IRB) approvals (or an equivalent approval/review based on the requirements of your country or institution) were obtained?
    \item[] Answer: \answerNA{}
    \item[] Justification: The work does not involve crowdsourcing or human-subject experiments, so IRB approval or equivalent review is not applicable.

\item {\bf Declaration of LLM usage}
    \item[] Question: Does the paper describe the usage of LLMs if it is an important, original, or non-standard component of the core methods in this research? Note that if the LLM is used only for writing, editing, or formatting purposes and does \emph{not} impact the core methodology, scientific rigor, or originality of the research, declaration is not required.
    \item[] Answer: \answerNA{}
    \item[] Justification: The core methodology does not use LLMs as an important, original, or non-standard research component.

\end{enumerate}

\end{document}